\documentclass[a4paper,11pt]{article}

\usepackage{a4wide}

\usepackage[utf8]{inputenc}
\usepackage{makeidx}  % allows for indexgeneration
\usepackage{amssymb}
\usepackage{amsmath}
\usepackage{amsthm}
\usepackage{mathtools}
\usepackage{stmaryrd}
\usepackage{comment}
\usepackage{float}
\usepackage{etoolbox}
\usepackage{xargs}
\usepackage{pbox}
\usepackage{stmaryrd}
\usepackage{tabularx}
\usepackage{substr}
\usepackage{tikz}
\usetikzlibrary{fit,backgrounds,shapes,positioning,calc,decorations.pathmorphing}
\makeatletter

\pgfarrowsdeclare{ke}{ke}
{
  \@tempdima=0.45pt%
  \advance\@tempdima by.3\pgflinewidth%
  \pgfarrowsleftextend{-3\@tempdima}
  \pgfarrowsrightextend{5\@tempdima}
}
{
  \@tempdima=0.45pt%
  \advance\@tempdima by.3\pgflinewidth%
  \pgfpathmoveto{\pgfpoint{5\@tempdima}{0pt}}
  \pgfpathlineto{\pgfpoint{-3\@tempdima}{4\@tempdima}}
  \pgfpathlineto{\pgfpointorigin}
  \pgfpathlineto{\pgfpoint{-3\@tempdima}{-4\@tempdima}}
  \pgfusepathqfill
}
\makeatother

\tikzset{world/.style={circle, draw, thick, fill=none, minimum size = .65cm, inner sep = 0.2pt, outer sep = 0cm}}
\tikzset{event/.style={rectangle, draw, thick, fill=gray!40, minimum size = .65cm, inner sep=0pt, outer sep=0cm}}
\tikzset{link/.style={draw, -ke}}
\tikzset{lbl/.style={draw=none, inner sep = 0cm, outer sep = 0 cm}}
\tikzset{node distance = 3.5pt} % Sets distance between "point" labels in plausibility / event models 
\def\wrappad{.9ex}
\tikzset{%
    wrapper/.style 2 args={%
      local bounding box=localbb,
      execute at end scope={ 
      \begin{scope}[on background layer] 
      \node[inner sep = \wrappad,
            draw,
            line width=.15pt,
            fit=(localbb),
            label = {left:#2}
            ] (#1) {}; 
  \end{scope}}}}

\theoremstyle{plain}% default
\newtheorem{theorem}{Theorem}
\newtheorem{lemma}{Lemma}
\newtheorem{proposition}{Proposition}
\newtheorem{corollary}{Corollary}
\theoremstyle{definition}
\newtheorem{definition}{Definition}
\newtheorem{example}{Example}

\def\props{P} % Set of proposition symbols
\def\lang{L}

\def\m{M} % Epistemic model name
\newcommand{\dom}[1]{D(#1)} % 

\def\comments{1}
\if\comments 1
  \newcommand{\mba}[1]{{\bfseries\color{purple} [Mikkel: #1]}}
  \newcommand{\tb}[1]{{\bfseries\color{cyan}[Thomas: #1]}}
  \newcommand{\mhj}[1]{{\bfseries\color{orange}[Martin: #1]}}
  \newcommand{\hvd}[1]{{\bfseries\color{teal}[Hans: #1]}}
\else
  \newcommand{\mba}[1]{}
  \newcommand{\tb}[1]{}
  \newcommand{\mhj}[1]{}
  \newcommand{\hvd}[1]{}
\fi

\newcommand{\imp}{\rightarrow}

\newcommand{\T}{\top}

\newcommand{\Dia}{\Diamond}

\renewcommand{\phi}{\varphi}
\newcommand{\union}{\cup}

\newcommand{\inter}{\cap}

\newcommand{\powerset}{{\mathcal P}}

\newcommand{\bisim}{{\raisebox{.3ex}[0mm][0mm]{\ensuremath{\medspace \underline{\! \leftrightarrow\!}\medspace}}}}

\newcommand{\bisrel}{\ensuremath{\mathfrak{R}}}
\DeclareMathOperator{\Min}{\ensuremath{\mathit{Min}}}
\newcommand{\II}[1]{[\![ #1 ]\!]} % consistent with Kluwer stylefile

\def\agents{A}

\newcommand{\equivclass}[2]{\ensuremath{[#1]_{#2}}}

\newcommand{\Domain}{D}
\newcommand{\Nat}{\mathbb N}
\newcommand{\Naturals}{\Nat}
\usepackage{url}

\renewcommand{\@}{{\color{red} @}}

\newcommand{\edge}[3]{edge[#3] node[#2] {$#1$}}

\renewcommand{\bisrel}{R}

\begin{document}

\title{Bisimulation and expressivity for conditional belief, degrees of belief, and safe belief}

\author{Mikkel Birkegaard Andersen\thanks{DTU Compute, Technical University of Denmark, \tt{\{mibi,tobo,mhje\}@dtu.dk}},
Thomas Bolander,\\%\thanks{DTU Compute, Technical University of Denmark, \tt{tobo@dtu.dk}}, 
Hans van Ditmarsch\thanks{LORIA, CNRS / Universit\'{e} de Lorraine, \tt{hans.van-ditmarsch@loria.fr}}, and 
Martin Holm Jensen%\thanks{DTU Compute, Technical University of Denmark, \tt{mhje@dtu.dk}}
}

\date{\today}

\maketitle

	\begin{abstract}
Plausibility models are Kripke models that agents use to reason about knowledge and belief, both of themselves and of each other. Such models are used to interpret the notions of conditional belief, degrees of belief, and safe belief. The logic  of conditional belief contains that modality and also the knowledge modality, and similarly for the logic of degrees of belief and the logic of safe belief. With respect to these logics, plausibility models may contain too much information. A proper notion of bisimulation is required that characterises them. We define that notion of bisimulation and prove the required characterisations: on the class of image-finite and preimage-finite models (with respect to the plausibility relation), two pointed Kripke models are modally equivalent in either of the three logics, if and only if they are bisimilar. As a result, the information content of such a model can be similarly expressed in the logic of conditional belief, or the logic of degrees of belief, or that of safe belief. This, we found a surprising result. Still, that does not mean that the logics are equally expressive: the logics of conditional and degrees of belief are incomparable, the logics of degrees of belief and safe belief are incomparable, while the logic of safe belief is more expressive than the logic of conditional belief. In view of the result on bisimulation characterisation, this is an equally surprising result. We hope our insights may contribute to the growing community of formal epistemology and on the relation between qualitative and quantitative modelling.
  \end{abstract}

%%%%%%%%%%%%%%%%%%
% Start of sections
% Bisimulation
%%%%%%%%%%%%%%%%%%

\setcounter{footnote}{0}

\section{Introduction}\label{sect:motivation}

A typical approach in belief revision involves preferential orders to express degrees of belief and knowledge \cite{klm:1990,meyeretal:2000}. This goes back to the `systems of spheres' in \cite{lewis:1973,grove:1988}. Dynamic doxastic logic was proposed and investigated in \cite{segerberg:1998} in order to provide a link between the (non-modal logical) belief revision and modal logics with explicit knowledge and belief operators. A similar approach was pursued in belief revision in dynamic epistemic logic \cite{aucher:2005a,hvd.prolegomena:2005,jfak.jancl:2007,balt.ea:qual,hvdetal.hw:2007}, that continues to develop strongly \cite{BritzVarzinczak2013,jfak.book:2011}. We focus on the proper notion of structural equivalence on models encoding knowledge and belief simultaneously. A prior investigation into that is \cite{demey:2011}, which we relate our results to at the end of the paper. Our motivation is to find suitable structural notions to reduce the complexity of solving planning problems. Solutions to planning problems are sequences of actions, such as iterated belief revision. It is the dynamics of knowledge and belief that, after all, motivates our research.

	The semantics of belief depends on the structural properties of models. To relate the structural properties of models to a logical language we need a notion of structural similarity, known as bisimulation. A bisimulation relation relates a modal operator to an accessibility relation. Plausibility models do not have an accessibility relation as such but a plausibility relation. This induces a set of accessibility relations: the {\em most plausible} worlds are the {\em accessible} worlds for the modal belief operator; and the {\em plausible} worlds are the {\em accessible} worlds for the modal knowledge operator. But it contains much more information: to each modal operator of conditional belief (or of degree of belief) one can associate a possibly distinct accessibility relation.This raises the question of how to represent the bisimulation conditions succinctly. Can this be done by reference to the plausibility relation directly, instead of by reference to these, possibly many, induced accessibility relations? It is now rather interesting to observe that relative to the modalities of knowledge and belief, the plausibility relation is already in some way too rich.
	
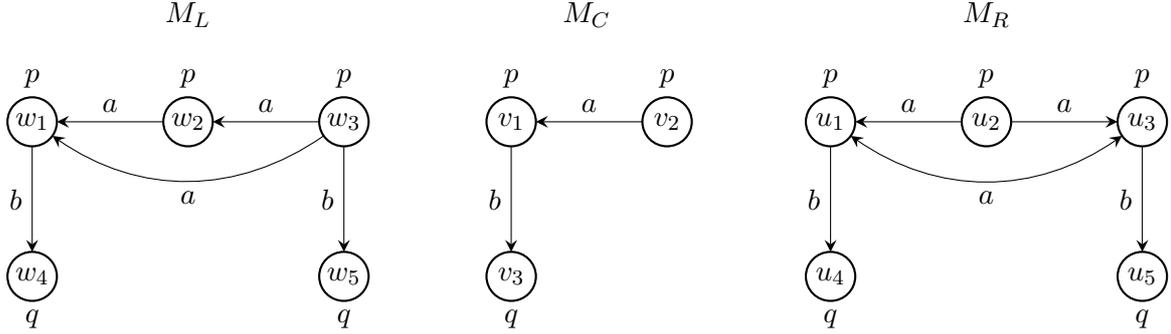
\begin{figure}
		\centering
%$\m_L$ \hspace{4.2cm} $\m_C$ \hspace{4.2cm} $\m_R$ \\ \ \\
		\def\worldX{1.4cm}
	  	\begin{tikzpicture}
	  		\begin{scope}[xshift = 0cm]
		  		\node[world] (w1) {$w_1$}
		  			node[lbl, above = of w1] {$p$};
		  		\node[world, right = \worldX of w1] (w2) {$w_2$}
		  			node[lbl, above = of w2] {$p$};
		  		\node[world, right = \worldX of w2] (w3) {$w_3$}
		  			node[lbl, above = of w3] {$p$};
		  		\node[world, below = \worldX of w1] (w4) {$w_4$}
		  			node[lbl, below = of w4] {$q$};
		  		\node[world, below = \worldX of w3] (w5) {$w_5$}
		  			node[lbl, below = of w5] {$q$};
				\path[link] (w3) \edge{a}{above}{} (w2);
				\path[link] (w2) \edge{a}{above}{} (w1);
				\path[link] (w3) \edge{a}{below}{bend left=35} (w1);
				\path[link] (w1) \edge{b}{left}{} (w4);
				\path[link] (w3) \edge{b}{left}{} (w5);

			\end{scope}
		%	\node[scale=1.4, transform shape] at (4,-0.8) {$\bisim$};
	  		\begin{scope}[xshift = 6.3cm]
		  		\node[world] (v1) {$v_1$}
		  			node[lbl, above = of v1] {$p$};
		  		\node[world, right = \worldX of v1] (v2) {$v_2$}
		  			node[lbl, above = of v2] {$p$};
		  		\node[world, below = \worldX of v1] (v3) {$v_3$}
		  			node[lbl, below = of v3] {$q$};
				\path[link] (v2) \edge{a}{above}{} (v1);
				\path[link] (v1) \edge{b}{left}{} (v3);
			\end{scope}
		%	\node[scale=1.4, transform shape] at (7.5,-0.8) {$\bisim$};
			\begin{scope}[xshift = 10.5cm]
		  		\node[world] (u1) {$u_1$}
		  			node[lbl, above = of u1] {$p$};
		  		\node[world, right = \worldX of u1] (u2) {$u_2$}
		  			node[lbl, above = of u2] {$p$};
		  		\node[world, right = \worldX of u2] (u3) {$u_3$}
		  			node[lbl, above = of u3] {$p$};
		  		\node[world, below = \worldX of u1] (u4) {$u_4$}
		  			node[lbl, below = of u4] {$q$};
		  		\node[world, below = \worldX of u3] (u5) {$u_5$}
		  			node[lbl, below = of u5] {$q$};
				\path[link] (u2) \edge{a}{above}{} (u3);
				\path[link] (u2) \edge{a}{above}{} (u1);
				\path[link] (u1) \edge{b}{left}{} (u4);
				\path[link] (u3) \edge{b}{left}{} (u5);
				\path[link,ke-ke] (u1) \edge{a}{below}{bend right = 35} (u3);
			\end{scope}
			
			%Model labels 
			\node[above = 0.8cm of w2] (lblw) {$M_L$}; 			
			\node[above = 0.8cm of u2] (lblu) {$M_R$}; 
			\path (lblw) -- (lblu) node[midway] (lblv) {$M_C$};

		\end{tikzpicture}
		\caption{An arrow $x \rightarrow y$ labelled by $a$ means $x \geq_a y$; agent $a$ considers $y$ at least as plausible as $x$. We use $x >_a y$ to mean $x \geq_a y$ and $y \not\geq_a x$. 
		 Here $w_2 >_a w_1$, so $w_1$ is strictly more plausible than $w_2$. Reflexive edges are omitted. Unlisted propositions are false.}
		\label{fig:intro:models}
	\end{figure}

The plausibility model $M_L$ on the left in Figure \ref{fig:intro:models} consists of five worlds. The proposition $p$ is true in the top ones and false in the bottom ones. The reverse holds for $q$: true at the bottom and false at the top. The $a$ relations in the model correspond to the plausibility order $w_3 >_a w_2 >_a w_1$, interpreted such that the smaller of two elements in the order is the most plausible of the two. Further, everything that is comparable with the plausibility order is considered epistemically possible. Hence, the epistemic equivalence classes for agent $a$ in $M_L$ are $\{w_1,w_2,w_3\}$, $\{w_4\}$ and $\{w_5\}$. We can then view the model as a standard multi-agent $S5$ model plus an ordering on the epistemic possibilities. As $w_1$ is the most plausible world for $a$ in the equivalence class $\{ w_1, w_2, w_3 \}$, she will in $w_3$ believe $p$ and that $b$ believes $\neg p \land q$. This works differently from the usual doxastic modal logic, where belief corresponds to the accessibility relation. In the logics of belief that we study, belief is what holds in the most plausible world(s) in an epistemic equivalence class. For $a$, the most plausible world in her equivalence class $\{w_1,w_2,w_3\}$ is $w_1$, so $a$ believes the same formulas in all of them. 

In $w_2$ agent $b$ knows $p$. If $a$ is given the information that $b$ does not consider $q$ possible (that is, the information that neither $w_1$ nor $w_3$ is the actual world), then $a$ believes that $b$ knows $p$ -- or conditional on $K_b \neg q$, $a$ believes $K_b p$. Such a statement is an example of the logic of conditional belief $\lang^C$ defined in Section \ref{section:semantics}. In $\lang^C$ we write this statement as $B^{K_b \neg q}_a K_b p$.

Now examine $w_3$. We will show that $w_1$ and $w_3$ are modally equivalent for $\lang^C$: they agree on all formulas of that language---no information expressible in $L^C$ distinguishes the two worlds. This leads to the observation that no matter where we move $w_3$ in the plausibility ordering for $a$, modal equivalence is preserved. Similarly, we can move $w_2$ anywhere we like \emph{except} making it more plausible than $w_1$. If we did, then $a$ would believe $K_b p$ unconditionally, and the formulas true in the model would have been changed.

It turns out that moving worlds about in the plausibility order can be done for all models, as long as we obey one (conceptually) simple rule: Grouping worlds into ``modal equivalence classes'' of worlds modally equivalent to each other, we are only required to preserve the ordering between the \emph{most} plausible worlds in each modal equivalence class. \emph{Only the most plausible world in each class matters}.

Another crucial observation is that standard bisimulation in terms of $\geq_a$ does not give correspondence between bisimulation and modal equivalence. For instance, while $w_1$ and $w_3$ are modally equivalent, they are not ``standardly'' bisimilar with respect to $\geq_a$: $w_3$ has a $\geq_a$-edge to a $K_b p$ world ($w_2$), whereas $w_1$ does not. Thus, the straightforward, standard definition of bisimulation does not work, because no modality in the language corresponds to the plausibility relation. Instead we have an infinite set of modalities corresponding to relations derived from the plausibility relation. One of the major contributions of this paper is a solution to exactly this problem.

Making $w_3$ as plausible as $w_1$ and appropriately renaming worlds gets us $\m_R$ of Figure \ref{fig:intro:models}. Here the modally equivalent worlds $u_1$ and $u_3$ are equally plausible, modally equivalent \emph{and} standardly bisimilar. This third observation gives a sense of how we solve the problem generally. Rather than using $\geq_a$ directly, our definition of bisimulation checks accessibility with respect to a relation $\geq_a^R$ derived from $\geq_a$ and the bisimulation relation $R$ itself. Postponing details for later we just note that in the present example the derived relation for $M_L$ is exactly the plausibility relation for $M_R$. This indicates what we later prove: This new derived relation reestablishes the correspondence between bisimilarity and modal equivalence.

The model $\m_C$ of Figure \ref{fig:intro:models} is the bisimulation contraction of the right model using standard bisimilarity. It is the bisimulation contraction of both models with the bisimulation notion informally defined in the previous paragraph. In previous work on planning with single-agent plausibility models \cite{dontplanfortheunexpected}, finding contractions of plausibility models is needed for decidability and complexity results. In this paper we do this for the first time for multi-agent plausibility models, opening new vistas in applications of modal logic to automated planning.

\paragraph*{Overview of content}
In Section \ref{section:plaus} we introduce plausibility models and the proper and novel notion of bisimulation on these models, and prove various properties of bisimulation. In Section \ref{section:semantics} we define the three logics of conditional belief, degrees of belief, and safe belief, and provide some further historical background on these logics. In Section \ref{section:corr} we demonstrate that bisimilarity corresponds to logical equivalence (on image-finite and preimage-finite models) for all three core logics, so that, somewhat remarkably, one could say that the content of a given model can equally well be described in any of these logics. Then, in Section \ref{section:expressivity} we determine the relative expressivity of the three logics, including more expressive combinations of their primitive modalities. The main result here is that the logics of conditional and degrees of belief are incomparable, and that the logics of degrees of belief and safe belief are incomparable, but that the logic of safe belief is (strictly) more expressive than the logic of conditional belief. In Section \ref{section:comparison}, we put our result in the perspective of other recent investigations, mainly the study by Lorenz Demey \cite{demey:2011}, and in the perspective of possible applications: decidable planning.

%%%%%%%%%%%%%%%%%%
% Plausibility models and bisimulation
%%%%%%%%%%%%%%%%%%	
\section{Plausibility models and bisimulation} \label{section:plaus}

A \emph{well-preorder} on a set $X$ is a reflexive and transitive binary relation $\trianglerighteq$ on $X$ such that every non-empty subset has $\trianglerighteq$-minimal elements. The set of \emph{minimal elements} (for $\trianglerighteq$) of some $Y \subseteq X$ is the set $\Min_{\trianglerighteq} Y$ defined as $\{ y \in Y \mid y' \trianglerighteq y \text{ for all } y' \in Y \}$.\footnote{This notion of minimality is non-standard and taken from \cite{balt.ea:qual}. Usually a minimal element of a set is an element that is not greater than any other element.} As any two-element subset $Y = \{x,y\}$ of $X$ also has minimal elements, we have that $x \trianglerighteq y$ or $y \trianglerighteq x$. Thus all elements in $X$ are $\trianglerighteq$-comparable. 

Given any binary relation $R$ on $X$, we use $R^=$ to denote the reflexive, symmetric, and transitive closure of $R$ (the equivalence closure of $R$). For any equivalence relation $R$ on $X$, we write $[x]_R$ for $\{x' \in X \mid (x,x') \in R\}$. A binary relation $R$ on $X$ is \emph{image-finite} if and only if for every $x \in X$, $\{ x' \in X \mid (x,x') \in R\}$ is finite. A relation is \emph{preimage-finite} if and only if for every $x \in X$, $\{ x' \in X \mid (x',x) \in R\}$ is finite. We say $R$ is \emph{(pre)image-finite} if it is both image-finite and preimage-finite. We often write $xRy$ instead of $(x,y) \in R$. Given subsets $Y,Z \subseteq X$, we define $Y R Z$ if and only if $y R z$ for all $y \in Y$ and all $z \in Z$. 

\begin{definition}[Plausibility model]
	A \emph{plausibility model} for a countably infinite set of propositional symbols $\props$ and a finite set of agents $\agents$ is a tuple $\m = (W, \geq, V)$, where
\begin{itemize}
	\item $W$ is a set of \emph{worlds} called the \emph{domain}, denoted $D(\m)$;
	\item $\geq: \agents \imp \powerset(W \times W)$ is a function mapping each $a \in \agents$ into a \emph{plausibility relation} $\geq\!\!(a)$, usually abbreviated $\geq_a$.  For each $a \in \agents$ and $w \in W$, $\geq_a$ is a well-preorder on the set $\{ w' \in W \mid w \geq_a w' \text{ or } w' \geq_a w \}$. Each $\geq_a$ is required to be (pre)image-finite;
\item $V: W \to 2^\props$ is a \emph{valuation}.
\end{itemize}
For $w \in W$, $(\m, w)$ is a \emph{pointed plausibility model}. % and we refer to $w$ as the \emph{actual world} of $(\m,w)$. In order to distinguish the components of different models, we may use the name of the model as an identifier and write $W^M$, $\geq^M$, and $V^M$, instead of, respectively, $W$, $\geq$, and $V$.
\end{definition}
%For $\geq\!\!(a)$ we write $\geq_a$. 
If $w \geq_a v$ then $v$ is {\em at least as plausible} as $w$ (for agent $a$), and the $\geq_a$-minimal elements are the {\em most plausible} worlds. For the symmetric closure of $\geq_a$ we write $\sim_a$: this is an equivalence relation on $W$ called the \emph{epistemic relation} (for agent $a$). If $w \geq_a v$ but $v \not\geq_a w$ we write $w >_a v$ ($v$ is {\em more plausible} than $w$), and for $w \geq_a v$ and $v \geq_a w$ we write $w \simeq_a v$ ($w$ and $v$ are {\em equiplausible}). Instead of $w \geq_a v$ ($w >_a v$) we may write $v \leq_a w$ ($v <_a w$).

Note that we have required each relation $\geq_a$ to be (pre)image-finite. This amounts to requiring that all equivalence classes of $\sim_a$ are finite, while still allowing infinite domains. This requirement is not part of the definition of plausibility models provided in~\cite{balt.ea:qual}. We require it here, since it leads to simplifications without any significant reduction in generality: 
\begin{enumerate}
  \item We will show full correspondence between bisimilarity and modal equivalence for three different logics over plausibility models. As is the case in standard modal logic, this correspondence can only be achieved for (pre)image-finite models (the direction from modal equivalence to bisimilarity only hold for such models, see e.g.\ \cite{blac.ea:moda}). Simply assuming (pre)image-finiteness from the outset simplifies the presentation, as we do not have to repeat this restriction in a large number of places.
  \item Some of our later results are going to rely on the existence of a largest autobisimulation (see below). Usually it is quite trivial to show the existence of a largest bisimulation, since the union of any set of bisimulations is a bisimulation. However, we need a non-standard notion of bisimulation for our purposes, and for such bisimulations closure under union is far from a trivial result.\footnote{Without the restriction to (pre)image-finite models, we were unable to prove the existence of a largest bisimulation. We leave this challenge open to future research(ers).} Given our first correspondence result between bisimilarity and modal equivalence, we are however going to get this result for free (see Section~\ref{section:ccond}). 
\end{enumerate}  
    
We now proceed to define a notion of autobisimulation on a plausibility model. This notion is non-standard, because there is no one-to-one relation between the plausibility relation for an agent and a modality for that agent in the logics defined later. In the definition below (and from now on), we allow ourselves some further notational abbreviations. Let $\m = (W,\geq,V)$ denote a plausibility model. Let $a\in\agents$ and $w\in W$, then we write $[w]_a$ instead of $[w]_{\sim_a}$. Now let $Z\subseteq [w]_a$, then we write $\Min_a Z$ instead of $\Min_{\geq_a} Z$. For any binary relation $R$ on $W$, we write $w \geq_a^R v$ for $\Min_a ([w]_{R^=} \cap [w]_a) \geq_a \Min_a ([v]_{R^=} \cap [v]_a)$. When $w \geq_a^R v$ and $v \geq_a^R w$, we write $w \simeq_a^R v$.
\begin{definition}[Autobisimulation]\label{def:autobisim}
	Let $\m = (W,\geq,V)$ be a plausibility model. An \emph{autobisimulation} on $\m$ is a non-empty relation $\bisrel \subseteq W \times W$ such that for all $(w, w') \in \bisrel$ and for all $a \in \agents$:
\begin{description}
	\item[{[atoms]}] $V(w) = V(w')$;
	\item[{[forth$_\geq$]}] If $v \in W$ and $w \geq_a^\bisrel v$, there is a $v' \in W$ such that $w' \geq_a^\bisrel v'$ and $(v,v') \in \bisrel$;
	\item[{[back$_\geq$]}] If $v' \in W$ and $w' \geq_a^\bisrel v'$, there is a $v \in W$ such that $w \geq_a^\bisrel v$ and $(v,v') \in \bisrel$;
	\item[{[forth$_\leq$]}] If $v \in W$ and $w \leq_a^\bisrel v$, there is a $v' \in W$ such that $w' \leq_a^\bisrel v'$ and $(v,v') \in \bisrel$;
	\item[{[back$_\leq$]}] If $v' \in W$ and $w' \leq_a^\bisrel v'$, there is a $v \in W$ such that $w \leq_a^\bisrel v$ and $(v,v') \in \bisrel$.
\end{description}
A \emph{total autobisimulation} on $\m$ is an autobisimulation with $W$ as both domain and codomain.
\end{definition}
Our bisimulation relation is non-standard in the [back] and [forth] clauses.  A standard [forth] condition based on an accessibility relation $\geq_a$ would be \begin{quote} If $v \in W$ and $w \geq_a v$, there is a $v' \in W'$ such that $w' \geq_a v'$ and $(v,v') \in R$. \end{quote} Here, $R$ only appears in the part `$(v,v') \in R$'. But in the definition of autobisimulation for plausibility models, in [forth$_\geq$], the relation $R$ also features in the condition for applying [forth$_\geq$] and in its consequent, namely as the upper index in $w \geq_a^R v$ and $w' \geq_a^R v'$.  This means that $R$ also determines which $v$ are accessible from $w$, and which $v'$ are accessible from $w'$. This explains why we define an autobisimulation on a single model before a bisimulation between distinct models: We need the bisimulation relation $R$ to determine the plausibility relation $\geq_a^R$ from the plausibility relation $\geq_a$ on any given model first, before structurally comparing distinct models. 

	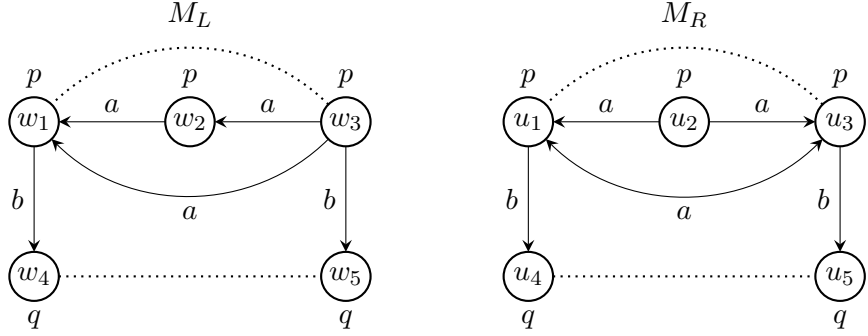
\begin{figure}[h]
		\centering
		\def\worldX{1.4cm}
	  	\begin{tikzpicture}
			\node[above = 0.8cm of w2] (lblw) {$M_L$}; 	
	  		\begin{scope}[xshift = 0cm]
		  		\node[world] (w1) {$w_1$}
		  			node[lbl, above = of w1] {$p$};
		  		\node[world, right = \worldX of w1] (w2) {$w_2$}
		  			node[lbl, above = of w2] {$p$};
		  		\node[world, right = \worldX of w2] (w3) {$w_3$}
		  			node[lbl, above = of w3] {$p$};
		  		\node[world, below = \worldX of w1] (w4) {$w_4$}
		  			node[lbl, below = of w4] {$q$};
		  		\node[world, below = \worldX of w3] (w5) {$w_5$}
		  			node[lbl, below = of w5] {$q$};
				\path[link] (w3) \edge{a}{above}{} (w2);
				\path[link] (w2) \edge{a}{above}{} (w1);
				\path[link] (w3) \edge{a}{below}{bend left=45} (w1);
				\path[link] (w1) \edge{b}{left}{} (w4);
				\path[link] (w3) \edge{b}{left}{} (w5);
				
				\draw[thick, dotted] (w3)  to [bend left=-45] (w1);
				\draw[thick, dotted] (w4) to (w5);
			\end{scope}
				\begin{scope}[xshift = 6.5cm]
				\node[world] (u1) {$u_1$}
		  			node[lbl, above = of u1] {$p$};
		  		\node[world, right = \worldX of u1] (u2) {$u_2$}
		  			node[lbl, above = of u2] {$p$};
		  		\node[world, right = \worldX of u2] (u3) {$u_3$}
		  			node[lbl, above = of u3] {$p$};
		  		\node[world, below = \worldX of u1] (u4) {$u_4$}
		  			node[lbl, below = of u4] {$q$};
		  		\node[world, below = \worldX of u3] (u5) {$u_5$}
		  			node[lbl, below = of u5] {$q$};
				\path[link] (u2) \edge{a}{above}{} (u3);
				\path[link] (u2) \edge{a}{above}{} (u1);
				\path[link] (u1) \edge{b}{left}{} (u4);
				\path[link] (u3) \edge{b}{left}{} (u5);
				\path[link,ke-ke] (u1) \edge{a}{below}{bend right = 45} (u3);
				
				\draw[thick, dotted] (u3)  to [bend left=-45] (u1);
				\draw[thick, dotted] (u4) to (u5);
			\end{scope}
			\node[above = 0.8cm of u2] (lblw) {$M_R$}; 	
		\end{tikzpicture}
		\caption{The left and right models of Figure \ref{fig:intro:models}, with the dotted lines showing the largest autobisimulations (modulo reflexivity).}
		\label{fig:bisim:autobisim}
	\end{figure}
\begin{example}%\label{ex:autobisim}
The models $\m_L$ and $\m_R$ of Figure~\ref{fig:intro:models} are reproduced in Figure~\ref{fig:bisim:autobisim}. Consider the relation $R = R_{id} \cup \{(w_1,w_3),(w_3,w_1), (w_4,w_5),(w_5,w_4)\}$, where $R_{id}$ is the identity relation on $W$.
With this $R$, we get that $w_1$ and $w_3$ are equiplausible for $\geq_a^R$:
\begin{align*}
	w_1 &\simeq^R_a w_3 \text{ iff } \\
	\Min_{a} ([w_1]_{\bisrel^=} \cap [w_1]_{a}) &\simeq_a \Min_{a}([w_3]_{\bisrel^=} \cap [w_3]_{a}) \text{ iff } \\
	\Min_{a} \{w_1,w_3\} &\simeq_a \Min_{a} \{w_1,w_3\} \text{ iff } \\
	w_1 &\simeq_a w_1
\end{align*}
We also get that $w_2 \geq_a^R w_3$:
\begin{align*}
	w_2 &\geq_a^R w_3 \text{ iff } \\
	\Min_a ([w_2]_{\bisrel^=} \cap [w_2]_a) &\geq_a \Min_a([w_3]_{\bisrel^=} \cap [w_3]_a) \text{ iff } \\
	\Min_a \{w_2\} &\geq_a \Min_a \{w_1,w_3\} \text{ iff } \\
	w_2 &\geq_a w_1
\end{align*} This gives ${\geq_a^R} = \{ (w_1, w_3), (w_3, w_1), (w_2, w_3), (w_2, w_1) \} \union R_{id}$.  For $b$, we get ${\geq_b^R} = {\geq_b}$. The autobisimulation $R$ on $M_L$ is shown in Figure \ref{fig:bisim:autobisim}. It should be easy to check that $R$ is indeed an autobisimulation. To help, we will justify why $(w_4, w_5)$ is in $R$: For $\geq_b^R$, we have that, as $(w_1,w_3) \in R$ and $w_1 \geq_b^R w_4$, there must be a world $v$ such that $w_3 \geq_b^R v$ and $(w_4,v) \in \bisrel$. This $v$ is $w_5$.

Note that $R$ is the largest autobisimulation. Based on [atoms] there are only two possible candidate pairs that could potentially be added to $R$ (modulo symmetry), namely $(w_1,w_2)$ and $(w_2, w_3)$. But $w_2$ does not have a $b$-edge to a $q$ world, whereas both $w_1$ and $w_3$ do. There is therefore nothing more to add.

The largest autobisimulation for $M_R$ is completely analogous, as shown in Figure~\ref{fig:bisim:autobisim}.
\end{example}
\begin{lemma}\label{lemma:equiplausr} Let $\m = (W, \geq, V)$ be a plausibility model and $R$ a binary relation on $W$. If $(w,w') \in R^=$ and $w \sim_a w'$ then $w \simeq_a^R w'$.
\end{lemma}
\begin{proof}
From $(w,w') \in R^=$ and $w \sim_a w'$ we get $[w]_{R^=} = [w']_{R^=}$ and $[w]_a = [w']_a$ and hence $[w]_{R^=} \cap [w]_a = [w']_{R^=} \cap [w']_a$. Thus also $\Min_a ([w]_{R^=} \cap [w]_a) = \Min_a ([w']_{R^=} \cap [w']_a)$, immediately implying $w \simeq_a^R w'$.
\end{proof}
Let $\m = (W,\geq,V)$ be a plausibility model. By definition, for each agent $a$, $\geq_a$ is a well-preorder on each $\sim_a$-equivalence class. The following result shows that the same holds for $\geq_a^R$ where $R$ is \emph{any} binary relation.
%given any binary relation $R$, the same holds for $\geq_a^R$. 
%The following result shows that for any binary relaion $R$ on the worlds of a plausibility model $\m = (W,\geq,V)$, $\geq_a^R$ and 
\begin{lemma} \label{lemma:wellr}
Let $\m = (W,\geq,V)$ be a plausibility model and $R$ a binary relation on $\m$. 
Then $\geq_a^R$ is a well-preorder on each $\sim_a$-equivalence class.
%Then $\geq_a^R$ is a set of well-preorders inducing the same partition as $\geq_a$, that is, $\geq_a^R$ partitions $W$ into a well-preorder on each $\sim_a$-equivalence class.
\end{lemma}
\begin{proof} 
The relation $\geq_a$ partitions $W$ into a well-preorder on each $\sim_a$-equivalence class, by definition. We need to show that $\geq_a^R$ does the same. Hence we need to prove: 1) $\geq_a^R$ is reflexive; 2) $\geq_a^R$ is transitive; 3) any $\sim_a$-equivalence class has $\geq_a^R$-minimal elements; 4) if two worlds are related by $\geq_a^R$ they are also related by $\sim_a$. 

Reflexivity of $\geq^R_a$ is trivial. \emph{Transitivity}: Let $(w,v),(v,u) \in {\geq_a^R}$. Then $\Min_a ([w]_{R^=} \cap [w]_a) \geq_a \Min_a ([v]_{R^=} \cap [v]_a)$, and  $\Min_a ([v]_{R^=} \cap [v]_a) \geq_a \Min_a ([u]_{R^=} \cap [u]_a)$. Using that for any sets $X,Y,Z$, if $X \geq_a Y$ and $Y \geq_a Z$ then $X \geq_a Z$ (transitivity of $\geq_a$ for sets is easy to check), we obtain that $\Min_a ([w]_{R^=} \cap [w]_a) \geq_a \Min_a ([u]_{R^=} \cap [u]_a)$ and therefore $(w,u) \in\ \geq_a^R$. \emph{Minimal elements}: Consider a $\sim_a$-equivalence class $W'' \subseteq W$, and let $W' \subseteq W''$ be a non-empty subset.  Suppose $W'$ does not have $\geq_a^R$ minimal elements. Then for all $w' \in W'$ there is a $w'' \in W'$ such that $w'' <^R_a w'$, i.e. $\Min_a ([w'']_{R^=} \cap [w'']_a) <_a \Min_a ([w']_{R^=} \cap [w']_a)$.  As $w' \in [w']_{R^=} \cap [w']_a$, we get $\{w'\} \geq_a \Min_a ([w']_{R^=} \cap [w']_a)$ and then $\Min_a ([w'']_{R^=} \cap [w'']_a) <_a \{w'\}$.  In other words, for all $w' \in W'$ there is a $u \in W$, namely any $u \in \Min_a ([w'']_{R^=} \cap [w'']_a)$, such that $u <_a w'$. This contradicts $\geq_a$ being a well-preorder on $W''$. We have now shown 1), 2) and 3). Finally we show 4): Assume $w \geq_a^R v$, that is, $\Min_a ([w]_{R^=} \cap [w]_a) \geq_a \Min_a ([v]_{R^=} \cap [v]_a)$. This implies the existence of an $x \in \Min_a ([w]_{R^=} \cap [w]_a)$ and a $y \in \Min_a ([w]_{R^=} \cap [v]_a)$ with $x \geq_a y$. By choice of $x$ and $y$ we have $x \sim_a w$ and $y \sim_a v$. From $x \geq_a y$ we get $x \sim_a y$. Hence we have $w \sim_a x \sim_a y \sim_a v$, as required. 
\end{proof}

\begin{proposition}\label{prop:maxautos}
 On any plausibility model there exists a largest autobisimulation. Furthermore, the largest autobisimulation is an equivalence relation. 
\end{proposition}
We postpone the proof of this result to Section~\ref{section:ccond}, since it is going to follow from the correspondence between bisimilarity and modal equivalence for our language of conditional belief (Theorem~\ref{theorem:bisimcharlc}). Let us already now reassure the reader that we are not risking any circular reasoning here: None of the results that lead to Theorem~\ref{theorem:bisimcharlc} and hence to Proposition~\ref{prop:maxautos} rely on largest autobisimulations. 
\begin{definition}[Bisimulation] \label{def.bisim}
Let $\m = (W,\geq,V)$ and $\m' = (W',\geq',V')$ be plausibility models and let $\m'' = \m \sqcup \m'$ be the disjoint union of the two. Given an autobisimulation $\bisrel$ on $\m''$, if $\bisrel' = \bisrel \cap (W \times W')$ is non-empty, then $R'$ is called a {\em bisimulation} between $\m$ and $\m'$. A bisimulation between $(\m,w)$ and $(\m',w')$ is a bisimulation between $\m$ and $\m'$ containing $(w,w')$.
\end{definition}
\begin{example}\label{ex:m1m2}
Take another look at $\m_C$ and $\m_R$ of Figure~\ref{fig:intro:models}. Let $\m' = \m_C \sqcup \m_R$ and consider possible autobisimulations here. From Figure~\ref{fig:bisim:autobisim} we have the existence of a largest autobisimulation on $\m_R$. For $\m_C$, the largest autobisimulation is just the identity. Naming them $R_R$ and $R_C$ respectively, we (trivially) have that $R_R \cup R_C$ is an autobisimulation on $\m'$. The question is whether we can extend $R_R \cup R_C$ to an autobisimulation on $\m'$ connecting the submodels $M_R$ to $M_C$. We can. This new autobisimulation is $R = R' \cup R_R \cup R_C$, where $R'(u_1) = R'(u_3) = \{v_1\}$, $R'(u_2) = \{v_2\}$ and $R'(u_4) = R'(u_5) = \{v_3\}$. Now we easily get a bisimulation between $\m_R$ and $\m_C$ as $R \cap (D(\m_R) \times D(\m_C)) = R'$. Figure~\ref{fig:bisimmodels} shows the bisimulation $R'$. 
\end{example}
\begin{figure}[t]
		\centering
bisim.\ contr.\ of $\m_R$ \hspace{2.5cm} $\m_C$ \hspace{4.5cm} $\m_R$ \hspace{2cm} \ \\ 
		\def\worldX{1.4cm}
	  	\begin{tikzpicture}
	  		\begin{scope}[xshift = 0cm]
		  		\node[world] (av1) {$\{u_1,u_3\}$}
		  			node[lbl, above = of av1] {$p$};
		  		\node[world, right = \worldX of av1] (av2) {$\{u_2\}$}
		  			node[lbl, above = of av2] {$p$};
		  		\node[world, below = \worldX of av1] (av3) {$\{u_4,u_5\}$}
		  			node[lbl, below = of av3] {$q$};
				\path[link] (av2) \edge{a}{above}{} (av1);
				\path[link] (av1) \edge{b}{left}{} (av3);
			\end{scope}
	  		\begin{scope}[xshift = 5cm]
		  		\node[world] (v1) {$v_1$}
		  			node[lbl, above = of v1] {$p$};
		  		\node[world, right = \worldX of v1] (v2) {$v_2$}
		  			node[lbl, above = of v2] {$p$};
		  		\node[world, below = \worldX of v1] (v3) {$v_3$}
		  			node[lbl, below = of v3] {$q$};
				\path[link] (v2) \edge{a}{above}{} (v1);
				\path[link] (v1) \edge{b}{left}{} (v3);
			\end{scope}
			\begin{scope}[xshift = 10cm]
		  		\node[world] (u1) {$u_1$}
		  			node[lbl, above = of u1] {$p$};
		  		\node[world, right = \worldX of u1] (u2) {$u_2$}
		  			node[lbl, above = of u2] {$p$};
		  		\node[world, right = \worldX of u2] (u3) {$u_3$}
		  			node[lbl, above = of u3] {$p$};
		  		\node[world, below = \worldX of u1] (u4) {$u_4$}
		  			node[lbl, below = of u4] {$q$};
		  		\node[world, below = \worldX of u3] (u5) {$u_5$}
		  			node[lbl, below = of u5] {$q$};
				\path[link] (u2) \edge{a}{above}{} (u3);
				\path[link] (u2) \edge{a}{above}{} (u1);
				\path[link] (u1) \edge{b}{left}{} (u4);
				\path[link] (u3) \edge{b}{left}{} (u5);
				\path[link,ke-ke] (u1) \edge{a}{below}{bend right = 35} (u3);
			\end{scope}
			
			\draw[thick, dotted] (v1)  to [bend left = 28, looseness = 1.14] (u1);
			\draw[thick, dotted] (v2) to [bend right = 25] (u2);
			\draw[thick, dotted] (v1) to [bend left = 38, looseness = 0.70] (u3);
			
			\draw[thick, dotted] (v3) to (u4);
			\draw[thick, dotted] (v3) to [bend right = 38, looseness = 0.70] (u5);

		\end{tikzpicture}
		
		\caption{See Figure \ref{fig:intro:models}. The dotted edges show the largest bisimulation between $\m_C$ and $\m_R$. Model $\m_C$ is isomorphic to the bisimulation contraction of $\m_R$ (on the left) and to the bisimulation contraction of $\m_L$ (not depicted).}
		\label{fig:bisimmodels}
	\end{figure}
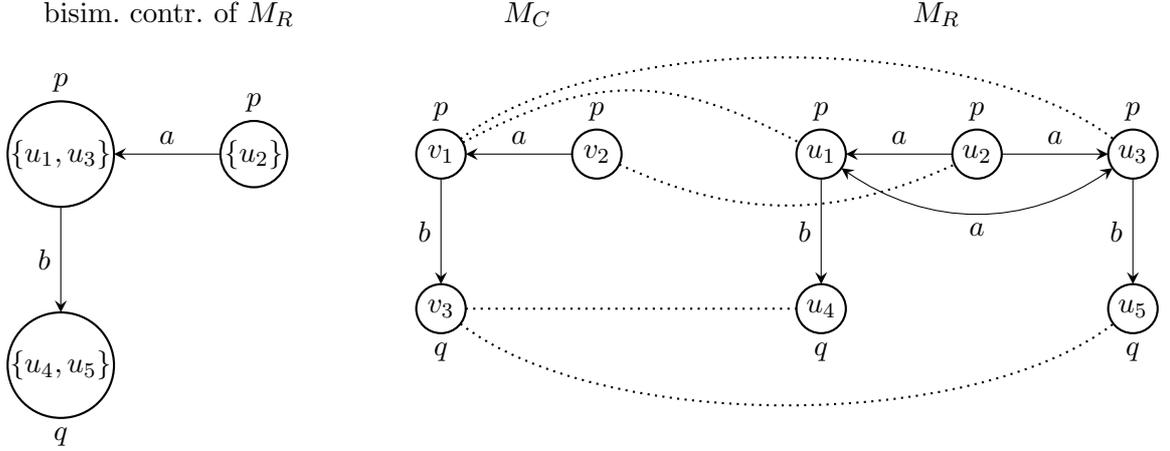
\begin{definition}[Bisimulation contraction]\label{def:contr}
Let $\m = (W,\geq,V)$ be a plausibility model and let $\bisrel$ be the largest autobisimulation on $\m$. The \emph{bisimulation contraction} of $\m$ is the model $\m' = (W',\geq',V')$ such that $W' = \{ [w]_R \mid w \in W \}$, $V'([w]_R) = V(w)$, and for all agents $a$ and worlds $w,v\in W$:
\[
  [w]_R \geq_a' [v]_R \quad \text{ iff } \quad \text{for some $w' \in [w]_R$ and $v' \in [v]_R$: $w' \geq_a^R v'$}. 
\]
\end{definition}
\begin{example}\label{ex:plausibilitymodels:bisimcontraction}
We compute the simulation contraction  $\m_R' = (W', \geq', V')$ of $\m_R = (W, \geq, V)$. For $\geq'_a$ and $\geq'_b$ take the reflexive closures. \[ \begin{array}{lll}
W' & = & \{\{u_1,u_3\}, \{u_2\}, \{u_4, u_5\}\} \\
\geq'_a & = & \{(\{u_2\}, \{u_1,u_3\}) \} \\ % & \union & \text \{(\{u_2\}, \{u_2\}),  \{(\{u_1,u_3\}, \{u_1,u_3\}) \},  \{(\{u_4, u_5\}, \{u_4,u_5\}) \} \} \\
\geq'_b &=& \{ \{u_1,u_3\}, \{u_4, u_5\}) \} \\ % & \union &  \dots \\
V'(\{u_1,u_3\}) &=& \{p\} \\
V'(\{u_2\}) &=& \{p\} \\
V'(\{u_4,u_5\}) &=& \{q\}
\end{array} \]
Model $\m_C$ is isomorphic to both the bisimulation contraction of $\m_L$ and the bisimulation contraction of $\m_R$.
\end{example}
\begin{proposition}
The bisimulation contraction of a plausibility model is a plausibility model and is bisimilar to that model.
\end{proposition}
This proposition is not hard to prove, so we do not provide a full proof, but only sketch the overall idea. First, to prove that the bisimulation contraction $(W',\geq',V')$ of a plausibility model $(W,\geq,V)$ is a plausibility model, we simply have to prove that the relations $\geq'_a$ are well-preorders on each $\sim'_a$ equivalence class. That is shown by first proving reflexivity of $\geq'_a$, then transitivity of $\geq'_a$, and finally by proving that any non-empty subset has minimal elements with respect to $\geq'_a$. To show that $(W,\geq,V)$ is bisimilar to $(W',\geq',V')$, we define the (functional) relation $S: W \to W'$ as $S = \{(w, [w]_R) \mid w \in W \}$ and show that this is a bisimulation relation (that it satisfied [atoms] and the [back] and [forth] conditions). 
\begin{definition}[Normal plausibility relation, normal model] \label{def:normalplaus}
Let $\m = (W,\geq,V)$ be a plausibility model and let $R$ be the largest autobisimulation on $\m$. For all agents $a$, the relation $\geq_a^R$ is the {\em normal plausibility relation} for agent $a$ in $\m$, for which we may also write $\succeq_a$.  The model is {\em normal} if for all $a$, ${\geq_a} = {\geq^R_a}$. Any model $\m$ can be \emph{normalised} by replacing all $\geq_a$ by $\geq_a^R$.
\end{definition}
\begin{example}%\label{ex:autobisim}
Consider again the models $M_L$ and $M_R$ of Figure~\ref{fig:bisim:autobisim}. From the largest bisimulation on $M_L$ (shown by dotted edges), we can conclude that $M_R$ is the normalisation of $M_L$ (modulo a renaming of the worlds $w_i$ to $u_i$, for $i= 1,\dots,5$).
%The models $\m_L$ and $\m_R$ of Figure~\ref{fig:intro:models} are reproduced in Figure~\ref{fig:normal:models}, along with the normalisation of $\m_L$, based on the previously seen largest autobisimulation. The largest autobisimulation shown for $M_R$ is completely analogous to the one we have seen for $M_L$.
\end{example}
\begin{proposition}
The bisimulation contraction of a plausibility model is normal.
\end{proposition}
\begin{proof}
Let $\m$ be a plausibility model, and let $\m' = (W', \geq', V')$ be the bisimulation contraction of $\m$.  The largest autobisimulation on $\m'$ is the identity relation $R_{id}$. For each agent $a$, we now have that ${{\geq'}^{R_{id}}_a} = {\geq'_a}$. Therefore, $\m'$ is normal.  
\end{proof}

\section{Logical language and semantics} \label{section:semantics}
In this section we define the language and the semantics of our logics.
\begin{definition}[Logical language]
For any countably infinite set of propositional symbols $\props$ and finite set of agents $\agents$ we define language $\lang^{CDS}_{\props\agents}$ by:
\[
	\phi ::= p \mid \neg \phi \mid \phi \wedge \phi \mid K_a \phi \mid B^\phi_a \phi
\mid B^n_a \phi \mid \Box_a \phi
\]
	where $p \in \props$, $a \in \agents$, and $n \in \Naturals$.  
\end{definition}
The formula $K_a \phi$ stands for `agent $a$ knows (formula) $\phi$', $B^\psi_a \phi$ stands for `agent $a$ believes $\phi$ on condition $\psi$', $B^n_a \phi$ stands for `agent $a$ believes $\phi$ to degree $n$', and $\Box_a \phi$ stands for `agent $a$ safely believes $\phi$'. (The semantics of these constructs is defined below.) The duals of $K_a$, $B^\phi_a$ and $\Box_a$ are denoted $\widehat{K}_a$, $\widehat{B}^\phi_a$ and $\Diamond_a$. We use the usual abbreviations for the boolean connectives as well as for $\top$ and $\bot$, and the abbreviation $B_a$ for $B^\top_a$. In order to refer to the type of modalities in the text, we call $K_a$ a {\em knowledge modality}, $B^\psi_a$ a {\em conditional belief modality}, $B^n_a$ a {\em degree of belief modality}, and $\Box_a$ a {\em safe belief modality}. 

In $\lang^{CDS}_{\props\agents}$, if $\agents$ is clear from the context, we may omit that and write $\lang^{CDS}_\props$, and if $\props$ is clear from the context, we may omit that as well, so that we get $\lang^{CDS}$. The letter $C$ stands for `conditional', $D$ for `degree', and $S$ for `safe'. Let $X$ be any subsequence of $CDS$, then $\lang^X$ is the language with, in the inductive definition, only the modalities $X$ (and with knowledge $K_a$) for all agents. In our work we focus on the {\em logic of conditional belief} with language $\lang^C$, the {\em logic of degrees of belief} with language $\lang^D$, and the {\em logic of safe belief} with language $\lang^S$.
\begin{definition}[Satisfaction Relation]\label{defi:sat-rel}
	Let $\m = (W,\geq,V)$ be a plausibility model for $\props$ and $\agents$, let $\succeq$ be the normal plausibility relation for $\m$, and let $w \in W$, $p \in \props$, $a \in \agents$, and $\phi, \psi \in \lang^{CDS}$. Then:
	\[
	  \begin{array}{lll}
	  		\m, w \models p & \text{iff} & p \in V(w) \\ 
			\m, w \models \neg \phi & \text{iff} & \m, w \not\models \phi \\
			\m, w \models \phi \wedge \psi & \text{iff} & \m, w \models \phi \textrm{ and } \m, w \models \psi \\
			\m, w \models K_a \phi & \text{iff}  & \m, v \models \phi \textrm{ for all } v \in [w]_a \\
			\m, w \models B_a^\psi \phi & \text{iff} & \m, v \models \phi \textrm{ for all } v \in \Min_a ( \llbracket \psi \rrbracket_\m \cap [w]_a) \\
\m,w \models B^n_a \phi & \text{iff} & \m,v \models \phi \text{ for all } v \in \Min^n_a [w]_a \\
\m,w \models \Box_a \phi & \text{iff} & \m,v \models \phi \text{ for all } v \text{ with } w \succeq_a v
% \xi\in\lang^{CD} \text{ such that } \m,w \models \xi\text{ and } v \in \Min_a ( \llbracket \xi \rrbracket_\m \cap [w]_a)  
	 \end{array}
   \]	
where 
\[ \begin{array}{lll}
\Min^0_a [w]_a  & = & \Min_{\succeq_a}  [w]_a \\
\Min^{n+1}_a [w]_a  & = & \left\{
  \begin{array}{ll} [w]_a \hfill \text{if } \Min^n_a [w]_a = [w]_a \\ \Min^n_a [w]_a  \union \Min_{\succeq_a} ([w]_a \setminus \Min^n_a [w]_a)   \hspace{2.5cm} \hfill \text{otherwise} \end{array} \right.
\end{array} \]
and where $\llbracket \phi \rrbracket_\m = \{ w \in W \mid \m,w \models \phi\}$. 
\end{definition}
We write  $\m \models \phi$ ($\phi$ is valid on $\m$) to mean that $\m,w \models \phi$ for all $w\in W$. 
\begin{definition}[Modal equivalence]
Consider the language $\lang^X_\props$, for $X$ a subsequence of $CDS$. Given are models $\m = (W,\geq,V)$ and $\m' = (W',\geq',V')$, and $w\in W$ and $w' \in W'$. We say that $(\m,w)$ and $(\m',w')$ are {\em modally equivalent} in $\lang^X_\props$, notation $(\m,w) \equiv^X_\props (\m',w')$, if and only if for all $\phi \in \lang^X_P$, $\m,w \models \phi$ if and only if $\m',w' \models \phi$. If $\props$ is obvious from the context we may write $(\m,w) \equiv^X (\m',w')$.
\end{definition}

\subsection*{The logic of conditional belief} \label{section:lcond}
The logic $L^C$ of conditional belief appears in \cite{stalnaker:1996,baltagetal.tlg:2008,jfak.jancl:2007,balt.ea:qual}, where particularly the latter two are foundational for dynamic belief revision (older roots are Lewis' counterfactual conditionals \cite{lewis:1973}). An axiomatisation is found in \cite{stalnaker:1996}. In this logic, defeasible belief $B_a \phi$ is definable as $B^\top_a \phi$, while $K_a \phi$ is definable as $B^{\neg\phi}_a \bot$.
\begin{example}
Consider Figure~\ref{fig:intro:models}. In the plausibility model $M_C$ we have, for instance: $M_C \models K_a p \rightarrow (B_a B_b q \land \neg K_a B_b q)$: If $a$ knows $p$ (true in $v_1$ and $v_2$), $a$ believes, but does not know, that $b$ believes $q$. Another example is $M_C \models B_a^{\neg B_b q} K_b \neg q$: Conditional on $b$ not believing $q$, $a$ believes that $b$ knows $\neg q$. Only in $v_2$ does $\neg B_b q$ hold; there $K_b \neg q$ holds. A final example is $M_C \models K_a p \rightarrow B_a^{\widehat{K}_b q} B_b q$: From $v_1$ and $v_2$ (where $K_a p$ holds), formula $\widehat{K}_b q$ only holds in $v_1$, and conditional to that, the one and only most plausible world $v_1$ satisfies $B_b q$. We can repeat this exercise in $M_L$ and $M_R$, as all three models are bisimilar and therefore, as will be proved in the next section, logically equivalent.
\end{example}

\subsection*{The logic of degrees of belief} \label{section:ldegrees}
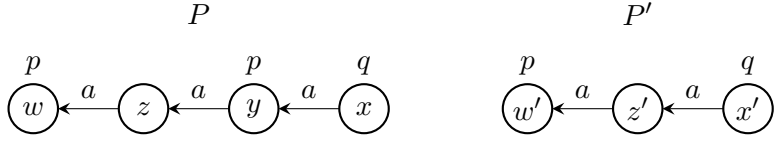
\begin{figure}[t]
		\centering
		\def\worldX{0.8cm}
	  	\begin{tikzpicture}
	  		\begin{scope}[xshift = 0cm]
		  		\node[world] (w) {$w$}
		  			node[lbl, above = of w] {$p$};
		  		\node[world, right = \worldX of w] (z) {$z$};
%		  			node[lbl, above = of z] {$p$};
		  		\node[world, right = \worldX of z] (y) {$y$}
		  			node[lbl, above = of y] {$p$};
		  		\node[world, right = \worldX of y] (x) {$x$}
		  			node[lbl, above = of x] {$q$};
				\path[link] (x) \edge{a}{above}{} (y);
				\path[link] (y) \edge{a}{above}{} (z);
				\path[link] (z) \edge{a}{above}{} (w);
			\end{scope}
	  		\begin{scope}[xshift = 6.5cm]
		  		\node[world] (w') {$w'$}
		  			node[lbl, above = of w'] {$p$};
		  		\node[world, right = \worldX of w'] (z') {$z'$};
		  		\node[world, right = \worldX of z'] (x') {$x'$}
		  			node[lbl, above = of x'] {$q$};
				\path[link] (x') \edge{a}{above}{} (z');
				\path[link] (z') \edge{a}{above}{} (w');
			\end{scope}	
			
			\node[above = 0.8cm of w] (pl) {}; 			
			\node[above = 0.8cm of x] (pr) {}; 
			\path (pl) -- (pr) node[midway] (lblp) {$P$};	
			
			\node[above = 0.8cm of w'] (p'l) {}; 			
			\node[above = 0.8cm of x'] (p'r) {}; 
			\path (p'l) -- (p'r) node[midway] (lblpp) {$P'$};

		\end{tikzpicture}
		\caption{A plausibility model $P$ and its bisimulation contraction $P'$.}
		\label{fig:naivebebeliefmodels}
\end{figure}
The logic $L^D$ of degrees of belief, also known as the logic of graded belief, goes back to Grove~\cite{grove:1988} and Spohn~\cite{spohn:1988}, although these could more properly be said to be semantic frameworks to model degrees of belief (there is no relation between the logic of degrees of belief and Fine's logic of graded belief \cite{fine:1972} and subsequent works, wherein we count the number of pairs $(v,w) \in R$ between two worlds $v$ and $w$, or, alternatively, label the accessibility relation with that number). Logics of degrees of belief have seen some popularity in artificial intelligence and AGM style belief revision, see e.g.\ \cite{hoek:1992,hoek:1993,laverny:2006}. Belief revision based on degrees of belief have been proposed by \cite{aucher:2005a,hvd.prolegomena:2005}. The typical distinction between conviction (arbitrarily strong belief) and knowledge, as in \cite{lenzen:1978,lenzen:2003}, is absent in our logic $L^D$, wherein the strongest form of belief defines knowledge. Reasoning with degrees of belief is often called quantitative, where conditional belief can then be called qualitative. In other communities both are called qualitative, and quantitive epistemic reasoning approaches are in that case those that combine knowledge and probabilities \cite{halpern:2003}. The zeroth degree of belief $B_a^0 \phi$ defines defeasible belief $B_a \phi$. How Spohn's work relates to dynamic belief revision as in \cite{baltagetal.tlg:2008} is discussed in detail in \cite{hvd.comments:2008}. There have also been various proposals combining knowledge and belief ($B^\T_a \phi$ or $B^0_a \phi$) in a single framework, without considering either conditional or degrees of belief, where the dynamics are temporal modalities, see \cite{klm:1990,krausetal:1988,friedmanetal:1994}. For purposes of further discussions and the proofs in Section \ref{section:cdegrees} we define belief layers as follows:
\begin{definition}[Belief Layers]
Let $\m = (W, \geq, V)$. For $w \in W$, $a \in \agents$ and $n \in \Naturals$, the $n$th (belief) layer of $w$ for $a$ is defined as $E^{n}_a [w]_a = \Min_{\succeq_a} ([w]_a \setminus \Min_a^{n-1} [w]_a)$, where we use the special case $\Min_a^{-1} [w]_a = \emptyset$.
\end{definition}
This immediately gives the following lemma:
\begin{lemma}
For $\m = (W, \geq, V)$, $w \in W$, $a \in \agents$ and $n \in \Naturals$, we have $\Min^n_a [w]_a = E^{n}_a [w]_a \cup \Min_a^{n-1} [w]_a$. For $n$ such that $\Min_a^n [w]_a = [w]_a$ we have $E_a^{n+1} [w]_a = \emptyset$. We name the smallest such $n$ the \emph{maximum degree} (for $a$ at $w$). If $n$ is the maximum degree for $a$ at $w$, we have $M, w \models K_a \phi \leftrightarrow B_a^n \phi$.
\end{lemma}
In \cite{aucher:2005a,hvd.prolegomena:2005,laverny:2006} different layers can contain bisimilar worlds. In our approach they cannot, because we define belief layers on the normal plausibility relation. Unlike \cite{spohn:1988} our semantics does not allow empty layers in between non-empty layers. If $E_a^n [w]_a \neq \emptyset$ and $E_a^{n+2} [w]_a \neq \emptyset$, then $E_a^{n+1} [w]_a \neq \emptyset$. Layers above the maximum degree will be empty, i.e.\ if there is a maximum degree $n$ for $a$ at $w$, as there will always be in our (pre)image-finite models, then for all $k > n$, we have $E_a^k [w]_a = \emptyset$.
\begin{example}
In Figure~\ref{fig:intro:models}, we have that $M_C \models B_a^0 B^0_b q$ but not $M_C \models B_a^1 B^0_b q$. The maximum degree of belief for $a$ in $M_C$ is at either $v_1$ and $v_2$, where it is $1$, so $M_C \models K_a \phi \leftrightarrow B_a^1 \phi$. This is also true in the other two models. Consider now the models $P$ and $P'$ in Figure~\ref{fig:naivebebeliefmodels} and an alternative definition of $B^n_a$ not using $\succeq_a$ but $\geq_a$ (as in \cite{aucher:2005a,hvd.prolegomena:2005,laverny:2006,balt.ea:qual}). In the $\geq_a$-semantics we have $P \models B_a^2 \neg q$, as $q$ is false in $\{y,z,w\}$. Only when we reach the third degree of belief does $q$ become uncertain: $P \not\models B_a^3 \neg q$. With $\succeq_a$-semantics, $2$ is the maximum degree so $P \not\models B_a^2 \neg q$. This can be seen in the bisimilar model $P'$, where $P' \not\models B_a^2 \neg q$.
\end{example}

\subsection*{The logic of safe belief} \label{section:lsafe}
The logic $L^S$ of safe belief goes back to Stalnaker \cite{stalnaker:1996} and has been progressed by Baltag and Smets (for example, how it relates to conditional belief and knowledge) in \cite{balt.ea:qual}, which also gives a detailed literature review involving the roots of conditional belief, degrees of belief, and safe belief. An agent has {\em safe belief} in a formula $\phi$ iff the agent will continue to believe $\phi$ no matter what {\em true} information conditions its belief, i.e. $M, w \models \Box_a \phi$ iff $M, w \models B_a^\psi \phi$ for all $\Box$-free $\psi$ s.t. $M,w \models \psi$. In \cite{balt.ea:qual} safe belief is defined as $M, w \models \Box_a \phi$ iff $M, v \models \phi$ for all $v$ s.t. $w \geq_a v$. For both \cite{stalnaker:1996} and \cite{balt.ea:qual} true information are subsets of the domain containing the actual world. When this is what true information is, there is a correspondence between the two definitions, as indeed noted by Baltag and Smets. The complications of this choice are addressed in detail in \cite{demey:2011}. For us, there is not a correspondence between the two definitions, because we can only condition on modally definable subsets. When we, as we do, define safe belief using $\succeq_a$, this correspondence is reestablished.
\begin{example}
Consider for a final time the models of Figure~\ref{fig:intro:models}. We have $M_C, v_1 \models \Box_a \widehat{K}_b q$, whereas $M_C, v_2 \not\models \Box_a \widehat{K}_b q$. Now consider $M_L$ and the $\geq_a$-version of safe belief  for which we have $M_L,w_3 \not \models \Box_a \widehat{K}_b q$. For \cite{stalnaker:1996,balt.ea:qual} this is as it should be: For the subset $\{w_2,w_3\}$ (which includes the actual world $w_3$ as required) we have $\Min_a (\{w_2,w_3\} \cap [w_3]_a) = \{w_2\}$ where $M_L, w_2 \not\models \widehat{K}_b q$. Using the $\succeq_a$-version of safe belief, we have $M_L,w_3 \models \Box_a \widehat{K}_b q$. For us, this is as it should be: As our conditional belief picks using $\llbracket \psi  \rrbracket \cap [w]_a$, any set containing $w_3$ must include the modally equivalent world $w_1$. This corresponds to first normalising $M_L$ to get $M_R$. In \emph{that} model, $u_2$ is strictly less plausible than $u_3$. 
\end{example}
The semantics we propose for degrees of belief and safe belief are non-standard. Still, as we show in the following, these non-standard semantics and the standard semantics for conditional belief are all bisimulation invariant. This makes the results in Section \ref{section:expressivity} showing a non-trivial expressivity hierarchy between these logics even more remarkable. 

\section{Bisimulation characterisation for $\lang^C$, $\lang^D$ and $\lang^S$} \label{section:corr}

\subsection{Bisimulation correspondence for conditional belief} \label{section:ccond}
In the following we prove that for the language $L^C$ bisimilarity implies modal equivalence and vice versa. This shows that our notion of bisimulation is proper for the language and models at hand. The proof of Proposition~\ref{theorem:bisimimpliesmodalequivlc} below is essentially a standard proof of bisimilarity implying modal equivalence: modal equivalence is proved by induction on the structure of the formula, where in the induction step the back and forth conditions of bisimilarity are applied to the induction hypothesis. However, the induction case of conditional belief formulas $B_a^\gamma \psi$ is a bit more involved than for standard modalities. Additional work is needed to ensure that when applying the back and forth conditions we produce a world which is among the minimal $\gamma$-worlds.
\begin{proposition}\label{theorem:bisimimpliesmodalequivlc}
Bisimilarity implies modal equivalence for $\lang^C$. 
\end{proposition}
\begin{proof}
\newcommand{\tbisrel}{R}
\newcommand{\mina}{\Min_{\leq_a}}
\newcommand{\sima}{{\sim_a}}
\newcommand{\leqa}{\leq_a}
\newcommand{\leqar}{\leq_a^\tbisrel}
\newcommand{\leqr}{\leq^\tbisrel}
\newcommand{\llgamma}{\llbracket \gamma \rrbracket}
Assume $(\m_1,w_1) \bisim (\m_2,w_2)$. Then, by definition, there exists an autobisimulation $R$ on the disjoint union of $\m_1$ and $\m_2$ with $(w_1,w_2) \in R$. Let $\m = (W,\geq,V)$ denote the disjoint union of $\m_1$ and $\m_2$. We then need to prove that $(\m,w_1)$ and $(\m,w_2)$ are modally equivalent in $L_C$.  We will show that for all $\phi$ in $L_C$, for all $(w,w') \in \tbisrel$, if $\m,w \models \phi$ then $\m,w' \models \phi$. This implies the required (the other direction being symmetric). % ($\tbisrel$ is symmetric). 
 The proof is by induction on the syntactic complexity of $\phi$. The propositional cases are easy, so we only consider the cases $\phi = K_a \psi$ and $\phi = B^\gamma_a \psi$. Consider first $\phi = K_a \psi$. In this case we assume $\m, w \models K_a \psi$, that is, $\m, v \models \psi$ for all $v$ with $w \sim_a v$.   Let $v'$ be chosen arbitrarily with $w' \sim_a v'$. We need to prove $\m, v' \models \psi$. From Lemma~\ref{lemma:wellr} we have that $\geq_a^R$ is a well-preorder on each $\sim_a$-equivalence class. Since $w' \sim_a v'$ we hence get that either $w' \geq_a^R v'$ or $v' \geq_a^R w'$. We can assume $w' \geq_a^R v'$, the other case being symmetric. Then since $(w,w') \in \tbisrel$ and $w' \geq_a^R v'$, [back$_\geq$] gives us a $v$ s.t.\ $(v,v') \in \tbisrel$ and $w \geq_a^R v$. Lemma~\ref{lemma:wellr} now implies $w \sim_a v$, and hence $\m, v \models \psi$. Since $(v,v') \in \tbisrel$, the induction hypothesis gives us $\m, v' \models \psi$, and we are done.
 
 Now consider the case $\phi = B_a^\gamma \psi$. This case is more involved. Assume $\m,w \models B_a^\gamma \psi$, that is, $\m, v \models \psi$ for all $v \in \Min_a ( \llbracket \gamma \rrbracket_\m \cap [w]_a)$. Letting $v' \in \Min_a (\llgamma_\m \cap [w']_a)$, we need to show $\m,v' \models \psi$ (if $\Min_a (\llgamma_\m \cap [w']_a)$ is empty there is nothing to show). 
 We will first find a $y$ in $\Min_a (\llgamma_\m \cap [w]_a)$, then find a $y'$ with $(y,y') \in \tbisrel$, and only then apply [back$_\geq$] to $y' \geq_a^R v'$ to produce the required $v$. The point is that our choice of $y$ in $\Min_a (\llgamma_\m \cap [w]_a)$ will ensure that $v$ is in $\Min_a (\llgamma_\m \cap [w]_a)$.

As mentioned, we want to start out choosing a $y$ in $\Min_a (\llgamma_\m \cap [w]_a)$, so we need to ensure that this set is non-empty. By choice of $v'$ we have $v' \in \llgamma_\m$ and $v' \sim_a w'$. From $v' \sim_a w'$ we get that $w' \geq_a^R v'$ or $w' \leq_a^R v'$, using Lemma~\ref{lemma:wellr}. Since also $(w,w') \in R$, we can apply [back$_\geq$] or [back$_\leq$] to get a $u$ such that $(u,v') \in R$ and either $w \geq^R_a u$ or $w \leq_a^R u$. From $(u,v') \in R$ and $v' \in \llgamma_\m$, we get $u \in \llgamma_\m$, using the induction hypothesis. From the fact that either $w \geq^R_a u$ or $w \leq_a^R u$ we get $w \sim_a u$, using Lemma~\ref{lemma:wellr}. Hence we have $u \in \llgamma_\m \cap [w]_a$. This shows the set $\llgamma_\m \cap [w]_a$ to be non-empty. 
Hence also $\Min_a  (\llgamma_\m \cap [w]_a)$ is non-empty, and we are free to choose a $y$ in that set. Since $y \sim_a w$, Lemma~\ref{lemma:wellr} gives us that
 either $y \geq_a^R w$ or $w \geq_a^R y$, so we can apply [forth$_\leq$] or [forth$_\geq$] to find a $y'$ with $(y,y') \in \tbisrel$ and either $y' \geq_a^R w'$ or $w' \geq_a^R y'$. 
 
\medskip \noindent \textit{Claim 1.} $y' \geq_a^R v'$.

\medskip
\noindent \textit{Proof of claim 1.} %Since $(y,y') \in \tbisrel$ and $y \in \llgamma$, the induction hypothesis gives us $y' \in \llgamma$. 
We need to prove $\Min_a([y']_{R^=} \cap [y']_a) \geq_a \Min_a ([v']_{R^=} \cap [v']_a)$. We first prove that $[y']_{R^=} \cap [y']_a \subseteq \llgamma_\m \cap [w']_a$: 
\begin{itemize}
  \item {$[y']_{R^=} \cap [y']_a \subseteq \llgamma_\m$:} Assume $y'' \in [y']_{R^=} \cap [y']_a$. Then $(y',y'') \in R^=$. Since we also have $(y,y') \in R$, we get $(y,y'') \in R^=$. From $(y,y'') \in R^=$ and $y \in \llgamma_\m$ a finite sequence of applications of the induction hypothesis gives us $y'' \in \llgamma_\m$. 
  \item {$[y']_{R^=} \cap [y']_a \subseteq [w']_a$:} Assume $y'' \in [y']_{R^=} \cap [y']_a$. Then $y'' \sim_a y'$. Since we have either $y' \geq_a^R w'$ or $w' \geq_a^R y'$, we must also have $y' \sim_a w'$, by Lemma~\ref{lemma:wellr}. Hence $y'' \sim_a y' \sim_a w'$ implying $y'' \in [w']_a$.
\end{itemize}
Since $v'$ is chosen minimal in $\llgamma_\m \cap [w']_a$ and $[y']_{R^=} \cap [y']_a  \subseteq \llgamma_\m \cap [w']_a$ we get 
$\Min_a ( [y']_{R^=} \cap [y']_a) \geq_a \{ v' \} \geq_a \Min_a ([v']_{R^=} \cap [v']_a)$, as required. This concludes the proof of the claim.

\medskip \noindent
By choice of $y'$ we have $(y,y') \in R$, and by Claim 1 we have $y' \geq_a^R v'$. We can now finally, as promised, apply [back$_\geq$] to these premises to get a $v$ s.t.\ $(v,v') \in \tbisrel$ and $y \geq_a^R v$. 

\medskip \noindent \textit{Claim 2.}  $\Min_a ([v]_{R^=} \cap [v]_a) \subseteq \Min_a (\llgamma_\m \cap [w]_a)$.

\medskip
\noindent \textit{Proof of claim 2.} Let $x \in \Min_a ([v]_{R^=} \cap [v]_a)$. We need to prove $x \in \Min_a (\llgamma_\m \cap [w]_a)$. We do this by proving $x \in \llgamma_\m$, $x \in [w]_a$ and $\{ x \} \leq_a \Min_a (\llgamma_\m \cap [w]_a)$: %. Proof of $x \in \llgamma \cap [w]$: 
\begin{itemize}
  \item {$x \in \llgamma_\m$:} By choice of $x$ we have $(v,x) \in R^=$.
  From  $(v,x) \in R^=$ and $(v,v') \in \tbisrel$ we get $(v',x) \in R^=$. From $(v',x) \in R^=$ and $v' \in \llgamma_\m$ a finite sequence of applications of the induction hypothesis gives us $x \in \llgamma_\m$.
  \item {$x \in [w]_\m$:}  By choice of $x$ we have $x \sim_a v$. Since $y \geq_a^R v$, Lemma~\ref{lemma:wellr} implies $v \sim_a y$. By choice of $y$ we have $y \sim_a w$, so in total we get $x \sim_a v \sim_a y \sim_a w$, as required.
  \item {$\{ x \} \leq_a \Min_a( \llgamma_\m \cap [w]_a)$:} 
  \[ \begin{array}{rll}
    \{ x \} & \leq_a \Min_a ( [v]_{R^=} \cap [v]_a) \qquad & \text{by choice of $x$} \\
       & \leq_a \Min_a ([y]_{R^=} \cap [y]_a) & \text{since $y \geq_a^R v$} \\
       & \leq_a \{ y \} &  \\ 
       & \leq_a \Min_a( \llgamma_\m \cap [w]_a) & \text{since $y \in \Min_a ( \llgamma_\m \cap [w]_a)$}.
     \end{array}
     \]
\end{itemize}
This concludes the proof of the claim.

\medskip \noindent
Now we are finally ready to prove $\m, v' \models \psi$. Let $z \in \Min_a ([v]_{R^=} \cap [v]_a).$ Then $z \in \Min_a( \llgamma_\m \cap [w]_a)$, by Claim 2. Hence $\m,z \models \psi$, by assumption. Since $(v,z) \in R^=$ and $(v,v') \in R$ we get $(z,v') \in R^=$, and hence a finite sequence of applications of the induction hypothesis gives us $\m,v' \models \psi$.
\end{proof}
We proceed now to show the converse, that modal equivalence with regard to $\lang^C$ implies bisimulation. The proof has the same structure as the Hennessy-Millner approach, though appropriately modified for our purposes. Given a pair of image-finite models $\m$ and $\m'$, the standard approach is to construct a relation $R \subseteq \Domain(\m) \times \Domain(\m')$ s.t.\ $(w,w') \in R$ if $\m, w \equiv \m',w'$. Using $\Diamond$-formulas, it is then shown that $R$ fulfils the requirements for being a bisimulation, as such formulas denote what is true at worlds accessible by whatever accessibility relation is used in the model. This means that modally equivalent worlds have modally equivalent successors, which is then used to show that $R$ fulfils the required conditions. For our purposes this will not do, as we only have $\widehat{K}_a$-formulas (i.e. for $\sim_a$). Instead, our equivalent to $\Diamond$-formulas are of the form $\widehat{B}^\psi_a \phi$, each such formula corresponding to accessibility to the most plausible $\psi$-worlds from all worlds in an equivalence class. What we want are formulas corresponding to specific links between worlds, so we first establish that such formulas exists. We thus have formulas with the same function as $\Diamond$-formulas serve in the standard approach.
\begin{proposition}\label{theorem:modalequivimpliesbisimlc}
 Modal equivalence with respect to $\lang^C$ implies bisimilarity.
\end{proposition}
\begin{proof}
Assume $(\m_1, w) \equiv^C (\m_2, w')$. We wish to show that $(\m_1, w) \bisim (\m_2, w')$. Let $\m = \m_1 \sqcup \m_2$ be the disjoint union of $\m_1$ and $\m_2$. We then need to show that $Q = \{(v,v') \in \dom\m \times \dom\m \mid \m, v \equiv^C \m, v'\}$ is an autobisimulation on $\m$. Note that as $\equiv^C$ is an equivalence relation, so is $Q$. We first show that $\Diamond$-like formulas talking about the $\geq_a^Q$-relations between specific worlds in $\m$ exist.

\medskip \noindent \textit{Claim 1.}
Let $w$ and $w'$ be worlds of the model $\m = (W, \geq, V)$ where $w \geq^Q_a w'$. Further let $\phi \in \lang^C$ be any formula true in $w'$. There then exists a formula $\psi \in \lang^C$ such that $([w]_Q \cup [w']_Q) \cap [w]_a = \llbracket \psi \rrbracket_\m \cap [w]_a$ and $\m, w \models \widehat{B}_a^{\psi} \phi$.

\medskip
\noindent \textit{Proof of Claim 1.}
If two worlds $s$ and $s'$ are not modally equivalent, there exists some distinguishing formula $\Psi_{s,s'}$ with $\m, s \models \Psi_{s,s'}$ and $\m, s' \not\models \Psi_{s,s'}$. As $\sim_a$ is image-finite (since both $\geq_a$ and its converse are) the following formula is finite:
\begin{align*}
\Psi_t = \bigwedge \{ \Psi_{t,t'} \mid t \sim_a t' \land (t, t') \not \in Q \}
\end{align*}
The formula $\Psi_t$ distinguishes $t$ from all the worlds in $[t]_a$ that it is not modally equivalent to. If there are no such worlds, $\Psi_t$ is the empty conjunction equivalent to $\top$.

We now return to our two original worlds $w$ and $w'$. With the assumption that $\m, w' \models \phi$, we show that $\psi = \Psi_w \lor \Psi_{w'}$ is a formula of the kind whose existence we claim. First note that $\llbracket \Psi_w \rrbracket_\m \cap [w]_a$ contains only those worlds in $[w]_a$ that are modally equivalent to $w$, exactly as $[w]_Q \cap [w]_a$ does. As $\llbracket \Psi_w \rrbracket_\m \cup \llbracket \Psi_{w'} \rrbracket_\m = \llbracket \Psi_w \lor \Psi_{w'} \rrbracket_\m$ we have $([w]_Q \cup [w']_Q) \cap [w]_a = \llbracket \Psi_w \lor \Psi_{w'} \rrbracket_\m \cap [w]_a$. To get $\m, w \models \widehat{B}_a^{\psi} \phi$ we need to show that $\exists v \in \Min_a ( \llbracket {\Psi_w \lor \Psi_{w'}} \rrbracket_\m \cap \equivclass{w}{a})$ s.t. $\m, v \models \phi$. Pick an arbitrary $v \in \Min_a([w']_Q \cap [w']_a)$. We will now show that this has the required properties.

Let $T = \llbracket \Psi_w \lor \Psi_{w'} \rrbracket_\m \cap \equivclass{w}{a}$. Since $T = ([w]_Q \cup [w']_Q) \cap [w]_a$, Lemma \ref{lemma:equiplausr} gives $u \simeq_a^Q w$ or $u \simeq_a^Q w'$ for all $u \in T$. Together with $w \geq_a^Q w'$, this gives $w' \in \Min_{\geq_a^Q} T$. Choose $u \in T$ arbitrarily. We then have $u \geq_a^Q w'$ and, by definition, that $\Min_a([u]_Q \cap [u]_a) \geq_a \Min_a([w']_Q \cap [w']_a)$. By choice of $v$ we can then conclude $\{v\} \leq_a \Min_a([w']_Q \cap [w']_a) \leq_a \Min_a ([u]_Q \cap [u]_a) \leq_a \{u\}$. As $u$ was chosen arbitrarily in $T$, this shows $v \in \Min_a T$. As $v \in [w']_Q$ we have $\m, v \equiv^C \m, w'$ and by assumption of $\m, w' \models \phi$ that $\m, v \models \phi$. We now have $v \in \Min_a ( \llbracket {\Psi_w \lor \Psi_{w'}} \rrbracket_\m \cap \equivclass{w}{a})$ and $\m, v \models \phi$, completing the proof of the claim.
\medskip

\medskip\noindent
We now proceed to show that $Q$ fulfils the conditions for being an autobisimulation on $\m$ (Definition \ref{def:autobisim}).  [atoms] is trivial. Next we show [forth$_\geq$]. 
Let $(w,w') \in Q$ (i.e. $(\m,w) \equiv^C (\m,w')$) and $w \geq^Q_a v$. We then have that [forth$_\geq$] is fulfilled if $\exists v' \in W$, s.t. $w' \geq^Q_a v'$ and $(v,v') \in Q$ (i.e. $(\m,v) \equiv^C (\m,v')$). To this end, we show that assuming for all $v' \in W$, $w' \geq^Q_a v'$ implies $(\m,v) \not\equiv^C (\m,v')$ leads to a contradiction. This is analogous to how $Q$ is shown to be a bisimulation in standard Hennessy-Millner proofs.

We first show that $\geq_a^Q$ is image-finite. First recall that by assumption on plausibility models, $\geq_a$ is (pre)image finite, that is, both $\geq_a$ and $\leq_a$ are image-finite. It follows that $\sim_a \ =\ \geq_a \cup \leq_a$ is image-finite as well. If a relation is image-finite, then so is any subset of the relation. Therefore, as $\geq_a^Q\ \subseteq\ \sim_a$, $\geq_a^Q$ must be image-finite. Hence the set of $\geq_a^Q$-successors of $w'$, $S = \{ v' \mid w' \geq^Q_a v' \} = \{ v'_1, \dots , v'_n \}$ is also finite. Having assumed that $v$ and none of the $v'_i$s are modally equivalent, we have that there exists a number of distinguishing formulae $\phi^{v'_i}$, one for each $v'_i$, such that $\m, v \models \phi^{v'_i}$ and $\m, v'_i \not \models \phi^{v'_i}$. Therefore, $\m, v \models \phi^{v'_1} \land \dots \land \phi^{v'_n}$. For notational ease, let $\phi = \phi^{v'_1} \land \dots \land \phi^{v'_n}$. 

With $\m, v \models \phi$, Claim 1 gives the existence of a formula $\psi$, such that $([w]_Q \cup [v]_Q) \cap [w]_a = \llbracket \psi \rrbracket_\m \cap [w]_a$ and $\m, w \models \widehat{B}_a^{\psi} \phi$. Due to modal equivalence of $w$ and $w'$, we must have $\m, w' \models \widehat{B}_a^{\psi} \phi$. This we have iff $\exists u' \in \Min_a (\llbracket \psi \rrbracket_\m \cap [w']_{a})$, s.t. $\m, u' \models \phi$. By construction of $\phi$, no world $v'_i$ exists such that $w' \geq^Q_a v'_i $ and $\m, v'_i \models \phi$, so we must have $u' >^Q_a w'$. As $u' \in [w']_a$, the definition of $>^Q_a$ gives $\Min_a ([u']_Q \cap [w']_a) >_a \Min_a ([w']_Q \cap [w']_a)$, so we get $\exists w'' \in \Min_a([w']_Q \cap [w']_a)$ s.t. $u' >_a w''$. As $u' \in \Min_a ( \llbracket \psi \rrbracket_\m \cap [w']_{a})$, we must therefore have $w'' \not\in \llbracket \psi \rrbracket_{\m}$, and then also $w' \not \in \llbracket \psi \rrbracket_{\m}$. But as $\m, w \models \psi$, we get the sought after contradiction (we initially assumed $(\m,w) \equiv^C (\m,w')$). We get [back$_\geq$] immediately from $Q$ being an equivalence relation.

\medskip \noindent
Now we get to [forth$_\leq$]. Let $(w,w') \in Q$ and $w \leq^Q_a v$. We have that [forth$_\leq$] is fulfilled if $\exists v' \in W$, s.t. $w' \leq^Q_a v'$ and $(v,v') \in Q$. 

\medskip \noindent \textit{Claim 2.} There exists a $v' \in [w']_a$ satisfying $(v,v') \in Q$.

\medskip
\noindent \textit{Proof of Claim 2.}
Suppose not. Then $v$ does not have a modally equivalent world in $[w']_a$. Thus there must be some formula $\phi$ holding in $v$ that holds nowhere in $[w']_a$. Since $v \in [w]_a$ (using Lemma~\ref{lemma:wellr}), this implies that $\m, w \models \widehat{K}_a \phi$ and $\m, w' \not\models \widehat{K}_a \phi$. However, this contradicts $(w,w') \in Q$, concluding the proof of the claim. 

\medskip \noindent
Let $v'$ be chosen as guaranteed by Claim 2. It now suffices to show $w' \leq^Q_a v'$. From $(v,v') \in Q$ and $v \geq^Q_a w$, [forth$_\geq$] gives a $w''$ s.t. $v' \geq^Q_a w''$ and $(w,w'') \in Q$. From $v' \geq^Q_a w''$ we get $v' \sim_a w''$, using Lemma~\ref{lemma:wellr}. Since $v' \in [w']_a$ we further get $w' \sim_a v' \sim_a w''$. Since $(w,w'') \in Q$ and $(w,w') \in Q$ we also get $(w',w'') \in Q$. From $w' \sim_a w''$ and $(w',w'') \in Q$ Lemma~\ref{lemma:equiplausr} gives us $w' \simeq_a^Q w''$. From this and $v' \geq^Q_a w''$ we get $v' \geq^Q_a w'$ and hence $w' \leq^Q_a v'$, as required. This concludes proof of [forth$_\leq$]. As for [back$_\geq$] getting to [back$_\leq$] is easy and left out.
\end{proof}
\begin{theorem}[Bisimulation characterisation for $L^C$]\label{theorem:bisimcharlc}
Let $(\m,w),(\m',w')$ be plausibility models. Then: \[ (\m,w) \bisim (\m',w') \text{ iff } (\m,w) \equiv^C (\m',w') \]
\end{theorem}
\begin{proof}
From Proposition \ref{theorem:bisimimpliesmodalequivlc} and Proposition \ref{theorem:modalequivimpliesbisimlc}.
\end{proof}
We can now finally give the promised proof of Proposition~\ref{prop:maxautos}.
\begin{proof}[Proof of Proposition~\ref{prop:maxautos}]
First note that neither the semantics of $L^C$ nor the proofs of Proposition~\ref{theorem:bisimimpliesmodalequivlc} and \ref{theorem:modalequivimpliesbisimlc} %, \ref{theorem:bisimcharlc} 
rely on the existence of largest autobisimulations. Hence we can use these in proving the proposition. Given a plausibility model $\m = (W,\geq,V)$ we define a relation $R$ by $R = \{(w,v) \in W^2 \mid \m, w \equiv^C \m, v \}$. Since modal equivalence implies bisimilarity (Theorem~\ref{theorem:modalequivimpliesbisimlc}), $R$ is a bisimulation relation (indeed, $R$ is exactly the relation shown to be an autobisimulation in the proof of Proposition~\ref{theorem:modalequivimpliesbisimlc}). Now we have to show that $R$ is the largest autobisimulation. If it was not, there would exist an autobisimulation $R'$ with $R' - R \neq \emptyset$.  By definition of $R$, $R'$ would then contain at least one pair $(w,v)$ with $ \m, w \not\equiv^C \m, v$. However, since bisimilarity implies modal equivalence (Proposition~\ref{theorem:bisimimpliesmodalequivlc}), this contradicts $R'$ being an autobisimulation. Hence $R$ must be the largest autobisimulation. It only remains to prove that $R$ is an equivalence relation. However, this is trivial given its definition in terms of modal equivalence.
\end{proof}

\subsection{Bisimulation correspondence for degrees of belief} \label{section:cdegrees}
\newcommand{\Rmax}{{R_{\text{max}}}}
We now show bisimulation characterisation results for the logic of degrees of belief $\lang^D$. Let $\m = (W, \geq, V)$. Recalling Definition \ref{defi:sat-rel}, for some world $w \in W$, the set $\Min_a^0 [w]_a$ contains the minimal worlds with respect to $\succeq_a$ in the $\sim_a$-equivalence class of $w$. For a given $w$ and $a$, we refer to the generalised definition $\Min_a^n [w]_a$ as (belief) sphere $n$ of $w$ for $a$. The distinction between $\Min_a^n$ and $\Min_a$ is important to keep straight! The former $\Min$---used to give semantics of the $B^n_a$ modality of $L^D$---is with respect to the relation $\succeq_a$. The latter $\Min$ is with respect to $\geq_a$, used to give the semantics of $L^C$.  Dealing as we do in this section with $L^D$, we first state some necessary observations about the properties of what we call beliefs spheres. When convenient we will simply say that $v$ is in (belief) sphere $n$ for $a$, understanding that this actually means $v \in \Min_a^n [v]_a$. 

It follows easily from the definitions, that for any world $w$, sphere $n$ for $a$ is wholly contained within sphere $n+1$ for $a$, i.e. $\Min_a^n [w]_a \subseteq \Min_a^{n+1} [w]_a$. %The following result is immediate from the definition of $\Min^n_a$.
\begin{lemma}\label{lemma:layersucc} Let $M = (W, \geq, V)$ be a plausibility model and consider $w, v \in W$. If $w \sim_a v$ and $w \not\in \Min^n_a [w]_a$, we have the following two properties:
\begin{enumerate}
\item[(i)] If $v \in \Min^n_a [w]_a$, then $w \succ_a v$.
\item[(ii)] If $v \in Min^{n+1}_a [w]_a$ then $w \succeq_a v$ .
\end{enumerate}
\end{lemma}
\begin{proof}
The truth of (i) easily comes from the definition of $\Min_a^n$. For (ii), we consider two exhaustive cases for $v$. Either $v \in Min^{n+1}_a [w]_a \setminus Min^{n}_a [w]_a$ in which case $w \succeq_a v$ follows from $\succeq_a$-minimality, since by assumption $w \in [w]_a \setminus Min^{n}_a [w]_a$. Otherwise $v \in Min^{n}_a [w]_a$, and so from $w \not \in Min^{n}_a [w]_a$ and (i) it follows that $w \succ_a v$ and hence also $w \succeq_a v$.
\end{proof}
Now getting to the meat of this section, showing bisimulation correspondence for $L^D$, we first show that bisimilar worlds belong to spheres of all the same degrees.
\begin{lemma}\label{lemma:bisimsamelayers}
If $(\m_1,w_1) \bisim (\m_2,w_2)$ then for all $n \in \Naturals$, $w_1 \in \Min^n_a [w_1]_a$ iff $w_2 \in \Min^n_a [w_2]_a$.
\end{lemma}
\begin{proof}
Assume $(\m_1,w_1) \bisim (\m_2,w_2)$. By definition there exists an autobisimulation $R$ on the disjoint union of $\m_1$ and $\m_2$ with $(w_1,w_2) \in R$. Using Proposition~\ref{prop:maxautos}, let $\Rmax$ denote the largest bisimulation on the disjoint union of $\m_1$ and $\m_2$ (so $\succeq_a = \geq_a^{\Rmax}$). Then $R \subseteq \Rmax$. We are going to show by contradiction that for any $(w,w') \in \Rmax$ (which includes $(w_1, w_2)$) and $n \in \Naturals$, $w \in \Min^n_a [w]_a$ iff $w' \in \Min^n_a [w']_a$. Suppose that this does not hold. Then there must be some pair of worlds $w$ and $w'$ such that $(w,w') \in \Rmax$ and either i) $w \in \Min^n_a [w]_a$ and $w' \not\in \Min^n_a [w']_a$, or ii) $w \not\in \Min^n_a [w]_a$ and $w' \in \Min^n_a [w']_a$ for some $n$. Let $n$ be the smallest natural number for which we have either i) or ii). Because the cases are symmetrical, we deal only with i). Using the alternative definition $\Min^n_a [w]_a = E^{n}_a [w]_a \cup \Min_a^{n-1} [w]_a$ we can deal with both $n > 0$ and $n = 0$ simultaneously.

By assumption of the smallest $n$ we have $w \not\in \Min^{n-1}_a [w]_a$, since $w' \not\in \Min_a^n [w']_a$ implies $w' \not\in \Min_a^k [w']_a$ for all $0 \leq k \leq n$ (so we could otherwise have chosen a smaller $n$). Therefore $w \in E_a^n [w]_a$ and $w' \not\in E_a^n [w']_a$. Because $w' \in [w'] \setminus \Min_a^n [w']_a$, we know that $n$ is not the maximum degree, so there must be \emph{some} world $v' \in E^n_a [w']_a$ which by definition means that $v' \not\in \Min_a^{n-1} [w']_a$. With $v' \in E^n_a [w']_a \subseteq \Min_a^n [w']_a$ and and $w' \not\in \Min_a^n [w']_a$, Lemma \ref{lemma:layersucc} gives $w' \succ_a v'$, i.e. $w' \succeq_a v'$ and $v' \not \succeq_a w'$. By [back$_\geq$] there is a $v$ s.t. $w \succeq_a v$ and $(v,v') \in \Rmax$. Because $v' \not\in Min^{n-1}_a [w']_a$ we cannot have $v \in \Min^{n-1}_a [w]_a$, as we could then again have chosen a smaller $n$ making either i) or ii) true. Thus $v \in [w]_a \setminus \Min^{n-1}_a [w]_a$. As $w \in \Min^n_a [w]_a$, Lemma \ref{lemma:layersucc} gives $v \succeq_a w$, so by [forth$_\geq$] there is a $u'$ s.t. $v' \succeq_a u'$ and $(w, u') \in \Rmax$. 

With $(w,w') \in \Rmax$ and $(w,u') \in \Rmax$, we have $(w',u') \in \Rmax$. As $w' \sim_a u'$ (we have $w' \succeq_a v'$ and $v' \succeq_a u'$), Lemma \ref{lemma:equiplausr} gives $w' \simeq_a^{\Rmax} u'$, i.e. $w' \succeq_a u'$ and $w' \preceq_a u'$. As $w' \not\in \Min_a^n [w']_a$, we then have $u' \not\in \Min_a^n [w']_a$. As $u' \not\in \Min^n_a [w']_a$ while $v' \in E_a^n [w']_a \subseteq \Min^n_a [w']_a$, Lemma \ref{lemma:layersucc} then gives $u' \succ_a v'$. But this contradicts $v' \succeq_a u'$, concluding the proof.
\end{proof}
\begin{proposition}\label{theorem:bisimimpliesmodalequivld}
	Bisimilarity implies modal equivalence for $\lang^D$.
\end{proposition}		
\begin{proof}
Assume $(\m_1,w_1) \bisim (\m_2,w_2)$. Let $\m = (W,\geq,V)$ denote the disjoint union of $\m_1$ and $\m_2$. Then there exists an autobisimulation $R$ on $\m$ with $(w_1,w_2) \in R$. Using Proposition~\ref{prop:maxautos}, let $\Rmax$ denote the largest autobisimulation on $\m$. Then $R \subseteq \Rmax$. We need to prove that $(\m,w_1) \equiv^D (\m,w_2)$.

We will show that for all $(w, w') \in \Rmax$, for all $\phi \in \lang^D$, $\m,w \models \phi$ iff $\m,w' \models \phi$ (which then also means that it holds for all $(w,w') \in R$). We proceed by induction on the syntactic complexity of $\phi$. The propositional and knowledge cases are already covered by Proposition~\ref{theorem:bisimimpliesmodalequivlc}, so we only go for $\phi = B_a^n \psi$. 

Assume $\m, w \models B_a^n \psi$. We need to prove that $M, w' \models B_a^n \psi$, that is $\m, v' \models \psi$ for all $v' \in \Min_a^n [w']_a$. Picking an arbitrary $v' \in \Min_a^n [w']_a$, we have $[w']_a = [v']_a$ from Lemma \ref{lemma:wellr}, and either $w' \succeq_a v'$ or $w' \preceq_a v'$ (so we also have $v' \in \Min_a^n [v']_a$). Using [back$_\geq$] or [back$_\leq$] as appropriate, we get that there is a $v$ such that $w \succeq_a v$ or $w \preceq_a v$, and $(v,v') \in \Rmax$. From this, $v' \in \Min_a^n [v']_a$, and Lemma \ref{lemma:bisimsamelayers} we get $v \in \Min_a^n [v]_a$, allowing us to conclude $v \in \Min_a^n [w]_a$ from $[w]_a = [v]_a$. With the original assumption of $\m, w \models B_a^n \psi$ we get $\m, v \models \psi$. As $(v,v') \in \Rmax$, the induction hypothesis gives $\m, v' \models \psi$. As $v'$ was chosen arbitrarily in $\Min_a^n [w']_a$ this gives $\m, w' \models B_a^n \psi$. Showing that $\m, w' \models B_a^n \psi$ implies $M, w \models B_a^n \psi$ is completely symmetrical and therefore left out.
\end{proof}
We now get to showing that modal equivalence for the language of degrees of belief implies bisimilarity. Trouble is, that the $B_a^n$ modality uses the largest autobisimulation for deriving the relation $\succeq_a$. This makes it difficult to go the Hennessy-Millner way of showing by contradiction that the modal equivalence relation $Q$ is an autobisimulation.

Instead, we establish that modal equivalence for $L^D$ implies modal equivalence for $L^C$. We go about this by way of a model and world dependent translation of $L^C$ formulas into $L^D$ formulas (Definition \ref{def:lc-to-ld-translation}). This translation has two properties. First, the translated formula is true at $M,w$ iff the untranslated formula is (Lemma \ref{lemma:gammaifftrans})---a quite uncontroversial property. More precisely, letting $M = (W, \geq, R)$ be a plausibility model, then for any $w \in W$, $\gamma \in L^C$ where $\sigma_{M,w}(\gamma)$ is the translation at $M,w$: $M, w \models \gamma \Leftrightarrow M, w \models \sigma_{M,w}(\gamma)$. Assume further that we have some $M',w'$ such that $(M,w) \equiv^D (M',w')$. As $\sigma_{M,w}(\gamma)$ is a formula of $L^D$ we can conclude $M',w' \models \sigma_{M,w}(\gamma)$. So in all we get that
\begin{align}
	\tag{*}
	M, w \models \gamma \Leftrightarrow M, w \models \sigma_{M,w}(\gamma) \Leftrightarrow M', w' \models \sigma_{M,w}(\gamma)
\end{align}
The second property is that the translation of $\gamma$ is the same for worlds modally equivalent for $L^D$ (Lemma \ref{lemma:sametrans}): If $(M, w) \equiv^D (M',w')$ then $\sigma_{M,w}(\gamma) = \sigma_{M',w'}(\gamma)$. 
This then gives  
\begin{align}
\tag{**}
M', w' \models \sigma_{M,w}(\gamma) \Leftrightarrow M',w' \models \sigma_{M',w'}(\gamma) \Leftrightarrow M',w' \models \gamma
\end{align}
Combining (*) and (**) gives that if $(M,w) \equiv^D (M',w')$ then $M, w \models \gamma$ iff $M', w' \models \gamma$ for any $\gamma \in L^C$, i.e. that $(M,w) \equiv^C (M',w')$. As shown in the previous section, modal equivalence for $L^C$ implies bisimilarity (Proposition \ref{theorem:modalequivimpliesbisimlc}), and we can therefore finally conclude that modal equivalence for $L^D$ implies bisimilarity (Proposition\ref{theorem:modalequivimpliesbisimld}). 
\newcommand{\trans}{\sigma_{M,w}}
\begin{lemma}\label{lemma:k}
For a plausibility model $M$, a world $w \in \dom M$, agent $a \in \agents$ and a formula $\psi$ of $L^C$, if $\llbracket \psi \rrbracket_M \cap [w]_a \not = \emptyset$, there is a unique natural number $k$ for which  $\Min_a (\llbracket \psi \rrbracket_M \cap [w]_a) \subseteq E^k_a [w]_a$ ($= \Min_{\succeq_a} ([w]_a \setminus \Min_a^{k-1} [w]_a)$).
\end{lemma}
\begin{proof}
Let $S = \llbracket \psi \rrbracket_\m \cap [w]_a$. We first show that all worlds in $\Min_a S$ are equiplausible with respect to $\succeq_a$.

Take any two worlds $v_1, v_2 \in \Min_a S$. We wish to show $v_1 \simeq_a^R v_2$, i.e. $\Min_a([v_1]_R \cap [v_1]_a) \simeq_a \Min_a([v_2]_R \cap [v_2]_a)$, where $R$ is the largest autobisimulation on $M$. With Proposition\ref{theorem:bisimimpliesmodalequivlc} (bisimilarity implies modal equivalence for $L^C$) and for $i =1,2$, we have that $[v_i]_R \subseteq \llbracket \psi \rrbracket_M$. Hence $[v_i]_R \cap [v_i]_a = [v_i]_R \cap [w]_a \subseteq \llbracket \psi \rrbracket_M \cap [w]_a = S$. With $v_i \in \Min_a S$ and $v_i \in [v_i]_R \cap [v_i]_a \subseteq S$, we have $v_i \in \Min_a ([v_i]_R \cap [v_i]_a)$ (if an element of a set $A$ is minimal in a set $B \supseteq A$, then it is also minimal in $A$). From this we can conclude that $\Min_a([v_i]_R \cap [v_i]_a) \simeq_a \{v_i\}$. Since $v_1 \simeq_a v_2$ we get $\Min_a([v_1]_R \cap [v_1]_a) \simeq_a \{v_1 \} \simeq_a \{v_2 \} \simeq_a \Min_a([v_2]_R \cap [v_2]_a)$, concluding the proof that $v_1 \simeq_a^R v_2$.

Due to (pre)image-finiteness of $M$, $[w]_a$ is finite. This means that for any $v \in [w]_a$ there is a unique natural number $k$ for which $v \in E_a^k [w]_a$. As all worlds in $\Min_a S$ are $\succeq_a$-equiplausible, we have that $\Min_a S \subseteq E^k_a [w]_a$ for some unique $k$. 
\end{proof}
Having established that if $\llbracket \psi \rrbracket_M \cap [w]_a \not = \emptyset$ then there does indeed exist a unique $k$ st. $\Min_a (\llbracket \psi \rrbracket_M \cap [w]_a) \subseteq E^k_a [w]_a$, we have that the following translation is well-defined.
\begin{definition}[Translation $\sigma_{M,w}$]\label{def:lc-to-ld-translation}
Let $M = (W, \geq, V)$ be a plausibility model and $\gamma \in L^C$ be given. We write $\trans(\gamma)$ for the \emph{translation} of $\gamma$ at $M,w$ into a formula of $L^D$ defined as follows:
\begin{align*}
	\sigma_{M,w}(p) =&\ p\\
	\sigma_{M,w}(\neg \phi) =&\ \neg \sigma_{M,w}(\phi)\\
	\sigma_{M,w}(\phi_1 \land \phi_2) =&\ \sigma_{M,w}(\phi_1) \land \sigma_{M,w}(\phi_2)
\end{align*}
\vspace{-0.975cm}
\begin{align*}
	\sigma_{M,w}(B^{\psi}_a \phi) =	 &
			\begin{cases} 
			B_a^k \bigvee \{ \sigma_{M,v}(\psi \rightarrow \phi) \mid v \in [w]_a\} \land \widehat{B}_a^k \bigvee \{ \sigma_{M,v}(\psi) \mid v \in [w]_a\} & \mbox{if } \llbracket \psi \rrbracket_M \cap [w]_a \neq \emptyset\\
			K_a \bigvee \{ \sigma_{M,v}(\neg \psi) \mid v \in [w]_a\} & \mbox{if } \llbracket \psi \rrbracket_M \cap [w]_a = \emptyset
			\end{cases}
	\end{align*}
where $k$ is the natural number such that $\Min_a (\llbracket \psi \rrbracket_M \cap [w]_a) \subseteq E^k_a [w]_a$. As $K_a \phi$ is definable in $L^C$ as $B_a^{\neg \phi} \bot$, we need no $K_a \phi$-case in the translation.
\end{definition}
We need (pre)image-finiteness of $M$ because the translation of $\sigma_{M,w}(B_a^\psi \phi)$ is based on either $[w]_a$ or $\Min_a(\llbracket \psi \rrbracket_M \cap [w]_a)$. For $\sigma_{M,w}(B_a^\psi \phi)$ to be finite, we need finiteness of $[w]_a$.

We now get to showing the first of the promised properties of the translation, namely that  the translated formula is true at $M,w$ iff the untranslated formula is.
\begin{lemma}\label{lemma:gammaifftrans}
Given a plausibility model $M = (W, \geq, V)$ and $\gamma \in L^C$ we have $M,w \models \gamma$ iff $M,w \models \sigma_{M,w}(\gamma)$ for all $w \in W$.
\end{lemma}
\begin{proof}
We show both directions by induction on the modal depth of $\gamma$. For the base case of a modal depth of $0$, we have $\sigma_{M,w}(\gamma) = \gamma$ easily, giving $M,w \models \gamma$ iff $M,w \models \sigma_{M,w}(\gamma)$. The $p$-, $\neg$-, $\land$-cases being quite easy, we deal only with $\gamma = B_a^\psi \phi$ in the induction step. For that case there are to subcases; whether $\sigma_{M,w}(\gamma)$ is a $K_a$-formula or not.

\medskip\noindent\textit{\textbf{($\Rightarrow$) : }$M,w \models \gamma$ implies $M,w \models \sigma_{M,w}(\gamma)$.}
\medskip

Take first the case $\llbracket \psi \rrbracket_M \cap [w]_a = \emptyset$ where $\sigma_{M,w}(B_a^\psi \phi) = K_a \bigvee \{ \sigma_{M,v}(\neg \psi) \mid v \in [w]_a\}$. If $\llbracket \psi \rrbracket_M \cap [w]_a = \emptyset$, then $M, v \models \neg \psi$ for all $v \in [w]_a$. Applying the induction hypothesis gives $M, v \models \sigma_{M,v}(\neg \psi)$ for all $v \in [w]_a$. Then we also have $M, u \models \bigvee \{ \sigma_{M,v}(\neg \psi) \mid v \in [w]_a\}$ for all $u \in [w]_a$ and finally $M, w \models K_a \bigvee \{ \sigma_{M,v}(\neg \psi) \mid v \in [w]_a\}$. 

Now take the case $\llbracket \psi \rrbracket_M \cap [w]_a \neq \emptyset$. Letting $S = \Min_a (\llbracket \psi \rrbracket_\m \cap [w]_a)$ and $k$ be chosen as in the translation, i.e. such that $S \subseteq E_a^k [w]_a$, we wish to prove that  $M, w \models B_a^\psi \phi$ implies $M, w \models B_a^k \bigvee \{ \sigma_{M,v}(\psi \rightarrow \phi) \mid v \in [w]_a\} \land \widehat{B}_a^k \bigvee \{ \sigma_{M,v}(\psi) \mid v \in [w]_a\}$. We first show $M, w \models \widehat{B}_a^k \bigvee \{ \sigma_{M,v}(\psi) \mid v \in [w]_a\}$. Because $M, v \models \psi$ for all $v \in S$, the induction hypothesis gives $M, v \models \sigma_{M,v}(\psi)$ for all $v \in S$. From this we can conclude $M, u \models \bigvee \{ \sigma_{M,v}(\psi) \mid v \in S\}$ for all $u \in S$, and thus also $M, u \models \bigvee \{ \sigma_{M,v}(\psi) \mid v \in [w]_a\}$ for all $u \in S$. From Lemma \ref{lemma:k} we have $S \subseteq \Min_a^k [w]_a$, so $M,u \models \bigvee \{ \sigma_{M,v}(\psi) \mid v \in [w]_a\}$ for some $u \in \Min_a^k [w]_a$. This gives $M, w \models \widehat{B}_a^k \bigvee \{ \sigma_{M,v}(\psi) \mid v \in [w]_a\}$. Next is $M, w \models B_a^k \bigvee \{ \sigma_{M,v}(\psi \rightarrow \phi) \mid v \in [w]_a\}$.

\medskip\noindent\textit{Claim.} If $M, w \models B_a^\psi \phi$, then for all $v \in E_a^k [w]_a \cap \llbracket \psi \rrbracket_M$, $M, v \models \phi$.

\medskip\noindent\textit{Proof of claim.} We show the claim by contradiction, assuming that at least one world in $E_a^k [w]_a \cap \llbracket \psi \rrbracket_M$ is a $\neg \phi$-world. Let $v$ be this $\psi \land \neg \phi$-world. As $v \in E_a^k [w]_a$, we have $\{v\} \simeq_a^{\Rmax} E_a^k [w]_a \simeq_a^{\Rmax} S$, and specifically that $\forall s \in S: v \simeq_a^{\Rmax} s$. This means $\forall s \in S: \Min([v]_{\Rmax} \cap [v]_a) \simeq_a \Min_a ([s]_{\Rmax} \cap [s]_a)$. Because $\forall s \in S : \Min_a ([s]_{\Rmax} \cap [s]_a) \simeq_a S$, we have $\Min([v]_{\Rmax} \cap [v]_a) \simeq_a S$ and thus some $v' \in \Min([v]_{\Rmax} \cap [v]_a)$ such that $\{v'\} \simeq_a S$. Combining $v' \in [v]_{\Rmax}$ with Theorem \ref{theorem:bisimcharlc} gives $M,v \equiv^C M,v'$ and thus that $M, v' \models \psi \land \neg \phi$. Putting $v' \in \llbracket \psi \rrbracket_M$ together with $\{v'\} \simeq_a S$, means that $v' \in S$. As $M, v' \models \neg \phi$, we have a contradiction of $M,w \models B_a^\psi \phi$, concluding the proof of the claim.
\medskip

With $M, w \models B_a^\psi \phi$, we now have $M, v \models \phi$ for all $v \in E_a^k [w]_a \cap \llbracket \psi \rrbracket_M$, and thus $M, v \models \psi \rightarrow \phi$ for all $v \in E_a^k [w]_a$. Lemma \ref{lemma:k} gives $S \subseteq E_a^k [w]_a$, and by definition we have $E_a^k [w]_a \cap \Min_a^{k-1} [w]_a = \emptyset$, that is, there are no $\psi$-worlds below layer $k$, so $M, v \models \psi \rightarrow \phi$ for all $v \in Min_a^k [w]_a$. Using the induction hypothesis gives $\m, v \models \sigma_{M,v}(\psi \rightarrow \phi)$ for all $v \in Min_a^k [w]_a$ and therefore $M, w \models B^k_a \bigvee \{ \sigma_{M,v}(\psi \rightarrow \phi) \mid v \in [w]_a \}$, finalising left-to-right direction of the proof.

\medskip\noindent\textit{\textbf{($\Leftarrow$) : }$M,w \models \sigma_{M,w}(\gamma)$ implies $M,w \models \gamma$.}
\medskip

We show the stronger claim that $M, w \models \sigma_{M,w'} (\gamma)$ for some $w' \in D(M)$ implies $M, w \models \gamma$. Let $\gamma = B_a^\psi \phi$ and suppose that $M, w \models \sigma_{M,w'} (\gamma)$ for some $w' \in D(M)$. We then need to show $M, w \models B_a^\psi \phi$. First take the case where $\llbracket \psi \rrbracket_M \cap [w']_a = \emptyset$. Then $\sigma_{M,w'}(B_a^\psi \phi) = K_a \bigvee \{ \sigma_{M,v'}(\neg \psi) \mid v' \in [w']_a\}$, i.e. $M, w \models K_a \bigvee \{ \sigma_{M,v'}(\neg \psi) \mid v' \in [w']_a\}$. This means that $M,v \models \bigvee \{ \sigma_{M,v'}(\neg \psi) \mid v' \in [w']_a\}$ for all $v \in [w]_a$, i.e. for any $v \in [w]_a$ there is a $v' \in [w']_a$ such that $M, v \models \sigma_{M,v'}(\neg \psi)$. Applying the induction hypothesis, we get $M, v \models \neg \psi$ for all $v \in [w]_a$. Thus $\llbracket \psi \rrbracket_M \cap [w]_a = \emptyset$ and we trivially have $M, w \models B_a^\psi \phi$.

Now take the case $\llbracket \psi \rrbracket_M \cap [w']_a \neq \emptyset$. Letting $S' = \Min_a(\llbracket \psi \rrbracket_M \cap [w']_a)$ and $k'$ be s.t. $S' \subseteq E_a^{k'} [w']_a$, we have $M, w \models B_a^{k'} \bigvee \{ (\sigma_{M,v'}(\psi \rightarrow \phi) \mid v' \in [w']_a\} \land \widehat{B}_a^{k'} \bigvee \{ \sigma_{M,v'}(\psi) \mid v' \in [w']_a\}$. From $M, w \models B_a^{k'} \bigvee \{ \sigma_{M,v'}(\psi \rightarrow \phi) \mid v' \in [w']_a\}$ we have $M, v \models \bigvee \{ \sigma_{M,v'}(\psi \rightarrow \phi) \mid v' \in [w']_a\}$ for all $v \in \Min_a^{k'} [w]_a$, i.e. for any $v \in [w]_a$ there is a $v' \in [w']_a$ such that $M, v \models \sigma_{M,v'}(\psi \rightarrow \phi)$. Applying the induction hypothesis, we get $M, v \models \psi \rightarrow \phi$ for all $v \in \Min_a^{k'} [w]_a$. From $M, w \models \widehat{B}_a^{k'} \bigvee \{ \sigma_{M,v'}(\psi) \mid v' \in [w']_a\}$ we have $M, v \models \bigvee \{ \sigma_{M,v'}(\psi) \mid v' \in [w']_a\}$ for some $v \in \Min_a^{k'} [w]_a$, i.e. there is a $v \in [w]_a$ and a $v' \in [w']_a$ such that $M, v \models \sigma_{M,v'}(\psi)$. Applying the induction hypothesis gets us $M, v \models \psi$. Thus we have $M, w \models B_a^{k'} (\psi \rightarrow \phi) \land \widehat{B}_a^{k'} \psi$ (where $\psi, \phi \in L^C$). 

From $M,w \models \widehat{B}_a^{k'} \psi$ we have that $\llbracket \psi \rrbracket_M \cap [w]_a \neq \emptyset$, so Lemma \ref{lemma:k} gives the existence of a $k$, s.t.  $\Min_a (\llbracket \psi \rrbracket_M \cap [w]_a) \subseteq \Min_a^{k} [w]_a$. We also have from $M,w \models \widehat{B}_a^{k'} \psi$ that $k \leq k'$, so $\Min_a (\llbracket \psi \rrbracket_M \cap [w]_a) \subseteq \Min_a^{k'} [w]_a$. With $M, w \models B_a^{k'} (\psi \rightarrow \phi)$ we get $M,v \models \psi \rightarrow \phi$ for all $v \in \Min_a(\llbracket \psi \rrbracket_M \cap [w]_a)$, then $M,v \models \phi$ for all $v \in \Min_a(\llbracket \psi \rrbracket_M \cap [w]_a)$, and finally $M, w \models B_a^\psi \phi$.
\end{proof}
We have now gotten to the second of the two promised properties; that the translation is the same for worlds modally equivalent for $L^D$.
\begin{lemma} Given plausibility models $M$ and $M'$, for any $w \in D(M)$ and $w' \in D(M')$, if $(M,w) \equiv^D (M',w')$ then for any formula $\gamma \in L^C$, $\sigma_{M,w}(\gamma) = \sigma_{M',w'}(\gamma)$.\label{lemma:sametrans}
\end{lemma}
\begin{proof}
We show this by another induction on the modal depth of $\gamma$. For the base case of modal depth 0 we trivially have $\sigma_{M,w}(\gamma) = \sigma_{M',w'}(\gamma)$. 

For the induction step we, as before, only deal with $\gamma = B_a^\psi \phi$. Note first that every world in $[w]_a$ is modally equivalent to at least one world in $[w']_a$. If that wasn't the case, there would be some $L^D$-formula $\phi$ true somewhere in $[w]_a$ and nowhere in $[w']_a$. Then $M, w \models \widehat{K}_a \phi$ while $M', w' \not\models \widehat{K}_a \phi$, contradicting $(M, w) \equiv^D (M', w')$. A completely analogous argument gives that every world in $[w']_a$ is modally equivalent to at least one world in $[w]_a$. Thus $\llbracket \psi \rrbracket_M \cap [w]_a = \emptyset$ iff $\llbracket \psi \rrbracket_M' \cap [w']_a = \emptyset$. We thus have two cases, either both $\sigma_{M,w}(B_a^\psi \phi)$ and $\sigma_{M',w'}(B_a^\psi \phi)$ are $K_a$-formulas, or both are $B_a^k$-formulas.

We deal first with the case where both translations are $K_a$-formulas. Here we have $\sigma_{M,w}(B_a^\psi \phi) = K_a \bigvee \{ \sigma_{M,v} (\neg \phi) \mid v \in [w]_a\}$ and $\sigma_{M',w'}(B_a^\psi \phi) = K_a \bigvee \{ \sigma_{M',v'} (\neg \phi) \mid v' \in [w']_a\}$. As already shown, for all $v \in [w]_a$ there is a $v' \in [w']_a$ such that $(M,w) \equiv^D (M',v')$, and vice versa. The induction hypothesis gives $\sigma_{M,v}(\neg \phi) = \sigma_{M',v'}(\neg \phi)$ for all these $v$s and $v'$s. Then $\bigvee \{ \sigma_{M,v} (\neg \phi) \mid v \in [w]_a\} = \bigvee \{ \sigma_{M',v'} (\neg \phi) \mid v' \in [w']_a\}$ and thus $\sigma_{M,w}(B_a^\psi \phi) = \sigma_{M',w'}(B_a^\psi \phi)$.

Take now the case where both translations are $B^k_a$-formulas. A similar argument as above gives $\bigvee \{ \sigma_{M,v}(\psi \rightarrow \phi) \mid v \in [w]_a \} = \bigvee \{ \sigma_{M',v'}(\psi \rightarrow \phi) \mid v' \in [w']_a \}$ and $\bigvee \{ \sigma_{M,v}(\psi) \mid v \in [w]_a \} = \bigvee \{ \sigma_{M',v'}(\psi) \mid v' \in [w']_a \}$.  Letting $k$ and $k'$ be the indices chosen in the translation of $\sigma_{M,w}(B_a^\psi \phi)$ and $\sigma_{M',w'}(B_a^\psi \phi)$ respectively, we have $\sigma_{M,w}(B_a^\psi \phi) = \sigma_{M',w'}(B_a^\psi \phi)$ if $k = k'$.  Assume towards a contradiction that $k > k'$. Lemma \ref{lemma:k} now gives $\Min_a(\llbracket \psi\rrbracket_M \cap [w]_a) \cap \Min_a^{k'} [w]_a = \emptyset$, so $M, v \models \neg \psi$ for all $v \in Min_a^{k'} [w]_a$. With Lemma \ref{lemma:gammaifftrans} we have $M, v \models \sigma_{M,v}(\neg \psi)$ for all $v \in Min_a^{k'} [w]_a$ and thus also that $M, w \models B_a^{k'} \bigvee \{ \sigma_{M,v}(\neg \psi) \mid v \in [w]_a\}$. From Lemma \ref{lemma:k} we also have $\Min_a(\llbracket \psi\rrbracket_M \cap [w']_a) \subseteq Min_a^{k'} [w']_a$, so $M', v' \not\models \neg \psi$ for some $v' \in Min_a^{k'} [w']_a$. From here we use Lemma \ref{lemma:gammaifftrans} to conclude $M', v' \not\models \sigma_{M',v'}(\neg \psi)$ for some $v' \in Min_a^{k'} [w']_a$ and thus $M', w' \not\models B_a^{k'} \bigvee \{ \sigma_{M',v'}(\neg \psi) \mid v' \in [w']_a \}$. By the work done so far, this also means $M', w' \not\models B_a^{k'} \bigvee \{ \sigma_{M,v}(\neg \psi) \mid v \in [w]_a\}$ which contradicts $(M,w) \equiv^D (M', w')$. The case when $k' > k$ is completely symmetrical, and the proof if thus concluded.
\end{proof}
\begin{proposition}\label{theorem:modalequivimpliesbisimld}
	Modal equivalence for $L^D$ implies bisimilarity.
\end{proposition}
\begin{proof} Let $M = (W, \geq, V)$ and $M' = (W', \geq', V')$ be two plausibility models. We first show that if $(M,w) \equiv^D (M',w')$ then $(M,w) \equiv^C (M',w')$. Assume $(M, w) \equiv^D (M', w,')$ and let $\gamma$ be any formula of 
	$L^C$.
\begin{align*}
	 M, w \models \gamma &\Leftrightarrow M, w \models \sigma_{M,w}(\gamma) \ &\text{(Lemma \ref{lemma:gammaifftrans})}  \\
						& \Leftrightarrow M', w' \models \sigma_{M,w}(\gamma) \ &\text{(by assumption)}\\
						&\Leftrightarrow M', w' \models \sigma_{M',w'}(\gamma) \ &\text{(Lemma \ref{lemma:sametrans})} \\
						&\Leftrightarrow M', w' \models \gamma \ &\text{(Lemma \ref{lemma:gammaifftrans})}
\end{align*}
	Putting this together with Theorem \ref{theorem:modalequivimpliesbisimlc} (modal equivalence for $L^C$ implies bisimilarity), we have that two worlds which are modally equivalent in $L^D$ are also modally equivalent in $L^C$ and therefore bisimilar.
\end{proof}
\begin{theorem}[Bisimulation characterisation for $L^D$]\label{theorem:bisimcharld}
Let $(\m,w),(\m',w')$ be plausibility models. Then: \[ (\m,w) \bisim (\m',w') \text{ iff } (\m,w) \equiv^D (\m',w') \]
\end{theorem}		
\begin{proof}
From Proposition \ref{theorem:bisimimpliesmodalequivld} and Proposition \ref{theorem:modalequivimpliesbisimld}.
\end{proof}		

\subsection{Bisimulation correspondence for safe belief} \label{section:csafe}
We now show bisimulation characterisation results for the logic of degrees of belief $\lang^S$.
\begin{proposition}\label{theorem:bisimimpliesmodalequivls}
Bisimilarity implies modal equivalence for $\lang^S$. 
\end{proposition}
\begin{proof}
\newcommand\restrict\upharpoonright
\newcommand\rmax{R_{max}}
Assume $\m_1 \bisim \m_2$. Then there is an autobisimulation $R'$ on the disjoint union $\m_1 \sqcup \m_2$ with $R' \cap (\Domain(\m_1) \times \Domain(\m_2)) \neq \emptyset$. Extend $R'$ into the largest autobisimulation $R$ on $\m_1\sqcup \m_2$ (using Proposition~\ref{prop:maxautos}). Define $R_1 = R \cap (\Domain(\m_1) \times \Domain(\m_1))$ and $R_2 = R \cap (\Domain(\m_2) \times \Domain(\m_2))$. 

\medskip \noindent 
\emph{Claim}. Let $i \in \{1,2 \}$ and $w \in \Domain(\m_i)$. Then
\begin{enumerate}
  \item[(i)] $R_i$ is the largest autobisimulation on $\m_i$. 
  \item[(ii)] $\Min_a([w]_{R^=} \cap [w]_a) = \Min_a([w]_{R_i^=} \cap [w]_a)$.
  \item[(iii)] For any $v$, $w \geq_a^R v$ iff $w \geq_a^{R_i} v$.
\end{enumerate}

\medskip \noindent
\emph{Proof of claim}. To prove (i), let $S_i$ denote the largest autobisimulation on $\m_i$. If we can show $S_i \subseteq R_i$ we are done. Since $S_i$ is an autobisimulation on $\m_i$, it must also be an autobisimulation on $\m_1 \sqcup \m_2$. Thus, clearly, $S_i \subseteq R$, since $R$ is the largest autobisimulation on $\m_1 \sqcup \m_2$. Hence, since $S_i \subseteq \Domain(\m_i) \times \Domain(\m_i)$, we get $S_i = S_i \cap (\Domain(\m_i) \times \Domain(\m_i)) \subseteq R \cap (\Domain(\m_i) \times \Domain(\m_i)) = R_i$. This shows $S_i \subseteq R_i$, as required. 

We now prove (ii). Since $w \in \Domain(\m_i)$ we get $[w]_a \subseteq \Domain(\m_i)$. Since $R_i = R \cap (\Domain(\m_i) \times \Domain(\m_i))$ this implies $[w]_R \cap [w]_a = [w]_{R_i} \cap [w]_a$. Now note that since $R$ is the largest autobisimulation on $\m_1 \sqcup \m_2$ and $R_i$ is the largest autobisimulation on $\m_i$, we have $R = R^=$ and $R_i = R_i^=$, by Proposition~\ref{prop:maxautos} (the largest autobisimulation is an equivalence relation). Hence from $[w]_R \cap [w]_a = [w]_{R_i} \cap [w]_a$ we can conclude $[w]_{R^=} \cap [w]_a = [w]_{R_i^=} \cap [w]_a$, and then finally $\Min_a ([w]_{R^=} \cap [w]_a) = \Min_a ([w]_{R_i^=} \cap [w]_a)$.

We now prove (iii). Note that if $w \geq_a^R v$ or $w \geq_a^{R_i}$ then $w \sim_a v$ (by Lemma~\ref{lemma:wellr}). So in proving $w \geq_a^R v \Leftrightarrow w \geq_a^{R_i} v$ for $w \in \Domain(\m_i)$, we can assume that also $v\in \Domain(\m_i)$. We then get:
\[
  \begin{array}{rcl}
    w \geq_a^R v &\Leftrightarrow &\Min_a ([w]_{R^=} \cap [w]_a) \geq_a \Min_a ([v]_{R^=} \cap [v]_a) \\
                             &\Leftrightarrow &\Min_a ([w]_{R_i^=} \cap [w]_a) \geq_a \Min_a ([v]_{R_i^=} \cap [v]_a) \quad \text{
                             by (ii), since $w,v\in \Domain(\m_i)$}                             
                              \\
                             &\Leftrightarrow &w \geq_a^{R_i} v.
  \end{array}
\]
This completes the proof of the claim.

\medskip \noindent
We will now show that for all $\phi$ and all $(w_1,w_2) \in R \cap (\Domain(\m_1) \times \Domain(\m_2))$, if $\m_1,w_1 \models \phi$ then $\m_2,w_2 \models \phi$ (the other direction being symmetric). The proof is by induction on the syntactic complexity of $\phi$. The propositional and knowledge cases are already covered by Proposition~\ref{theorem:bisimimpliesmodalequivlc}, so we only need to consider the case $\phi = \Box_a \psi$. Hence assume $\m_1,w_1 \models \Box_a \psi$ and $(w_1,w_2) \in R \cap (\Domain(\m_1) \times \Domain(\m_2))$. We need to prove $\m_2,w_2 \models \Box_a \psi$. Pick an arbitrary $v_2 \in \Domain(\m_2)$ with $w_2 \succeq_a v_2$. If we can show $\m_2, v_2 \models \psi$, we are done. By (i), $R_2$ is the largest autobisimulation on $\m_2$.
Hence $w_2 \succeq_a v_2$ by definition means $w_2 \geq_a^{R_2} v_2$. Using (iii), we can from $w_2 \geq_a^{R_2} v_2$ conclude $w_2 \geq_a^R v_2$. Since $R$ is an autobisimulation, we can now apply [back$_\geq$] to $(w_1,w_2) \in R$ and $w_2 \geq_a^R v_2$ to get a
 $v_1$ with $w_1 \geq_a^{R} v_1$ and $(v_1,v_2) \in R$. Using (iii) again we can conclude from $w_1 \geq_a^R v_1$ to $w_1 \geq_a^{R_1} v_1$, since $w_1 \in \Domain(\m_1)$. By (i), $R_1$ is the largest autobisimulation on $\m_1$, so $w_1 \geq_a^{R_1} v_1$ is by definition the same as $w_1 \succeq_a v_1$. 
 Since we have assumed $\m_1, w_1 \models \Box_a \psi$, and since $w_1 \succeq_a v_1$, we get $\m_1, v_1 \models \psi$. Since $(v_1,v_2) \in R$, the induction hypothesis gives us $\m_2, v_2 \models \psi$, and we are done.
\end{proof}
As for the previous logics, the converse also holds, that is, modal equivalence with regard to $L^S$ implies bisimulation. This is going to be proved as follows. First we prove that any conditional belief formula $\phi_C$ can be translated into a logically equivalent safe belief formula $\phi_S$. This implies that if two pointed models $(M,w)$ and $(M',w')$ are modally equivalent in $L^S$, they must also be modally equivalent in $L^C$: Any formula $\phi_C \in L^C$ is true in $(M,w)$ \emph{iff} its translation $\phi_S \in L^S$ is true in $(M,w)$ \emph{iff} $\phi_S$ is true in $(M',w')$ \emph{iff} $\phi_C$ is true in $(M',w')$. Now we can reason as follows: If two pointed models $(M,w)$ and $(M',w')$ are modally equivalent in $L^S$ then they are modally equivalent in $L^C$ and hence, by Theorem~\ref{theorem:modalequivimpliesbisimlc}, bisimilar. This is the result we were after. We postpone the full proof until Section~\ref{expressivity:section:safe-belief}, which is where we provide the translation of conditional belief formulas into safe belief formulas (as part of a systematic investigation of the relations between the different languages and their relative expressivity). Here we only state the result:
\begin{proposition}\label{theorem:modalequivimpliesbisimls}
	Modal equivalence for $L^S$ implies bisimilarity.
\end{proposition}
\begin{proof}
See Section~\ref{expressivity:section:safe-belief}.
\end{proof}
As for the two previous languages, $L^C$ and $L^D$, we now get the following bisimulation characterisation result.
\begin{theorem}[Bisimulation characterisation for $L^S$]\label{theorem:bisimcharls}
Let $(\m,w),(\m',w')$ be plausibility models. Then: \[ (\m,w) \bisim (\m',w') \text{ iff } (\m,w) \equiv^S (\m',w') \]
\end{theorem}
\begin{proof}
From Proposition \ref{theorem:bisimimpliesmodalequivls} and Proposition \ref{theorem:modalequivimpliesbisimls}.
\end{proof}

\subsection{Combining characterisation results}
By combining Theorems \ref{theorem:bisimcharlc}, \ref{theorem:bisimcharld} and \ref{theorem:bisimcharls} from the previous subsections we immediately have the following result.
\begin{corollary} \label{expressivity:corollary:bisim-implies-modal-equiv-cds}
	Bisimilarity corresponds to modal equivalence in the logics of conditional belief, degrees of belief, and safe belief, and in any logic containing two or all three of these belief modalities; and modal equivalence in one of these logics corresponds to modal equivalence in any other.
	%		Let a plausibility model $\m = (W, \geq, V)$ be given and let $R$ be an autobisimulation on $\m$. For any $(w,w') \in R$ we have that $(\m, w) \equiv^{CDS} (\m, w')$.
	\end{corollary}
For example, we also have that $(\m,w) \bisim (\m',w')$ iff $(\m,w) \equiv^{CDS} (\m',w')$, or that $(\m,w) \bisim (\m',w')$ iff $(\m,w) \equiv^{DS} (\m',w')$. Also, to be explicit, we now have that
\begin{itemize}
\item $(\m,w)\equiv^C(\m',w')$ iff $(\m,w)\equiv^D(\m',w')$
\item $(\m,w)\equiv^C(\m',w')$ iff $(\m,w)\equiv^S(\m',w')$
\item $(\m,w)\equiv^D(\m',w')$ iff $(\m,w)\equiv^S(\m',w')$
\end{itemize}
In other words, the information content of a pointed plausibility model is equally well described in any of these logics. This seems to suggest that it does not matter which logic you use to describe the information content of such a model, apart from the usual considerations of succinctness. Still, this is not the case: our logics are not equally expressive. This will now be addressed in the next section.

%%%%%%%%%%%%%%%%%%
% Results on expressivity
%%%%%%%%%%%%%%%%%%

\section{Expressivity} \label{section:expressivity}
In this section we will determine the expressivity hierarchy of the logics under consideration. Abstractly speaking, expressivity is a yardstick for measuring whether two logics are able to capture the same properties of a class of models. More concretely in our case, we will for instance be interested in determining whether the conditional belief modality can be expressed using the degrees of belief modality (observe that the translation in Section \ref{section:cdegrees} depends on a particular model). With such results at hand we can for instance justify the inclusion or exclusion of a modality, and it also sheds light upon the strengths and weaknesses of our doxastic notions. To start things off we now formally introduce the notion of expressivity found in \cite{hvdetal.del:2007}. 
	\begin{definition} \label{expressivity:definition:expressivity}
		Let $\lang$ and $\lang'$ be two logical languages interpreted on the same class of models. 
		\begin{itemize}
			\item For $\phi \in \lang$ and $\phi' \in \lang'$, we say that $\phi$ and $\phi'$ are \emph{equivalent} ($\phi \equiv \phi'$) iff they are true in the same pointed models of said class.\footnote{With our usage of $\equiv$ it is clear from context whether we're referring to modal equivalence, formulas or languages.}
			\item $\lang'$ is \emph{at least as expressive} as $\lang$ ($\lang \leqq \lang'$) iff for every $\phi \in \lang$ there is a $\phi' \in \lang'$ s.t. $\phi \equiv \phi'$.
			\item $\lang$ and $\lang'$ are \emph{equally expressive} ($\lang \equiv \lang'$) iff $\lang \leqq \lang'$ and $\lang' \leqq \lang$.
			\item $\lang'$ is \emph{more expressive} than $\lang$ ($\lang < \lang'$) iff $\lang \leqq \lang'$ and $\lang' \not \leqq \lang$.
			\item $\lang$ and $\lang'$ are \emph{incomparable} ($\lang \bowtie \lang'$) iff $\lang \not \leqq \lang'$ and $\lang' \not \leqq \lang$.
		\end{itemize}
	\end{definition}
	Below we will show several cases where $\lang \not \leqq \lang'$; i.e. that $\lang'$ is \emph{not} at least as expressive as $\lang$. Our primary modus operandi (obtained by logically negating $\lang \leqq \lang'$) will be to show that there is a $\phi \in \lang$, where for any $\phi' \in \lang'$ we can find two pointed models $(\m,w), (\m',w')$ such that 
	\begin{align*}
		\m,w \models \phi, \ \ \m',w' \not \models \phi \quad \text{ and } \quad ( \m,w \models \phi' \Leftrightarrow \m', w' \models \phi' )
	\end{align*}
	In other words, for some $\phi \in \lang$, no matter the choice of $\phi' \in \lang'$, there will be models which $\phi$ distinguishes but $\phi'$ does not, meaning that $\phi \not \equiv \phi'$. 
	
	Our investigation will be concerned with the 7 distinct languages that are  obtained by considering each $\lang^X$ such that $X$ is a non-empty subsequence of $CDS$. In Section \ref{expressivity:section:safe-belief} our focus is on safe belief, and in Section \ref{expressivity:section:degrees-of-belief} on degrees of belief. Using these results, we provide in Section \ref{expressivity:section:summary} a full picture of the relative expressivity of each of these logics, for instance showing that we can formulate 5 distinct languages up to equal expressivity. We find this particularly remarkable in light of the fact that our notion of bisimulation is the right fit for all our logics.
	
\subsection{Expressivity of Safe Belief} \label{expressivity:section:safe-belief}
		Our first result, Proposition~\ref{expressivity:proposition:conditional-safe-identity}, shows that the conditional belief modality can be expressed in terms of the safe belief modality. Similar results can be found elsewhere in the literature, for instance in \cite[Fact 31]{demey:2011} and \cite{baltagetal.tlg3:2008}. In fact, the overall idea of reducing the binary conditional belief operator to a unary belief operator goes back to \cite{boutilier1990conditional} and \cite{lamarre1991s4}. 
		
		Below we prove that the identity found in \cite{demey:2011} is also a valid identity in our logics, which is not a given as our semantics differ in essential ways. In particular the semantics of safe belief in \cite{demey:2011} is a standard modality for $\geq_a$, whereas our semantics uses the derived relation $\succeq_a$. A more in-depth account of this matter is provided in Section \ref{section:comparison}. Returning to the matter at hand, we point out that our work in Section \ref{section:corr} actually serves our investigations here, as evident from the crucial role of Proposition~\ref{theorem:bisimimpliesmodalequivlc} in the following proof.
		\begin{proposition} \label{expressivity:proposition:conditional-safe-identity}
					Let $\phi, \psi \in L_C$. Then
			the formula $B_a^\psi \phi \leftrightarrow ( \widehat{K}_a \psi \rightarrow \widehat{K}_a( \psi \wedge \Box_a( \psi \rightarrow \phi ) ) )$ is valid.
				\end{proposition}
			\begin{proof}
				We let $\m = (W, \geq, V)$ be any plausibility model with $w \in W$, and further let $\succeq_a$ denote the normal plausibility relation for an agent $a$ in $\m$. We will show that $\m, w \models B_a^\psi \phi \leftrightarrow ( \widehat{K}_a \psi \rightarrow \widehat{K}_a( \psi \wedge \Box_a( \psi \rightarrow \phi ) ) )$. To this end we let $X = \Min_a ( \llbracket \psi \rrbracket_\m \cap [w]_a)$. Immediately we have that if $X = \emptyset$ then no world in $[w]_a$ satisfies $\psi$, thus trivially yielding both $\m,w \models B_a^\psi \phi$ and $\m, w \models \widehat{K}_a \psi \rightarrow \widehat{K}_a( \psi \wedge \Box_a( \psi \rightarrow \phi ) )$. For the remainder we therefore assume $X$ is non-empty. We now work under the assumption that $\m,w \models B_a^\psi \phi$ and show that this implies $\m, w \models \widehat{K}_a \psi \rightarrow \widehat{K}_a( \psi \wedge \Box_a( \psi \rightarrow \phi ) )$.
				
				\medskip
				\noindent \textit{Claim 1}. 
				Let $x \in X$ be arbitrarily chosen, then $\m, x \models \psi \wedge \Box_a( \psi \rightarrow \phi )$.
				
				\medskip
				\noindent \textit{Proof of claim 1.} 
				From $x \in X$ we have first that $\m,x \models \psi \wedge \phi$ and $w \sim_a x$. Since $\m,x \models \psi$ this means we have proven Claim 1 if $\m,x \models \Box_a( \psi \rightarrow \phi )$ can be shown. To that effect, consider any $y \in W$ s.t. $x \succeq_a y$, for which we must prove $\m, y \models \psi \rightarrow \phi$. When $\m, y \not \models \psi$ this is immediate, and so we may assume $\m, y \models \psi$. Since $x \succeq_a y$ we have $\Min_a ([x]_{R^=} \cap [x]_a) \geq_a \Min_a ([y]_{R^=} \cap [y]_a)$ with $R$ being the largest autobisimulation on $\m$. As $R$ is an autobisimulation we have worlds $x', y'$ in $\m$ such that $(y,y') \in R$, $x \geq_a x'$ and $x' \geq_a y'$. Applying Proposition~\ref{theorem:bisimimpliesmodalequivlc} and $\m,y \models \psi$ it follows that $\m, y' \models \psi$. Using $\geq_a$-transitivity we have $x \geq_a y'$ and hence $w \sim_a x \sim_a y'$, allowing the conclusion that $y' \in X$. By assumption this means $\m, y' \models \psi \wedge \phi$, and so applying once more Proposition~\ref{theorem:bisimimpliesmodalequivlc} it follows that $\m, y \models \psi \rightarrow \phi$ thus completing the proof of this claim.
				
				\medskip \noindent
				To show $\m, w \models \widehat{K}_a \psi \rightarrow \widehat{K}_a( \psi \wedge \Box_a( \psi \rightarrow \phi ) )$ we take any $x \in X$, for which we have $w \sim_a x$ by definition of $X$. Combining this with Claim 1 it follows that $\m, w \models \widehat{K}_a( \psi \wedge \Box_a(\psi \rightarrow \phi ) )$. Consequently this also means that $\m, w \models \widehat{K}_a \psi \rightarrow \widehat{K}_a( \psi \wedge \Box_a( \psi \rightarrow \phi ) )$, thus completing the proof of this direction.
				
				For the converse assume now that $\m, w \models \widehat{K}_a \psi \rightarrow \widehat{K}_a( \psi \wedge \Box_a( \psi \rightarrow \phi ) )$. As $X \neq \emptyset$ there is a world $u \in W$ s.t. $w \sim_a u$ and $\m, u \models \psi \wedge \Box_a(\psi \rightarrow \phi)$. Therefore we have $\m, u' \models \psi \rightarrow \phi$ for all $u \succeq_a u'$. 
				
				\medskip
				\noindent \textit{Claim 2}. 
				Let $x \in X$ be arbitrarily chosen, then $\m, x \models \phi$.
				
				\medskip
				\noindent \textit{Proof of claim 2.} % Can be proven directly
				Since $x \in X$ we have by definition that $\m, x \models \psi$. It is sufficient to prove that $u \succeq_a x$ because this implies $\m, x \models \psi \rightarrow \phi$ and hence $\m, x \models \phi$ as required. To show $u \succeq_a x$ we assume towards a contradiction that $u \not \succeq_a x$. Now let $R$ denote the largest bisimulation on $\m$ and consider any $x' \in \Min_a ([x]_{R^=} \cap [x]_a)$. As $R^=$ and $\sim_a$ are both reflexive, we have $x \geq_a x'$. From $u \not \succeq_a x$ we therefore have a $u' \in \Min_a ([u]_{R^=} \cap [u]_a)$ s.t. $u' \not \geq_a x'$, $u \sim_a u'$ and $(u,u') \in R$ (thus also $x' >_a u'$). Since $u' \sim_a u$ and $u \sim_a w$ we have also $u' \sim_a w$, and additionally from $x \geq_a x'$ and $x' >_a u'$ we can conclude that $x \geq_a u'$ and $u' \not \geq_a x$. Using $\m, u \models \psi$ and $(u,u') \in R$ we apply Proposition~\ref{theorem:bisimimpliesmodalequivlc} which implies $\m, u' \models \psi$. As $x \in X$, $u' \sim_a w$ and $x \geq_a u'$ it must be the case that $u' \in X$. From $u' \not \geq_a x$ we also have that $x \not \in X$, but this this contradicts our initial assumption that $x \in X$. We therefore have $u \succeq_a x$ and hence that $\m, x \models \phi$ which completes the proof of the claim.
				\medskip \noindent
				
				Recalling that $\m, w \models B_a^\psi \phi$ iff $\m, x \models \phi$ for all $x \in X$, Claim 2 readily shows this direction, and thereby completes the proof. 
			\end{proof}
		This result shows there is an equivalence-preserving translation from formulas in $\lang^C$ to formulas in $\lang^S$, and so we have the following results.
		\begin{corollary} \label{expressivity:corollary:conditional-translates-to-safe}
			For any $\phi \in \lang^{C}$ there is a formula $\phi' \in \lang^{S}$ s.t. $\phi \equiv \phi'$.
		\end{corollary}
		\begin{corollary} \label{expressivity:corollary:conditional-loses-safe}
			$\lang^C \leqq \lang^S$, $\lang^S \equiv \lang^{CS}$ and $\lang^{DS} \equiv \lang^{CDS}$.
		\end{corollary}
		From Corollary \ref{expressivity:corollary:conditional-loses-safe} we have that any expressivity result for $\lang^S$ also holds for $\lang^{CS}$, and similarly for $\lang^{DS}$ and $\lang^{CDS}$. In other words, the conditional belief modality is superfluous in terms of expressivity when the safe belief modality is at our disposal. What is more, we can now finally give a full proof of Proposition~\ref{theorem:modalequivimpliesbisimls}.
			\begin{proof}[Proof of Proposition~\ref{theorem:modalequivimpliesbisimls}]
				Let $(\m, w)$ and $(\m', w')$ be plausibility models which are modally equivalent in $\lang^S$. For any $\phi_C \in \lang^C$ it follows from Corollary \ref{expressivity:corollary:conditional-translates-to-safe} that there is a $\phi_S \in \lang^S$ s.t. $\phi_C \equiv \phi_S$. Therefore
				\begin{align*}
					\m,w \models \phi_C \Leftrightarrow \m,w \models \phi_S \xLeftrightarrow[]{\ \equiv^S } \m',w' \models \phi_S \Leftrightarrow  \m',w' \models \phi_C
				\end{align*}
				and hence $(\m,w) \equiv^C (\m',w')$. Using Proposition~\ref{theorem:modalequivimpliesbisimlc} we can conclude $(\m,w) \bisim (\m',w')$ as required.
			\end{proof}		
		We now proceed to show that $\lang^{CD}$ is not at least as expressive as $\lang^S$. In doing so we need only work with $\agents = \{a\}$, meaning that the result holds even in the single-agent case. This is also true for our results in Section \ref{expressivity:section:degrees-of-belief}.
		\begin{lemma} \label{expressivity:lemma:conditional-not-distinguish}
			Let $p,q$ be distinct symbols in $\props$, and let $\m = (W, \geq, V)$ and $\m' = (W', \geq', V')$ denote the two plausibility models presented in Figure \ref{expressivity:fig:modal-equiv-for-lc}. Then for $\props' = \props \setminus \{q\}$ we have that $(\m,w_3) \equiv^{CD}_{\props'} (\m', w_3')$.
				\end{lemma}
			\begin{proof}
				We prove the stronger result that for any $\phi \in \lang^{CD}_{\props'}$:
				\begin{align*}
					\text{for each } i \in \{1,2,3\} : (\m,w_i \models \phi) \Leftrightarrow ( \m', w_i' \models \phi )
				\end{align*}
				We proceed by induction on $\phi$ and let $i \in \{1,2,3\}$. When $\phi$ is a propositional symbol $r$ in $\props'$, we have that $r \neq q$ and so $r \in V(w_i)$ iff $r \in V'(w_i')$, thus completing the base case. Negation and conjunction are readily shown using the induction hypothesis. 
				
				For $\phi = K_a \psi$ we have that $\m,w_i \models K_a \psi$ iff $\m, v \models \psi$ for all $v \in \{w_1,w_2,w_3\}$, since $[w_i]_a = \{w_1,w_2,w_3\}$. Applying the induction hypothesis to each element this is equivalent to $\m',v' \models \psi$ for all $v' \in \{w_1',w_2',w_3'\}$ iff $\m',w_i' \models K_a \psi$ (as $[w_i']_a = \{w_1',w_2',w_3'\}$), which completes this case. Continuing to consider $\phi = B_a^\gamma \psi$ we can simplify notation slightly, namely $\Min_a ( \llbracket \gamma \rrbracket_\m \cap [w_i]_a) = \Min_a \llbracket \gamma \rrbracket_\m$ since $[w_i]_a = W$. The same holds for each world $w_i'$ of $\m'$. 
				
				\medskip
				\noindent \textit{Claim 1}. For $\m$ and $\m'$ we have that $w_i \in \Min_a \llbracket \gamma \rrbracket_\m$ iff $w_i' \in \Min_a \llbracket \gamma \rrbracket_{\m'}$.
				\medskip
				
				\noindent \textit{Proof of Claim 1.}
					For $\m$ we have that $w_3 >_a w_2$ and $w_2 >_a w_1$, and similarly $w_3' >_a' w_2'$ and $w_2' >_a' w_1'$ for $\m'$. Thus the claim follows from the argument below.
					\begin{align*}
						w_i \in \Min_a \llbracket \gamma \rrbracket_\m &\Leftrightarrow  \\
						\m,w_i \models \gamma \text{ and there is no } j <_a i \text{ s.t. } \m,w_j \models \gamma &\xLeftrightarrow[]{\text{(IH)}} \\
						\m',w_i' \models \gamma \text{ and there is no } j <_a i \text{ s.t. } \m',w_j' \models \gamma &\Leftrightarrow \\
						w_i' \in \Min_a \llbracket \gamma \rrbracket_{\m'} &
					\end{align*}
				
				\medskip \noindent
				We now have that $\m,w_i \models B_a^\gamma \psi$ iff $\m, v \models \psi$ for all $v \in \Min_a \llbracket \gamma \rrbracket_\m$. Applying both the induction hypothesis and Claim 1, we have that this is equivalent to $\m, v' \models \psi$ for all $v' \in \Min_a \llbracket \gamma \rrbracket_{\m'}$ iff $\m',w_i' \models B_a^\gamma \psi$. 
				
				Finally we consider the case of $\phi = B_a^n \psi$. To this end we note that the union of $\{ (w_1', w_3') \}$ and the identity relation on $W$ is the largest bisimulation on $\m'$ (this relation cannot be extended and still satify [atoms]). As $w_1'$ and $w_3'$ are bisimilar, it follows from Corollary \ref{expressivity:corollary:bisim-implies-modal-equiv-cds} that $\m', w_1' \models \psi$ iff $\m',w_3' \models \psi$ $(\ast)$. 
				
				\medskip
				\noindent \textit{Claim 2}. For $n \in \Naturals$ we have that $\m, w \models \psi$ for all $w \in \Min_a^n[w_i]$ iff $\m', w' \models \psi$ for all $w' \in \Min_a^n[w_i']$.
				\medskip
				
				\noindent \textit{Proof of Claim 2.}
					We treat three exhaustive cases for $n$.
					\begin{itemize}
						\item $n = 0$: $\m, w \models \psi$ for all $w \in \Min_a^0[w_i]$ $\Leftrightarrow$ $\m, w_1 \models \psi$ $\xLeftrightarrow[]{\text{(IH)}}$ $\m', w_1' \models \psi$ $\xLeftrightarrow[]{(\ast)}$ $\m', w_3' \models \psi$. Therefore $\m, w \models \psi$ for all $w \in \Min_a^0[w_i]$ is equivalent to $\m',w' \models \psi$ for all $w' \in \{w_1', w_3'\}$, and as $\Min_a^0[w_i'] = \{w_1', w_3'\}$ this concludes this case.
						\item $n = 1$: Since $\Min_a^1[w_i] = \{w_1, w_2\}$ we have that $\m, w \models \psi$ for all $w \in \{w_1, w_2\}$ $\xLeftrightarrow[]{\text{(IH)}}$ $\m', w' \models \psi$ for all $w' \in \{w_1'
						, w_2'\}$. Using $(\ast)$ this is equivalent to $\m',w' \models \psi$ for all $w' \in \{w_1', w_2', w_3'\}$. By this argument and the fact that $\Min_a^1[w_i'] = \{w_1', w_2', w_3'\}$, we can conclude $\m, w \models \psi$ for all $w \in \Min_a^1[w_i]$ $\Leftrightarrow$ $\m', w' \models \psi$ for all $w' \in \Min_a^1[w_i']$ as required.
						\item $n \geq 2$: We have that $\Min_a^m[w_i] = \{w_1, w_2, w_3\}$ and $\Min_a^m[w_i'] = \{w_1', w_2', w_3'\}$, hence this is exactly as the case of $\phi = K \psi$.
					\end{itemize}
				
				\medskip \noindent
				We have that $\m, w_i \models B_a^n \psi$ iff $\m, w \models \psi$ for all $w \in \Min_a^n[w_i]$. Applying Claim 2 this is equivalent to $\m', w' \models \psi$ for all $w' \in \Min_a^n[w_i']$ iff $\m', w_i' \models B_a^n \psi$, thereby completing the final case of the induction step. It follows that $(\m,w_3) \equiv^{CD}_{\props'} (\m', w_3')$ as required.
			\end{proof}
		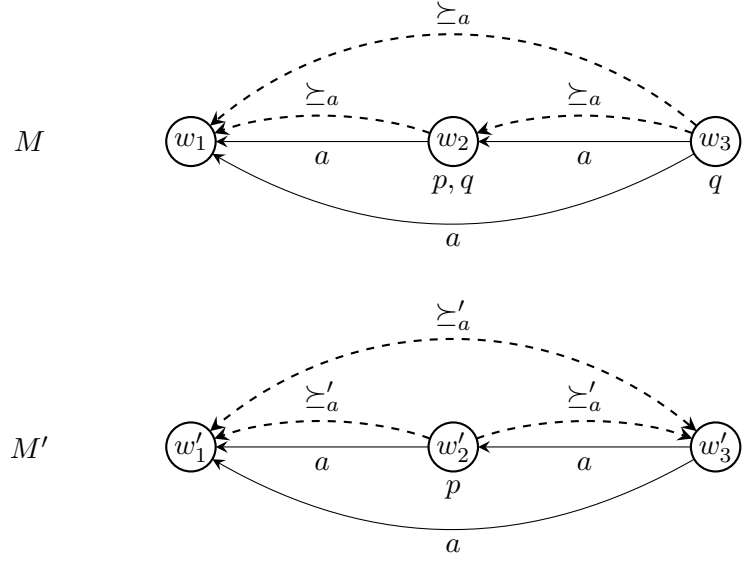
\begin{figure}[tb]
			\centering
			\def\worldX{2.8cm}
			\def\worldY{3.4cm}
			\begin{tikzpicture}
				%M
				\node[world] (w1) at (0,0) {$w_1$}
					node[lbl, below = of w1] {};
				\node[world, right = \worldX of w1] (w2) {$w_2$}
					node[lbl, below = of w2] {$p,q$};
				\node[world, right = \worldX of w2] (w3) {$w_3$}
					node[lbl, below = of w3] {$q$};
			
				\node[left of = w1, xshift=-2cm]{$\m$};
				
				%M'
				\node[world, below = \worldY of w1] (v1)  {$w_1'$}
					node[lbl, below = of v1] {};
				\node[world, right = \worldX of v1] (v2) {$w_2'$}
					node[lbl, below = of v2] {$p$};
				\node[world, right = \worldX of v2] (v3) {$w_3'$}
					node[lbl, below = of v3] {};
					
				\node[left of = v1, xshift=-2cm]{$\m'$};
				
				% \geq
				\path[link] (w3) -- node[below]{$a$} (w2);
				\path[link] (w2) -- node[below]{$a$} (w1);
				\draw[link] (w3) to [out = 210, in = -30, looseness = 1.0] node[below]{$a$} (w1);
				
				\path[link] (v3) -- node[below]{$a$} (v2);
				\path[link] (v2) -- node[below]{$a$} (v1);
				\draw[link] (v3) to [out = 210, in = -30, looseness = 1.0] node[below]{$a$} (v1);
				
				%\geq_a^M
				\draw[thick, dashed, link] (w2) to [out = 160, in = 20, looseness = 0.8] node[above]{$\succeq_a$} (w1);
				\draw[thick, dashed, link] (w3) to [out = 160, in = 20, looseness = 0.8] node[above]{$\succeq_a$} (w2);
				\draw[thick, dashed, link] (w3) to [out = 140, in = 40, looseness = 1.0] node[above]{$\succeq_a$} (w1);
				
				%\geq_a^M'
				\draw[thick, dashed, link] (v2) to [out = 160, in = 20, looseness = 0.8] node[above]{$\succeq_a'$} (v1);
				\draw[thick, dashed, link] (v2) to [out = 20, in = 160, looseness = 0.8] node[above]{$\succeq_a'$} (v3);
				\draw[thick, dashed, ke-ke] (v3) to [out = 140, in = 40, looseness = 1.0] node[above]{$\succeq_a'$} (v1);

			\end{tikzpicture}
			\caption{Two single-agent plausibility models and their normal plausibility relations (dashed arrows). As usual reflexive arrows are omitted. }
			\label{expressivity:fig:modal-equiv-for-lc}
		\end{figure}
		\begin{proposition} \label{expressivity:proposition:safe-beat-conditional}
			$\lang^{S} \not \leqq \lang^{CD}$.
			\end{proposition}
			\begin{proof}
				Consider the formula $\Diamond_a p$ of $\lang^{S}$ with $p \in \props$, and take some arbitrary formula $\phi_{CD} \in \lang_\props^{CD}$. As $\phi_{CD}$ is finite and $\props$ is countably infinite, there will be \emph{some} $q \neq p$ not occurring in $\phi_{CD}$. Letting $\props' = \props \setminus \{q\}$ this means that $\phi_{CD} \in \lang_{\props'}^{CD}$. This choice of $p$ and $q$ can always be made, and consequently there also exists models $\m$ and $\m'$ as given in Figure \ref{expressivity:fig:modal-equiv-for-lc}. The largest bisimulation on $\m$ is the identity as no two worlds have the same valuation. At the same time $\{(w_1', w_1'), (w_1', w_3'), (w_2', w_2'), (w_3’,w_1’), (w_3', w_3')\}$ is the largest bisimulation on $\m'$. This gives rise to the normal plausibility relations $\succeq_a$ (for $\m$) and $\succeq_a'$ (for $\m'$) depicted in Figure \ref{expressivity:fig:modal-equiv-for-lc} using dashed edges. 
				
				Since $w_3 \succeq_a w_2$ and $\m,w_2 \models p$ it follows that $\m,w_3 \models \Diamond_a p$. Furthermore we have that the image of $w_3'$ under $\succeq_a'$ is $\{w_1', w_3'\}$. This means that there is no $v' \in W'$ s.t. $w_3' \succeq_a' v'$ and $\m', v' \models p$, and consequently $\m', w_3' \not \models \Diamond_a p$. At the same time we have by Lemma \ref{expressivity:lemma:conditional-not-distinguish} that $\m, w_3 \models \phi_{CD}$ iff $\m', w_3' \models \phi_{CD}$. Therefore using the formula $\Diamond_a p$ of $\lang^{S}$, for any formula of $\phi_{CD} \in \lang^{CD}$ there are models which $\Diamond_a p$ distinguishes but $\phi_{CD}$ does not, and so $\Diamond_a p \not \equiv \phi_{CD}$. Consequently we have $\lang^{S} \not \leqq \lang^{CD}$ as required.
			\end{proof}
		To further elaborate on this result, what is really being put to use here is the ability of the safe belief modality to (at least in part) talk about propositional symbols that do not occur in a formula. This is an effect of the derived relation $\succeq_a$ depending on the largest bisimulation.

	\subsection{Expressivity of Degrees of Belief} \label{expressivity:section:degrees-of-belief}
		We have now settled that safe belief is at least as expressive as conditional belief, and further that the combination of the conditional belief modality and the degrees of belief modality does not allow us to express the safe belief modality. A hasty conclusion would be that the safe belief modality is the one modality to rule them all, but this is not so. In fact $\lang^{S}$ (equivalent to $\lang^{CS}$ cf. Corollary \ref{expressivity:corollary:conditional-loses-safe}) falls short when it comes to expressing degrees of belief, which we now continue to prove.
		\begin{lemma} \label{expressivity:lemma:conditionalsafe-not-distinguish}
			Let $p,q$ be distinct symbols in $\props$, and let $\m = (W, \geq, V)$ and $\m' = (W', \geq', V')$ denote the two plausibility models presented in Figure \ref{expressivity:fig:modal-equiv-for-lcs}. Then for $\props' = \props \setminus \{q\}$ we have that $(\m,x_1) \equiv^{S}_{\props'} (\m', x')$.
				\end{lemma}
			\begin{proof}
				We will show the following stronger version of this lemma: For $i \in \{1,2\}: $ $(\m,x_i) \equiv^{S}_{\props'} (\m', x')$ and $(\m,y) \equiv^{S}_{\props'} (\m', y')$. We proceed by induction on $\phi \in \lang^{S}_{\props'}$, showing that:
				\begin{align}
					\text{ for } i \in \{1,2\}: \m,x_i \models \phi \text{ iff } \m', x' \models \phi \qquad \text{ and }  \qquad \m,y \models \phi \text{ iff } \m', y' \models \phi \label{expressivity:lemma:conditionalsafe-not-distinguish:property}
				\end{align}
				For the base case we have $\phi = r$ for some $r \in \props \setminus \{q\}$. Because $r \neq q$ it is clear that $r \in V(x_1)$ iff $r \in V'(x')$. Since we also have $V(x_2) = V'(x')$ and $V'(y) = V(y')$ this completes the base case. The cases of negation and conjunction are readily established using the induction hypothesis, and $\phi = K_a \psi$ is shown just as we did in the proof of Lemma \ref{expressivity:lemma:conditional-not-distinguish}. Before proceeding we recall that $\agents = \{a\}$ and note that for any $w \in W$ we have $[w]_a = \{x_1,x_2,y\}$, as well as $[w']_a = \{x',y'\}$ for any $w' \in W'$. Moreover, the largest bisimulation on $\m$ and $\m'$ respectively is the identity relation, meaning that $\geq_a = \succeq_a$ and $\geq_a' = \succeq_a'$. For the case of $\phi = \Box_a \psi$ we can therefore argue as follows.
				\begin{align*}
					&\m,x_1 \models \Box_a \psi \Leftrightarrow \m, x_1 \models \psi \xLeftrightarrow[]{\text{(IH)}} \m', x' \models \psi \Leftrightarrow \m', x' \models \Box_a \psi \\
					&\m, x_2 \models \Box_a \psi \Leftrightarrow ( \forall i \in \{1,2\}: \m, x_i \models \psi ) \xLeftrightarrow[]{\text{(IH)}} \m', x' \models \psi \Leftrightarrow \m', x' \models \Box_a \psi \\
					&\m, y \models \Box_a \psi \Leftrightarrow ( \forall w \in W: \m, w \models \psi ) \xLeftrightarrow[]{\text{(IH)}} ( \forall w' \in W': \m', w' \models \psi ) \Leftrightarrow \m', y' \models \Box_a \psi
				\end{align*}
				In fact the last line is essentially the case of $K_a \psi$, as the image of $y$ under $\succeq_a$ is $W$ (and $W'$ is the image of $y'$ under $\succeq_a'$). This concludes our proof by induction, shows (\ref{expressivity:lemma:conditionalsafe-not-distinguish:property}) and allows us to conclude that $(\m,x_1) \equiv^{S}_{\props'} (\m', x')$.
			\end{proof}
		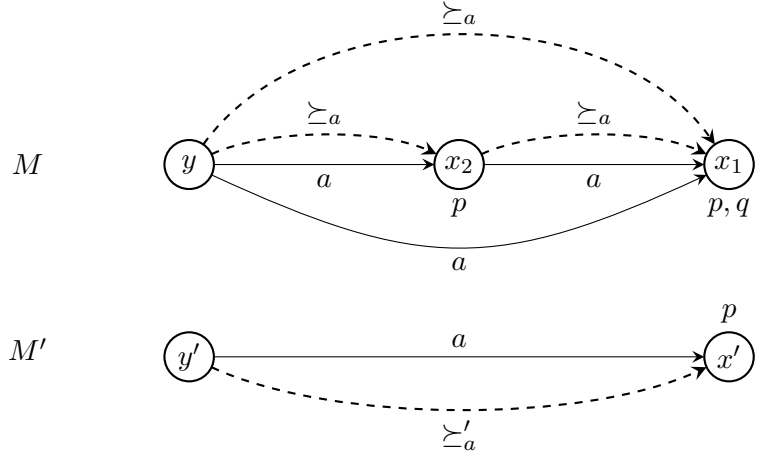
\begin{figure}[tb]
			\def\worldX{2.9cm}
			\def\sepY{1.9cm}
			\centering
			\begin{tikzpicture}
				% M
				\node[world] (w1) {$x_1$}
					node[lbl, below = of w1] {$p, q$};
				\node[world, left = \worldX of w1] (w2) {$x_2$}
					node[lbl, below = of w2] {$p$};
				\node[world, left = \worldX of w2] (w3) {$y$}
					node[lbl, below = of w3] {};
				\node[left of = w3, xshift=-2cm] (mlbl) {$\m$};
				
				\path[link] (w3) -- node[below] {$a$} (w2);
				\path[link] (w2) -- node[below] {$a$} (w1);
				\path[link] (w3) to [out = -25, in = 205, looseness = 1.2] node[below]{$a$} (w1);
				
				%\succeq
				\draw[thick, dashed, link] (w2) to [out = 25, in = 155, looseness = 0.7] node[above]{$\succeq_a$} (w1);
				\draw[thick, dashed, link] (w3) to [out = 25, in = 155, looseness = 0.7] node[above]{$\succeq_a$} (w2);
				\draw[thick, dashed, link] (w3) to [out = 55, in = 125, looseness = 0.9] node[above]{$\succeq_a$} (w1);
				
				% M'
				\node[world, below = \sepY of w1] (v1) {$x'$}
					node[lbl, above = of v1] {$p$};
				\node[world, below = \sepY of w3] (v2) {$y'$}
					node[lbl, above = of v2] {};
				\node[below = \sepY of mlbl]{$\m'$};
				
				%\geq
				\path[link] (v2) -- node[above] {$a$} (v1);
				\draw[thick, dashed, link] (v2) to [out = -25, in = 205, looseness = 0.7] node[below]{$\succeq_a'$} (v1);
				
			\end{tikzpicture}
			\caption{Two single-agent plausibility models and their normal plausibility relations (dashed arrows). As usual reflexive arrows are omitted. }
			\label{expressivity:fig:modal-equiv-for-lcs}
		\end{figure}
		\begin{proposition} \label{expressivity:proposition:degrees-beat-safe}
			$\lang^{D} \not \leqq \lang^{S}$.
			\end{proposition}
			\begin{proof}
				Consider the formula $B_a^1 p \in \lang^{D}$ with $p \in \props$, and additionally take any formula $\phi_{S} \in \lang_\props^{S}$. As $\phi_{S}$ is finite and $\props$ is countably infinite, there will be \emph{some} $q \neq p$ which does not occur in $\phi_{S}$. With $\props' = \props \setminus \{q\}$ we therefore have $\phi_{S} \in \lang_{\props'}^{S}$. As we can always make such a choice of $p$ and $q$, this means that there always exists models $(\m, x_1)$, $(\m',x')$ of the form given in Figure \ref{expressivity:fig:modal-equiv-for-lcs}. 
				
				As in the proof of Lemma \ref{expressivity:lemma:conditionalsafe-not-distinguish} the largest bisimulation on $\m$ and $\m'$ is the identity and so $\Min_a^1[x_1]_a = \{x_1, x_2\}$ and $\Min_a^1[x']_a = \{x', y'\}$. Consequently $\m,x_1 \models B_a^1 p$ whereas $\m',x' \not \models B_a^1 p$. Since $\phi_{S} \in \lang_{\props'}^{S}$ it follows from Lemma \ref{expressivity:lemma:conditionalsafe-not-distinguish} that $\m, x \models \phi_{S}$ iff $\m', x' \models \phi_{S}$. What this proves is that using the formula $B_a^1 p$ of $\lang^{D}$, no matter the choice of formula $\phi_{S}$ of $\lang^{S}$ there will be models which $B_a^1 p$ distinguishes but $\phi_{S}$ does not, hence $B_a^1 p \not \equiv \phi_{S}$. From this follows $\lang^{D} \not \leqq \lang^{S}$ as required. 
			\end{proof}
		We find that this result is quite surprising. Again it is a consequence of our use of the largest bisimulation when defining our semantics. The purpose of $x_1$ in model $\m$ (which is otherwise identical to $\m'$) is to inject an additional belief sphere without adding any factual content from $\props'$, as that could allow the safe belief formula $\phi_S$ to distinguish $x_1$ from $x_2$. 

		At this point it might seem as if all hope was lost for the conditional belief modality, however our final direct result somewhat rebuilds the reputation of this hard-pressed modality. To this end we define for any $k \in \Naturals$ the language $\lang^{Dk}$, which contains every formula of $\lang^D$ for which if $B^n_a \phi$ occurs then $n \leq k$. In other words formulas of $\lang^{Dk}$ talk about belief to at most degree $k$, which comes in handy as we investigate the relative expressive power of $\lang^D$ and $\lang^C$.		
		
		\def\nnn{N^k}
		
		\begin{lemma}\label{expressivity:lemma:degrees-not-distinguish}
			Let $k \in \Naturals$ be given, and let $(\m^k, w_0)$ and $(\nnn, w_0')$ denote the two plausibility models presented in Figure \ref{expressivity:fig:conditional-beat-degrees}. Then we have that $(\m^k, w_0)$ and $(\nnn, w_0')$ are modally equivalent in $\lang^{Dk}$.
			\end{lemma}
			\begin{proof}
				We prove a stronger version of this lemma, namely that 
				$(\m^k, w_i)\equiv^{Dk} (\nnn, w_i')$ for $0 \leq i \leq k$, 
				$(\m^k, x) \equiv^{Dk} (\nnn, x')$ and 
				$(\m^k, y) \equiv^{Dk} (\nnn, y')$.
				
				Key to this proof is the fact that $x$ (resp. $y$) has the same valuation as $x'$ (resp. $y'$), and that $x$ is more plausible than $y$ whereas $y'$ is more plausible than $x'$. We proceed by induction on $\phi \in \lang^{Dk}$. In the base case $\phi$ is a propositional symbol, and so as the valuation of each $w_i$ matches that of $w_i'$ ($0 \leq i \leq k$), $x$ matches $x'$ and $y$ matches $y'$ this completes the base case. The cases of negation and conjunction readily follow using the induction hypothesis, and for $\phi = K_a \psi$ the argument is essentially that used in the proof of Lemma \ref{expressivity:lemma:conditional-not-distinguish}.

				Lastly we consider $\phi = B_a^j \psi$ for any $0 \leq j \leq k$, and recall that this is sufficient as $\phi \in \lang_\props^{Dk}$. As neither model contains two worlds with the same valuation, the largest autobisimulation on either model is the identity, and so both models are normal. With the epistemic relation of agent $a$ being total, we have for all $w \in W$ that $\Min_a^{j} [w]_a = \{w_0, \ldots, w_j\}$ and similarly for all $w' \in W'$ that $\Min_a^{j} [w']_a = \{w_0', \ldots, w_j'\}$. We therefore have
				\begin{align*}
					\forall w \in W: \m^k, w \models B_a^j \psi \Leftrightarrow \forall v \in \{w_0, \ldots, w_j\} &: \m^k, v \models \psi \xLeftrightarrow[]{\text{(IH)}} \\
					\forall v' \in \{w_0', \ldots, w_j'\}&: \nnn, v' \models \psi \Leftrightarrow \ \forall w' \in W' : \nnn,w' \models B_a^j \psi
				\end{align*}
				as required. Observe that we can apply the induction hypothesis since $j \leq k$, and that importantly $x$, $y$ are not in $\Min_a^{j}[w]_a$, and $x'$, $y'$ are not in $\Min_a^{j}[w']_a$. Thus we have shown that $(\m^k, w_0) \equiv^{Dk} (\nnn, w_0')$ thereby completing the proof.
			\end{proof}		
		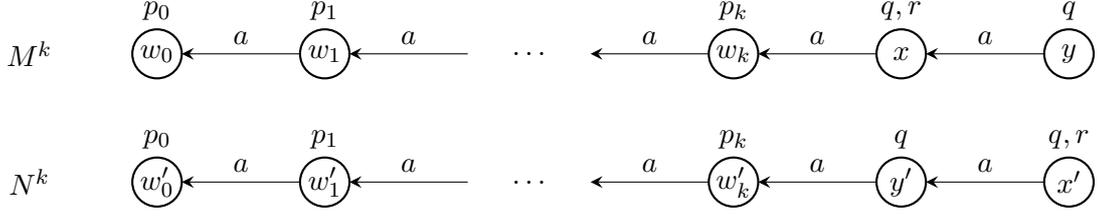
\begin{figure}[tb]
			\centering
			\def\worldX{\textwidth/10}
			\begin{tikzpicture}
				\begin{scope}[xshift = 0cm] %, wrapper={r}{}]
					\node[world] (w1) {$w_0$}
						node[lbl, above = of w1] {$p_0$};
					\node[world, right = \worldX of w1] (w2) {$w_1$}
						node[lbl, above = of w2] {$p_1$};
					\node[right = \worldX of w2, minimum width=1.6cm] (wdots) {$\ldots$}
						node[lbl, above = of wdots] {};
					\node[world, right = \worldX of wdots] (wk) {$w_k$}
						node[lbl, above = of wk] {$p_k$};
					\node[world, right = \worldX of wk] (x) {$x$}
						node[lbl, above = of x] {$q,r$};
					\node[world, right = \worldX of x] (y) {$y$}
						node[lbl, above = of y] {$q$};
						
					\path[link] (x)  -- node[above] {$a$} (wk);
					\path[link] (y)  -- node[above] {$a$} (x);
					\path[link] (wk) -- node[above] {$a$} (wdots);
					\path[link] (wdots) -- node[above] {$a$} (w2);
					\path[link] (w2) -- node[above] {$a$} (w1);
					\node[left of = w1, xshift=-\worldX]{$\m^k$};
				\end{scope}
				
				\begin{scope}[yshift = -1.7cm] %, wrapper={m}{}]
					\node[world] (w1) {$w_0'$}
						node[lbl, above = of w1] {$p_0$};
					\node[world, right = \worldX of w1] (w2) {$w_1'$}
						node[lbl, above = of w2] {$p_1$};
					\node[right = \worldX of w2, minimum width=1.6cm] (wdots) {$\ldots$}
						node[lbl, above = of wdots] {};
					\node[world, right = \worldX of wdots] (wk) {$w_k'$}
						node[lbl, above = of wk] {$p_k$};
						
					\node[world, right = \worldX of wk] (x) {$y'$}
						node[lbl, above = of x] {$q$};
					\node[world, right = \worldX of x] (y) {$x'$}
						node[lbl, above = of y] {$q,r$};
						
					\path[link] (x) -- node[above] {$a$} (wk);
					\path[link] (y) -- node[above] {$a$} (x);
					\path[link] (wk) -- node[above] {$a$} (wdots);
					\path[link] (wdots) -- node[above] {$a$} (w2);
					\path[link] (w2) -- node[above] {$a$} (w1);
					\node[left of = w1, xshift=-\worldX]{$\nnn$};
				\end{scope}
			\end{tikzpicture}
			\caption{Two single-agent plausibility models. We've omitted reflexive arrows and for the sake of readability also some transitive arrows. }
			\label{expressivity:fig:conditional-beat-degrees}
		\end{figure}
		\begin{proposition} \label{expressivity:proposition:conditional-beat-degrees}
			$\lang^{C} \not \leqq \lang^{D}$.
				\end{proposition}
			\begin{proof}
				Consider now $B_a^q r$ belonging to $\lang^{C}$ and any formula $\phi_D \in \lang^{D}$. Since $\phi_D$ is finite we can choose some $k \in \Naturals$ such that $\phi_D \in \lang^{Dk}$. Because $p_0, \ldots, p_k, q, r$ are taken from the countably infinite set $\props$, no matter the choice of $k$ there exists pointed plausibility models $(\m^k,w_0)$ and $(\nnn,w_0')$ as presented in Figure \ref{expressivity:fig:conditional-beat-degrees}. 
								
				To determine the truth of $B_a^q r$ in $(\m^k,w_0)$ and $(\nnn,w_0')$ respectively we point out that $\II{q}_{\m^k} = \{x,y\}$ and $\II{q}_{\nnn} = \{y',x'\}$. Therefore we have that $\Min_a(\II{q}_{\m^k} \cap [w_0]_a) = \{x\}$ and $\Min_a(\II{q}_{\nnn} \cap [w_0']_a) = \{y'\}$. Since $\m^k, x \models r$ and $\nnn, y' \not \models r$, it follows $\m^k,w_0 \models B_a^q r$ whereas $\nnn,w_0' \not \models B_a^q r$. By Lemma \ref{expressivity:lemma:degrees-not-distinguish} we have that $\m^k, w_0 \models \phi_D$ iff $\nnn, w_0' \models \phi_D$. With this we have shown that taking the formula $B_a^q r$ of $\lang^{C}$, there are for any $\phi_D \in \lang^D$ pointed plausibility models which $B_a^q r$ distinguishes but $\phi_D$ does not, thus $B_a^q r \not \equiv \phi_D$. It follows that $\lang^{C} \not \leqq \lang^{D}$ as required.
			\end{proof}
		We have now shown that the degrees of belief modality cannot capture the conditional belief modality. What this really showcases is that for $B_a^\psi \phi$, $\psi$ potentially enables us to talk about worlds of arbitrarily large degree. This sets it apart from the degrees of belief modality, and causes for instance a difference in expressivity.
		
	\subsection{Mapping Out the Relative Expressive Power}\label{expressivity:section:summary}

		\pgfkeysdef{/lbl/sublang}{(a)}		
		\pgfkeysdef{/lbl/trans}{(b)}
		\pgfkeysdef{/lbl/con1}{(c)}
		\pgfkeysdef{/lbl/con2}{(d)}
		\pgfkeysdef{/lbl/con3}{(e)}
		
% For unknown reasons this enters an infinite loop when compiling
%\edef\weirdnecessity{\pgfkeys{/lbl/sublang}, \pgfkeys{/lbl/trans}, \pgfkeys{/lbl/con2} and \pgfkeys{/lbl/con3}}
% IMPORTANT: When changing above keys, also change this string accordingly.
\edef\weirdnecessity{(a), (b), (d) and (e) }
		\pgfkeys{/exp prop/.code={\pgfkeys{/lbl/#1}}% (Proposition \ref{expressivity:proposition:leqq-properties})
		}
		
		With the results we have now shown, we are in fact able to determine the relative expressivity of all our languages. To this end we make use of the following facts related to expressivity, where we let $\lang$, $\lang'$ and $\lang''$ denote logical languages interpreted on the same class of models: 
			\begin{enumerate}
				\item[\pgfkeys{/lbl/sublang}] If $\lang$ is a sublanguage of $\lang'$ then $\lang \leqq \lang'$.
				\item[\pgfkeys{/lbl/trans}] If $\lang \leqq \lang'$ and $\lang' \leqq \lang''$ then $\lang \leqq \lang''$ (transitivity).
				\item[\pgfkeys{/lbl/con1}] If $\lang \equiv \lang'$ then $\lang \leqq \lang''$ iff $\lang' \leqq \lang''$ (transitivity consequence 1).
				\item[\pgfkeys{/lbl/con2}] If $\lang \leqq \lang'$ and $\lang'' \not \leqq \lang'$ then $\lang'' \not \leqq \lang$ (transitivity consequence 2).
				\item[\pgfkeys{/lbl/con3}] If $\lang \leqq \lang'$ and $\lang \not \leqq \lang''$ then $\lang' \not \leqq \lang''$ (transitivity consequence 3).
			\end{enumerate}
		
\newcounter{ctEqn}
\setcounter{ctEqn}{0}
%Args: {subkey}{equation}
\newcommand{\lblkey}[3][\empty]{%
	\stepcounter{ctEqn}%
	\pgfkeysedef{/lbl/#2/no}{\arabic{ctEqn}}%
	\pgfkeysdef{/lbl/#2/eqn}{#3}%
	\IfCharInString{Y}{#1}{\pgfkeyssetvalue{/lbl/#2/highlight}{def}}{}
}
\pgfkeys{/no/.code={%
	\pgfkeysifdefined{/lbl/#1/highlight}{
		\textbf{(\pgfkeys{/lbl/#1/no})}
	}{
		(\pgfkeys{/lbl/#1/no})
	}
}}
\pgfkeys{/eqn/.code={$\pgfkeys{/lbl/#1/eqn}$}}
\pgfkeys{/arg/.code={$\pgfkeys{/lbl/#1/eqn}$ from (\pgfkeys{/lbl/#1/no}})}
%\pgfkeys{/arg/.code={(\pgfkeys{/lbl/#1/no}})}

\lblkey{C-leqq-S}{\lang^C \leqq \lang^S}
\lblkey{S-not-leqq-CD}{\lang^S \not \leqq \lang^{CD}}
\lblkey{S-not-leqq-C}{\lang^S \not \leqq \lang^C}
\lblkey[Y]{C-<-S}{\lang^C < \lang^S}

\lblkey{D-not-leqq-S}{\lang^D \not \leqq \lang^S}
\lblkey{S-not-leqq-D}{\lang^S \not \leqq \lang^D}
\lblkey[Y]{D-bowtie-S}{\lang^D \bowtie \lang^S}

\lblkey{C-not-leqq-D}{\lang^{C} \not \leqq \lang^{D}}
\lblkey{D-not-leqq-C}{\lang^D \not \leqq \lang^C}
\lblkey[Y]{C-bowtie-D}{\lang^C \bowtie \lang^D}

% c vs cd
% d vs cd
\lblkey{CD-not-leqq-C}{\lang^{CD} \not \leqq \lang^{C}}
\lblkey[Y]{C-<-CD}{\lang^C < \lang^{CD}}
\lblkey{CD-not-leqq-D}{\lang^{CD} \not \leqq \lang^{D}}
\lblkey[Y]{D-<-CD}{\lang^D < \lang^{CD}}
% s vs cd
\lblkey{CD-not-leqq-S}{\lang^{CD} \not \leqq \lang^S}
\lblkey[Y]{S-bowtie-CD}{\lang^S \bowtie \lang^{CD}}

% c vs ds
% d vs ds
\lblkey{CDS-leqq-DS}{\lang^{CDS} \leqq \lang^{DS}}
\lblkey{C-leqq-DS}{\lang^{C} \leqq \lang^{DS}}
\lblkey{DS-not-leqq-C}{\lang^{DS} \not \leqq \lang^{C}}
\lblkey[Y]{C-<-DS}{\lang^{C} < \lang^{DS}}

\lblkey{DS-not-leqq-D}{\lang^{DS} \not \leqq \lang^{D}}
\lblkey[Y]{D-<-DS}{\lang^{D} < \lang^{DS}}

%s vs ds
\lblkey{CD-leqq-DS}{\lang^{CD} \leqq \lang^{DS}}
\lblkey{DS-not-leqq-S}{\lang^{DS} \not \leqq \lang^S}
\lblkey[Y]{S-<-DS}{\lang^{S} < \lang^{DS}}

\lblkey{CD-not-leqq-DS}{\lang^{CD} \not \leqq \lang^{DS}}
\lblkey[Y]{CD-<-DS}{\lang^{CD} < \lang^{DS}}

		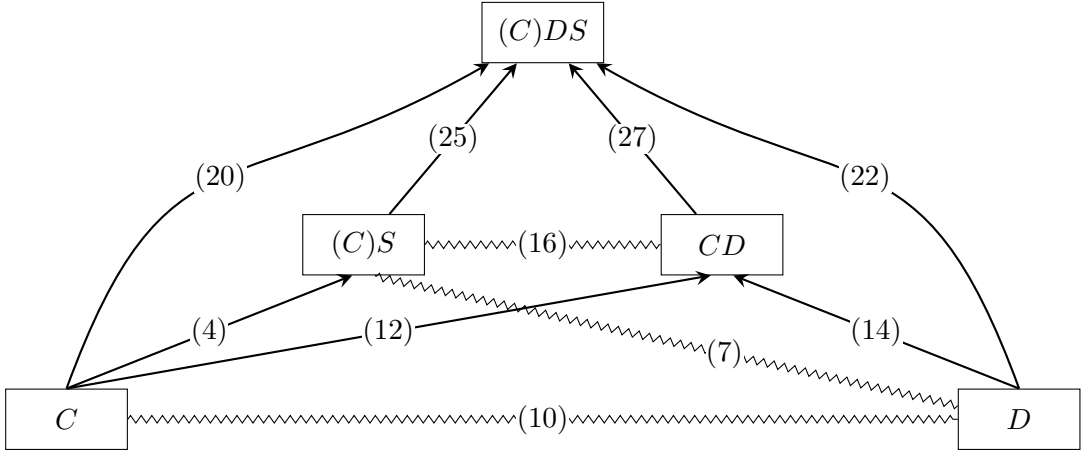
\begin{figure}[tb]
			\centering
			\def\langSepY{1.5cm}
			\begin{tikzpicture}
				\tikzset{lang/.style={draw, black, rectangle, inner sep = 0.19cm, minimum height=.8cm, minimum width=1.6cm}}
				\tikzset{derivstep/.style={fill=white, inner sep = 0, outer sep = 0}}
				\tikzset{more expressive/.style={draw,-ke, line width = .8}}
				\tikzset{incomparable/.style={draw, decorate, decoration={zigzag, amplitude=.05cm, segment length=0.14cm}}}
				
				\node[lang] (s) {$(C)S$};
				\node[lang, right = 3.1cm of s] (cd) {$CD$};
				\node[lang, above = 2cm of cd] at ($(cd.north)!0.5!(s.north)$) (ds) {$(C)DS$};
				\node[lang, below left = 1.5cm and 3.1cm of s.south] (c) {$C$};
				\node[lang, below right = 1.5cm and 3.1cm of cd.south] (d) {$D$};
				
				\draw[more expressive] (c.north) -- node[derivstep, pos=.5] {(\pgfkeys{/lbl/C-<-S/no})} (s.250);
				\draw[more expressive] (c.north) -- node[derivstep, pos=.5] {(\pgfkeys{/lbl/C-<-CD/no})} (cd.250);
				\draw[more expressive] (c.north) to [out = 70, in = 210, looseness = 1.4] node[derivstep, pos=.5] {(\pgfkeys{/lbl/C-<-DS/no})} (ds);
				\draw[more expressive] (d.north) -- node[derivstep, pos=.5] {(\pgfkeys{/lbl/D-<-CD/no})} (cd.290);
				\draw[more expressive] (d.north) to [out = 110, in = 330, looseness = 1.4] node[derivstep, pos=.5] {(\pgfkeys{/lbl/D-<-DS/no})} (ds);
				\draw[more expressive] (s) -- node[derivstep, pos=.5] {(\pgfkeys{/lbl/S-<-DS/no})} (ds);
				\draw[more expressive] (cd) -- node[derivstep, pos=.5] {(\pgfkeys{/lbl/CD-<-DS/no})} (ds);
				
				\draw[incomparable] (c) -- node[derivstep, pos=.5] {(\pgfkeys{/lbl/C-bowtie-D/no})} (d);
				\draw[incomparable] (d) -- node[derivstep, pos=.4] {(\pgfkeys{/lbl/D-bowtie-S/no})} (s.290);
				\draw[incomparable] (s) -- node[derivstep, pos=.5] {(\pgfkeys{/lbl/S-bowtie-CD/no})} (cd);
			\end{tikzpicture}
			%The absence of an arrow from $X$ to $X'$ indicates that $\lang^X$ is not at least as expressive as $\lang^{X'}$.
			\caption{Summary of expressivity results for our logics. An arrow $X \xrightarrow{\quad\ } X'$ indicates that $\lang^{X'}$ is more expressive than $\lang^{X}$. A zig-zag line between $X$ and $X'$ means that $\lang^X$ and $\lang^{X'}$ are incomparable. The abbreviation $(C)DS$ means both $CDS$ and $DS$, and similarly for $C(S)$ indicating both $CS$ and $S$. Labels on arrows and zig-zag lines signify from where the result is taken in Table \ref{expressivity:table:derivation-of-full-map}.
			%\mhj{Alternatively, remove zig-zag lines (including their labels) and then write: When $X$ and $X'$ are not linked they are incomparable.}. 
			}
			\label{expressivity:fig:summary}
		\end{figure}

\begin{table}[h!]
\centering
\renewcommand{\arraystretch}{1.4}
{\footnotesize
\begin{tabularx}{0.95\textwidth}{|c|l|X|}
  \hline
  \textbf{\#} & \textbf{Result} & \textbf{Inferred from} \\
  \hline
  \hline
  \pgfkeys{/no = C-leqq-S}  		& \pgfkeys{/eqn = C-leqq-S} 		& Corollary \ref{expressivity:corollary:conditional-loses-safe}.  \\
  \pgfkeys{/no = S-not-leqq-CD}  	& \pgfkeys{/eqn = S-not-leqq-CD}	& Proposition \ref{expressivity:proposition:safe-beat-conditional}. \\
  \pgfkeys{/no = S-not-leqq-C}  	& \pgfkeys{/eqn = S-not-leqq-C} 	& $\lang^C \leqq \lang^{CD}$ from \pgfkeys{/exp prop = sublang}, \pgfkeys{/arg = S-not-leqq-CD} and applying \pgfkeys{/exp prop = con2}. \\
  \pgfkeys{/no = C-<-S}  			& \pgfkeys{/eqn = C-<-S} 			& \pgfkeys{/arg = C-leqq-S}, \pgfkeys{/arg = S-not-leqq-C}. \\
  \hline
  \pgfkeys{/no = D-not-leqq-S}  	& \pgfkeys{/eqn = D-not-leqq-S} 	& Proposition \ref{expressivity:proposition:degrees-beat-safe}. \\
  \pgfkeys{/no = S-not-leqq-D}  	& \pgfkeys{/eqn = S-not-leqq-D} 	& $\lang^D \leqq \lang^{CD}$ from \pgfkeys{/exp prop = sublang}, \pgfkeys{/arg = S-not-leqq-CD} and applying \pgfkeys{/exp prop = con2}. \\
  \pgfkeys{/no = D-bowtie-S}  		& \pgfkeys{/eqn = D-bowtie-S} 		& \pgfkeys{/arg = D-not-leqq-S}, \pgfkeys{/arg = S-not-leqq-D}. \\
  \hline
  \pgfkeys{/no = C-not-leqq-D}  	& \pgfkeys{/eqn = C-not-leqq-D} 	& Proposition \ref{expressivity:proposition:conditional-beat-degrees}.  \\
  \pgfkeys{/no = D-not-leqq-C}  	& \pgfkeys{/eqn = D-not-leqq-C} 	& \pgfkeys{/arg = C-leqq-S}, \pgfkeys{/arg = D-not-leqq-S} and applying \pgfkeys{/exp prop = con2}. \\
  \pgfkeys{/no = C-bowtie-D}  		& \pgfkeys{/eqn = C-bowtie-D} 		& \pgfkeys{/arg = C-not-leqq-D}, \pgfkeys{/arg = D-not-leqq-C}. \\
  \hline
  \pgfkeys{/no = CD-not-leqq-C}  	& \pgfkeys{/eqn = CD-not-leqq-C} 	& $\lang^{D} \leqq \lang^{CD}$ from \pgfkeys{/exp prop = sublang}, \pgfkeys{/arg = D-not-leqq-C} and applying \pgfkeys{/exp prop = con3}. \\ 
  \pgfkeys{/no = C-<-CD}  			& \pgfkeys{/eqn = C-<-CD} 			& $\lang^{C} \leqq \lang^{CD}$ from \pgfkeys{/exp prop = sublang}, \pgfkeys{/arg = CD-not-leqq-D}. \\  
  \hline
  \pgfkeys{/no = CD-not-leqq-D}  	& \pgfkeys{/eqn = CD-not-leqq-D} 	& $\lang^{C} \leqq \lang^{CD}$ from \pgfkeys{/exp prop = sublang}, \pgfkeys{/arg = C-not-leqq-D} and applying \pgfkeys{/exp prop = con3}. \\ 
  \pgfkeys{/no = D-<-CD}  			& \pgfkeys{/eqn = D-<-CD} 			& $\lang^{D} \leqq \lang^{CD}$ from \pgfkeys{/exp prop = sublang}, \pgfkeys{/arg = CD-not-leqq-D}. \\  
  \hline
  \pgfkeys{/no = CD-not-leqq-S}  	& \pgfkeys{/eqn = CD-not-leqq-S} 	& $\lang^{D} \leqq \lang^{CD}$ from \pgfkeys{/exp prop = sublang}, \pgfkeys{/arg = D-not-leqq-S} and applying \pgfkeys{/exp prop = con3}. \\
  \pgfkeys{/no = S-bowtie-CD}  		& \pgfkeys{/eqn = S-bowtie-CD} 		& \pgfkeys{/arg = S-not-leqq-CD}, \pgfkeys{/arg = CD-not-leqq-S}. \\
  \hline
  \pgfkeys{/no = CDS-leqq-DS}  		& \pgfkeys{/eqn = CDS-leqq-DS} 		& $\lang^{CDS} \equiv \lang^{DS}$ from Corollary \ref{expressivity:corollary:conditional-loses-safe} and Definition \ref{expressivity:definition:expressivity}.\\
 \pgfkeys{/no = C-leqq-DS}  		& \pgfkeys{/eqn = C-leqq-DS} 		& $\lang^{C} \leqq \lang^{CDS}$ from \pgfkeys{/exp prop = sublang}, \pgfkeys{/arg = CDS-leqq-DS} and applying \pgfkeys{/exp prop = trans}.\\
  \pgfkeys{/no = DS-not-leqq-C}  	& \pgfkeys{/eqn = DS-not-leqq-C} 	& $\lang^{S} \leqq \lang^{DS}$ from \pgfkeys{/exp prop = sublang}, \pgfkeys{/arg = S-not-leqq-C} and applying \pgfkeys{/exp prop = con3}. \\
  \pgfkeys{/no = C-<-DS}  			& \pgfkeys{/eqn = C-<-DS} 			& \pgfkeys{/arg = C-leqq-DS}, \pgfkeys{/arg = DS-not-leqq-C}. \\  
  \hline
  \pgfkeys{/no = DS-not-leqq-D}  	& \pgfkeys{/eqn = DS-not-leqq-D} 	& $\lang^{S} \leqq \lang^{DS}$ from \pgfkeys{/exp prop = sublang}, \pgfkeys{/arg = S-not-leqq-D} and applying \pgfkeys{/exp prop = con3}. \\
  \pgfkeys{/no = D-<-DS}  			& \pgfkeys{/eqn = D-<-DS} 			& $\lang^{D} \leqq \lang^{DS}$ from \pgfkeys{/exp prop = sublang}, \pgfkeys{/arg = DS-not-leqq-D}. \\  
  \hline
  \pgfkeys{/no = CD-leqq-DS}  		& \pgfkeys{/eqn = CD-leqq-DS} 		& $\lang^{CD} \leqq \lang^{CDS}$ from \pgfkeys{/exp prop = sublang}, \pgfkeys{/arg = CDS-leqq-DS} and applying \pgfkeys{/exp prop = trans}.\\
  \pgfkeys{/no = DS-not-leqq-S}  	& \pgfkeys{/eqn = DS-not-leqq-S} 	& \pgfkeys{/arg = CD-leqq-DS}, \pgfkeys{/arg = CD-not-leqq-S} and applying \pgfkeys{/exp prop = con3}. \\
  \pgfkeys{/no = S-<-DS}  			& \pgfkeys{/eqn = S-<-DS} 			& $\lang^S \leqq \lang^{DS}$ from \pgfkeys{/exp prop = sublang}, \pgfkeys{/arg = DS-not-leqq-S}. \\
  \hline
  \pgfkeys{/no = CD-not-leqq-DS}  	& \pgfkeys{/eqn = CD-not-leqq-DS} 	& $\lang^{S} \leqq \lang^{DS}$ from \pgfkeys{/exp prop = sublang}, \pgfkeys{/arg = S-not-leqq-CD} and applying \pgfkeys{/exp prop = con3}. \\
  \pgfkeys{/no = CD-<-DS}  		& \pgfkeys{/eqn = CD-<-DS} 				& \pgfkeys{/arg = CD-leqq-DS}, \pgfkeys{/arg = CD-not-leqq-DS}. \\
  \hline
\end{tabularx} 
} 
\caption{Derivation of the relative expressivity of our logics. Each of the references \weirdnecessity refer to properties stated at the start of Section \ref{expressivity:section:summary}. Bold faced numbers are illustrated in Figure \ref{expressivity:fig:summary}.}
\label{expressivity:table:derivation-of-full-map}
\end{table}
	Now comes our main result, which shows the relative expressivity between the logic of conditional belief, the logic of degrees of belief and the logic of safe belief.
	\begin{theorem}\label{expressivity:theorem:lc-ld-ls}
		$\lang^C < \lang^S$, $\lang^C \bowtie \lang^D$, $\lang^D \bowtie \lang^S$.
	\end{theorem}
		\begin{proof}
			See the derivation of (\pgfkeys{/lbl/C-<-S/no}), (\pgfkeys{/lbl/D-bowtie-S/no}) and (\pgfkeys{/lbl/C-bowtie-D/no}) in Table \ref{expressivity:table:derivation-of-full-map}. 
		\end{proof}
	Beyond showing the above theorem, Table \ref{expressivity:table:derivation-of-full-map} fully accounts for the relative expressivity between $\lang^C$, $\lang^D$, $\lang^S$, $\lang^{CD}$ and $\lang^{DS}$. Finally, using Corollary \ref{expressivity:corollary:conditional-loses-safe} and property \pgfkeys{/lbl/con1} we have that any expressivity result for $\lang^S$ holds for $\lang^{CS}$ and similarly for $\lang^{DS}$ and $\lang^{CDS}$. A more pleasing presentation of these results is found in Figure \ref{expressivity:fig:summary}.

\subsection{Reflection on bisimulation characterisation and expressivity} \label{expressivity:section:discussion}
Our bisimulation characterisation results are:
\[\begin{array}{llll} 
(\m,w) \bisim (\m',w') & \text{ iff } & (\m,w) \equiv^C (\m',w') \hspace{2cm} \ & \text{Theorem \ref{theorem:bisimcharlc}} \\
(\m,w) \bisim (\m',w') & \text{ iff } & (\m,w) \equiv^D (\m',w') & \text{Theorem \ref{theorem:bisimcharld}} \\
(\m,w) \bisim (\m',w') & \text{ iff } & (\m,w) \equiv^S (\m',w') & \text{Theorem \ref{theorem:bisimcharls}}
\end{array} \] In other words, bisimulation corresponds to modal equivalence in all three logics. Our expressivity results can be summarised as (Theorem~\ref{expressivity:theorem:lc-ld-ls}) 
\[\begin{array}{lcl} \lang^C &<& \lang^S \\ \lang^C &\bowtie& \lang^D \\ \lang^D &\bowtie& \lang^S \end{array} \]
The logic of conditional belief is less expressive than the logic of safe belief, the logic of conditional belief and the logic of degrees of belief are incomparable, as are the logic of degrees of belief and the logic of safe belief. 

Our results on bisimulation characterisation suggest that, in some sense, the three logics are the same, whereas our results on expressive power suggest that, in another sense, the three logics are different. It is therefore a good moment to explain how to interpret our results.

The bisimulation characterisation result in Corollary~\ref{expressivity:corollary:bisim-implies-modal-equiv-cds}
says that the information content of a given plausibility model is equally well described in the three logics. 
Now consider an even more specific case: a finite model; and consider a characteristic formula of that model (these can be shown to exist for plausibility models along the lines of \cite{jfak.lonely:2006,hvdetal.simbisim:2012}---where we note that we take models, not pointed models). For a model $M$ this gives us, respectively, formulas $\phi^C_M$, $\phi^D_M$, and $\phi^S_M$. Then the bisimulation characterisation results say that $\phi^C_M$, $\phi^D_M$, and $\phi^S_M$ are all equivalent. Now a characteristic formula is a very special formula with a unique model (modulo bisimilarity). For other formulas that do not have a singleton denotation (again, modulo bisimilarity) in the class of plausibility models, this equivalence cannot be achieved. That is the expressivity result. For example, given that $\lang^C < \lang^S$, there is a safe belief formula that is not equivalent to any conditional belief formula. This formula should then describe a property that has several non-bisimilar models. It is indeed the case that the formula $\Dia_a p$ used in the proof of Proposition \ref{expressivity:proposition:safe-beat-conditional} demonstrating $\lang^C < \lang^S$ has many models! It is tempting to allow ourselves a simplication and to say that the expressivity hierarchy breaks down if we restrict ourselves to formulas with unique models.\footnote{If we consider infinitary versions of the modalities in our logical languages, in other words, common knowledge and common belief modalities, we preserve the bisimulation characterisation results (for a more refined notion of bisimulation) but it is then to be expected that all three logics become equally expressive (oral communication by Tim French).} 

Finally, we must point out that in the publication on single-agent bisimulation \cite[p. 285]{DBLP:conf/ausai/AndersenBDJ13}, we posed the following conjecture: \begin{quote} {\em In an extended version of the paper we are confident that we will prove that the logics of conditional belief and knowledge, of degrees of belief and knowledge, and both with the addition of safe belief are all expressively equivalent. } \end{quote} It therefore seems appropriate to note that we have proved our own confident selves resoundingly wrong!

\section{Comparison and applications} \label{section:comparison}
We compare our bisimulation results to those in Demey's work \cite{demey:2011}, our expressivity results to those obtained in Baltag and Smets' \cite{baltagetal.tlg3:2008}, and finally discuss the relevance of our results for epistemic planning \cite{bolanderetal:2011}.
	
\paragraph*{Bisimulation}
Prior to our work Demey discussed the model theory of plausibility models in great detail in \cite{demey:2011}. Our results add to the valuable original results he obtained. Demey does not consider degrees of belief; he considers knowledge, conditional belief and safe belief. Our plausibility models are what \cite{demey:2011} refers to as uniform and locally connected epistemic plausibility models; he also considers models with fewer restrictions on the plausibility function. But given \cite[Theorem 35]{demey:2011}, these types of models are for all intents and purposes equivalent to ours. The semantics for conditional belief and knowledge are as ours, but his semantics for safe belief is different (namely as in \cite{baltagetal.tlg3:2008}). The difference is that in his case an agent safely believes $\phi$ if $\phi$ is true in all worlds as least as plausible as the current world, whereas in our case it is like that but \emph{in the normalised model}. This choice of semantics has several highly significant implications as we will return to shortly. 

In line with his interpretation of safe belief as a standard modality, Demey's notion of bisimulation for plausibility models is also standard. For example, whereas we require that \begin{quote} [forth$_\geq$] If $v \in W$ and $w \geq_a^\bisrel v$, $\exists v' \in W$ such that $w' \geq_a^\bisrel v'$ and $(v,v') \in \bisrel$, \end{quote} where we recall that $w \geq_a^\bisrel v$ means $\Min_a ([w]_R \inter [w]_a) \geq_a \Min_a([v]_R \inter [v]_a)$, he requires that \begin{quote} [forth$_\geq$] If $v \in W$ and $w \geq_a v$, $\exists v' \in W$ such that $w' \geq_a v'$ and $(v,v') \in \bisrel$. \end{quote} He obtains correspondence for bisimulation and modal equivalence in the logic of safe belief in \cite[Footnote 12 and Theorem 32]{demey:2011}. Our notion of bisimulation is less restrictive, as we will now illustrate by way of the examples in Figure~\ref{comparision:fig:uncontractable}.

	\begin{figure}[tb]
		\def\worldX{1.2cm}
		\centering
		\begin{tikzpicture}
			% M
			\node[world, inner sep =.13cm] (w1) {$w_1$}
				node[lbl, above = of w1] {};
			\node[world, inner sep =.13cm, right = \worldX of w1] (w2) {$w_2$}
				node[lbl, above = of w2] {$p$};
			\node[world, inner sep =.13cm, right = \worldX of w2] (w3) {$w_3$}
				node[lbl, above = of w3] {};
			\node[draw=white, inner sep =.1cm, world, right = 1.7*\worldX of w3] (w-mid) {$\cdots$};
			
			\node[world, inner sep =.03cm, right = 1.7*\worldX of w-mid] (wi-1) {$w_{i-1}$}
				node[lbl, above = of wi-1] {$p$};
			\node[world, inner sep =.13cm, right = \worldX of wi-1] (wi) {$w_{i}$}
				node[lbl, above = of wi] {};
			\node[left of = w1, xshift=-1.4cm] (mlbl) {$\m_i$};
			
			\path[link] (w2) -- node[below] {$a$} (w1);
			\path[link] (w3) -- node[below] {$a$} (w2);
			\path[link] ([xshift=-.7*\worldX]w-mid.west) -- node[below] {$a$} (w3);
			\path[link] (wi-1) -- node[below] {$a$} ([xshift=.7*\worldX]w-mid.east);
			\path[link] (wi) -- node[below] {$a$} (wi-1);		
		\end{tikzpicture}
\ \\ \ \\
		\begin{tikzpicture}
			% M
			\node[world, inner sep =.13cm] (w1) {$v_1$}
				node[lbl, above = of w1] {};
			\node[world, inner sep =.13cm, right = \worldX of w1] (w2) {$v_2$}
				node[lbl, above = of w2] {$p$};
			\node[left of = w1, xshift=-1.4cm] (mlbl) {$\m_2$};
			\path[link] (w2) -- node[below] {$a$} (w1);
		\end{tikzpicture}
		\caption{According to Demey's notion of bisimulation, model $\m_i$ (above) with alternating $\neg p$ and $p$ worlds is a bisimulation contraction. In this particular case $i$ is odd as $p$ does not hold at $w_i$. According to our notion of bisimulation, all $p$ worlds in model $\m_i$ are bisimilar and also all $\neg p$ worlds. Model $\m_2$ (below) is the contraction.}
		\label{comparision:fig:uncontractable}
	\end{figure}
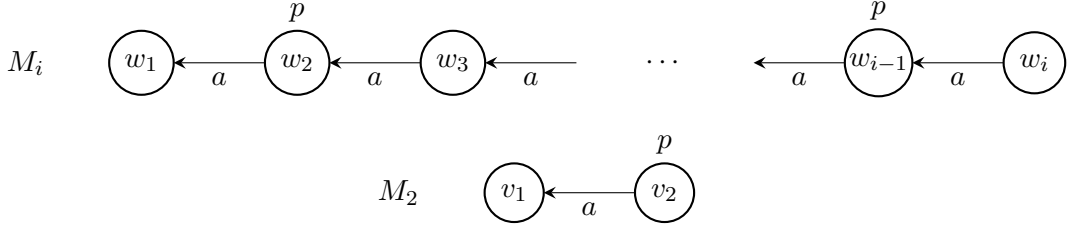
	
	Consider model $\m_i$ in Figure \ref{comparision:fig:uncontractable}. This is a single-agent model on a single proposition $p$ containing $i$ worlds, where the image of a world $w_j$ under $\geq_a$ is $\{w_1, \ldots, w_j\}$. The valuation is such that if the index of a world is even then $p$ holds, and otherwise $p$ does not hold. Now, using Demey's notion of bisimulation entails that the largest autobisimulation on $\m_i$ is the identity, and thus $\m_i$ is a bisimulation contraction. For example, we can find a formula that distinguishes $(\m_i, w_i)$ from $(\m_{i+2}, w_{i+2})$. For safe belief $\Box$ we now have Demey's semantics (see above) $\m,w \models \Box_a \phi$ iff $\m,v \models \phi$ for all $v$ with $w \geq_a v$. We now define $\phi_0 = \top$ and for any natural number $n \geq 1$ we let:
	$$
		\phi_n = \begin{cases}
					\Dia_a( \phi_{n-1} \wedge p ) & \mbox{if } n \mbox{ is even;}\\
					\Dia_a( \phi_{n-1} \wedge \neg p ) & \mbox{if } n \mbox{ is odd;}
				\end{cases}
	$$
	
for example	$$
		\phi_4 = \Dia_a ( \Dia_a ( \Dia_a ( \Dia_a ( \top \wedge \neg p ) \wedge p ) \wedge \neg p) \wedge p ).
	$$
We now have that for any $i \geq 1$, $\m_i, w_i \models \phi_i \wedge \neg \phi_{i+1}$, which makes this a distinguishing formula between $(\m_i, w_i)$ from $(\m_{i+2}, w_{i+2})$. In fact, the semantics of $\Box_a$ allow us to count the number of worlds in $\m_i$. In this sense Demey's logic is immensely expressive. 
	
Again referring to Figure \ref{comparision:fig:uncontractable}, consider $M_3$, the model with a most plausible $\neg p$ world, a less plausible $p$ world and an even less plausible $\neg p$ world. In the logic $L^C$ of conditional belief $w_1$ and $w_3$ of $M_3$ are modally equivalent. Hence they also ought to be bisimilar. But in Demey's notion of bisimilarity they are not. Hence we have a mismatch between modal equivalence and bisimilarity, which is not supposed to happen: it is possible for two worlds to be modally equivalent but not bisimilar. Demey also was aware of this, of course. To remedy the problem one can either strengthen the notion of modal equivalence or weaken the notion of bisimilarity. Demey chose the former (namely by adding the safe belief modality to the conditional belief modality), we chose the latter. Thus we regain the correspondence between bisimilarity and modal equivalence. Baltag and Smets~\cite{baltagetal.tlg3:2008} achieve the same via a different route: they include in the language special propositional symbols, so-called $S$-propositions. The denotation
of an $S$-proposition can be any subset of the domain. This therefore also makes the language much more expressive.

We believe that in particular for application purposes, weakening the notion of bisimulation, as we have done, is preferable over strengthening the logic, as in \cite{baltagetal.tlg3:2008,demey:2011}. This come at the price of a more complex bisimulation definition (and, although we did not investigate this, surely a higher complexity of determining whether two worlds are bisimilar), but, we venture to observe, also a very elegant bisimulation definition given the ingenious use of the bisimulation relation itself in the definition of the forth and back conditions of bisimulation. We consider this one of the highlights of our work.

\paragraph*{Expressivity}

In \cite{baltagetal.tlg3:2008} one finds many original expressivity results. Our results copy those, but also go beyond. We recall Table \ref{expressivity:table:derivation-of-full-map} for the full picture of our results, and the main results of those namely $\lang^C < \lang^S$, $\lang^C \bowtie \lang^D$, and $\lang^D \bowtie \lang^S$. The first, $\lang^C < \lang^S$, is originally found in \cite[page 34, Equation 1.7]{baltagetal.tlg3:2008}, and we obtained it using the same embedding translation. However, it may be worth to point out that in our case this translation still holds for the (in our opinion) more proper bisimulation preserving notion of safe belief. Baltag and Smets' $S$-propositions are arbitrary subsets of the domain, the (unnecessarily) far more expressive notion of safe belief. Baltag and Smets also discuss degrees of belief but do not obtain expressivity results for that, so $\lang^C \bowtie \lang^D$ may be considered novel and interesting. In artificial intelligence, the degrees of belief notion seems more widely in use than the conditional belief notion, so an informed reader had better be aware of the incomparability of both logics and may choose the logic to suit his or her needs. The result that $\lang^D \bowtie \lang^S$ could possibly also be considered unexpected, and therefore valuable. % Finally we should say that, at some stage, we expected to obtain another result, namely $\lang^D < \lang^S$, as the safe modality after all uses the $\leq$ relation more or less as a standard accessibility relation. It should therefore be more expressive than anything else. But, in fact, we obtained $\lang^D \bowtie \lang^S$. So, what is to be expected is after all not always the same as what is true.

\paragraph*{Planning}
An application area of plausibility models is epistemic planning. A consequence of Demey's notion of bisimulation is that even for single-agent models on a finite set of propositions, the set of distinct, contraction-minimal pointed plausibility models is infinite. For example, we recall that in Figure \ref{comparision:fig:uncontractable} any two pointed plausibility models in $\{ (\m_i,w_i) \mid i \in \Naturals \}$ are non-bisimilar. With our notion of bisimulation, there are in the single-agent case only finitely many distinct pointed plausibility models up to bisimulation. This was already reported in \cite{DBLP:conf/ausai/AndersenBDJ13}. Our motivation for this bisimulation investigation was indeed prompted by the application of doxastic logics in planning. 
	
In planning, an agent attempt to find a sequence of action, a plan, that achieves a given goal. A planning problem implicitly represents a state-transition system, where transitions are induced by actions. By exploring this state-space we can reason about actions and synthesise plans. A growing community investigates planning by applying dynamic epistemic logics \cite{bolanderetal:2011,loweetal:2011,andersenetal:2012}, where actions are epistemic actions. Planning with doxastic modalities has also been considered \cite{dontplanfortheunexpected}. This is done by identifying states with (pointed) plausibility models, and the goal with a formula of the doxastic language. Epistemic actions can be public actions, like hard and soft announcements \cite{jfak.jancl:2007}, but also non-public actions, such as event models \cite{baltagetal.tlg3:2008}. 
			
	With the state-space consisting of plausibility models, model theoretic results become pivotal when deciding the plan existence problem. Unlike Demey's approach, our framework leads to a finite state-space in the single-agent case and therefore the single-agent plan existence problem is decidable \cite{bolanderetal:2011}. At the same time we know that even in a purely epistemic setting the multi-agent plan existence problem is undecidable \cite{bolanderetal:2011}. But by placing certain restrictions on the planning problem it is possible to find decidable fragments even in the multi-agent case, for example, event models with propositional preconditions \cite{yu13:maee}. % Naturally, if even the single-agent case suffers from a potentially infinite state-space, the task of finding decidable fragments in the multi-agent case is even more difficult. For this reason we claim that our framework is superior to that of \cite{demey:2011} when used to formulate planning with doxastic attitudes.

%%%%%%%%%%%%%%%%%%
% Planning
%%%%%%%%%%%%%%%%%%

\section*{Acknowledgements}
We thank Giovanni Cina and Johannes Marti for productive exchanges of ideas following an ILLC seminar in 2015. We are also in dept to the anonymous reviewers for their thorough reading of the manuscript leading to many helpful comments and suggestions for revisions. Hans van Ditmarsch is also affiliated to IMSc (Institute of Mathematical Sciences), Chennai, as research associate. He acknowledges support from European Research Council grant EPS 313360. Preliminary versions of the results in this paper can be found in the PhD theses of Mikkel Birkegaard Andersen~\cite[Chapter 4]{andersen2015towards} and Martin Holm Jensen~\cite[Chapter 5]{jensenepistemic}.

\iffalse
\section*{Acknowledgements}
We thank Giovanni Cina and Johannes Marti for productive exchanges of ideas following an ILLC seminar in 2015. We are also in dept to the anonymous reviewers for many helpful comments and suggestions. Hans van Ditmarsch is also affiliated to IMSc (Institute of Mathematical Sciences), Chennai, as research associate. He acknowledges support from European Research Council grant EPS 313360. Preliminary versions of the results in this paper can be found in the PhD theses of Mikkel Birkegaard Andersen~\cite[Chapter 4]{andersen2015towards} and Martin Holm Jensen~\cite[Chapter 5]{jensenepistemic}.
\fi

	\bibliographystyle{plain}
	\bibliography{biblio2013MAB}

\end{document}